\pdfoutput=1
\documentclass[11pt,a4paper]{article}

\usepackage{amsfonts, amsmath, amssymb, color, a4wide}
\usepackage{hyperref}
\usepackage{url}            
\usepackage{booktabs}       
\usepackage{nicefrac}       
\usepackage{microtype}      
\usepackage{fancybox,framed}
\usepackage{authblk}
\usepackage{lmodern}
\usepackage{makecell}
\usepackage{bm}
\usepackage{algorithm}
\usepackage{algorithmic}
\usepackage{mathtools}
\usepackage{xspace}
\usepackage{xcolor}
\usepackage{amsthm}
\usepackage{cases}
\usepackage{natbib}
\usepackage{dsfont}
\usepackage{bbm}
\RequirePackage[shortlabels]{enumitem}

\newcommand*\diff{\mathop{}\!\mathrm{d}}

\newcommand{\bg}{\boldsymbol{g}}

\newcommand{\bx}{\boldsymbol{x}}
\newcommand{\bu}{\boldsymbol{u}}
\newcommand{\by}{\boldsymbol{y}}
\newcommand{\br}{\boldsymbol{r}}

\newcommand{\bI}{\boldsymbol{I}}

\newcommand{\bX}{\boldsymbol{X}}

\newcommand{\bxi}{\boldsymbol{\xi}}

\newcommand{\bz}{\boldsymbol{z}}

\newcommand{\bv}{\boldsymbol{v}}
\newcommand{\bzero}{\boldsymbol{0}}

\newcommand{\sA}{\mathcal{A}}
\newcommand{\sF}{\mathcal{F}}

\newcommand{\sX}{\mathcal{X}}
\newcommand{\sS}{\mathcal{S}}

\newcommand{\sN}{\mathcal{N}}
\newcommand{\sE}{\mathcal{E}}
\newcommand{\sH}{\mathcal{H}}
\newcommand{\sG}{\mathcal{G}}

\newcommand{\argmin}{\mathop{\mathrm{argmin}}}
\newcommand{\argmax}{\mathop{\mathrm{argmax}}}

\newcommand{\field}[1]{\mathbb{#1}}

\newcommand{\fP}{\field{P}}

\newcommand{\R}{\field{R}}
\newcommand{\A}{\field{A}}
\newcommand{\mO}{\field{O}}

\newcommand{\E}{\field{E}}

\newcommand{\Q}{\field{Q}}

\newcommand{\norm}[1]{\left\|{#1}\right\|}

\newcommand{\scO}{\mathcal{O}}

\DeclareMathOperator{\KL}{KL}

\newtheorem{theorem}{Theorem}[section]
\newtheorem{definition}[theorem]{Definition}

\newtheorem{lemma}[theorem]{Lemma}

\newtheorem*{remark}{Remarks}

\DeclarePairedDelimiter{\ceil}{\lceil}{\rceil}

\newcommand{\vertiii}[1]{{\left\vert\kern-0.25ex\left\vert\kern-0.25ex\left\vert #1 
		\right\vert\kern-0.25ex\right\vert\kern-0.25ex\right\vert}}

\DeclarePairedDelimiter\floor{\lfloor}{\rfloor}
\newcommand{\abs}[1]{\left\lvert#1\right\rvert}

\newcommand{\RSI}{\mathrm{RSI}}
\newcommand{\QG}{\mathrm{QG}}
\newcommand{\QC}{\mathrm{QC}}



\title{New Lower Bounds for Stochastic Non-Convex Optimization through Divergence Decomposition}

\author{El Mehdi Saad}
\author{Wei-Cheng Lee}
\author{Francesco Orabona}
\date{}
\affil{King Abdullah University of Science and Technology (KAUST)
Thuwal, 23955-6900, Kingdom of Saudi Arabia}
\affil[ ]{mehdi.saad@kaust.edu.sa, weicheng.lee@kaust.edu.sa, francesco@orabona.com}

\begin{document}
	
	\maketitle
	
	\begin{abstract}%
		We study fundamental limits of first-order stochastic optimization in a range of non-convex settings, including \(L\)-smooth functions satisfying Quasar-Convexity (QC), Quadratic Growth (QG), and Restricted Secant Inequalities (RSI). While the convergence properties of standard algorithms are well-understood in deterministic regimes, significantly fewer results address the stochastic case, where only unbiased and noisy gradients are available. We establish new lower bounds on the number of noisy gradient queries to minimize these classes of functions, also showing that they are tight (up to a logarithmic factor) in all the relevant quantities characterizing each class.
		Our approach reformulates the optimization task as a function identification problem, 
		leveraging \emph{divergence decomposition} arguments to construct a challenging subclass that leads to sharp lower bounds. 
		Furthermore, we present a specialized algorithm in the one-dimensional setting that achieves faster rates, suggesting that certain dimensional thresholds are intrinsic to the complexity of non-convex stochastic optimization.
	\end{abstract}
	
	
	\section{Overview and Main Contributions}\label{sec:intro}	
	In recent years, the optimization community has increasingly shifted its attention toward non-convex optimization problems. One strategy for addressing the intrinsic complexity of such problems is to focus on broad function classes that exhibit particular structural properties, such as Quasar-Convexity (QC), Restricted Secant Inequalities (RSI), the Polyak-Łojasiewicz (PL) condition, and Error Bound (EB) properties. Although the theoretical guarantees of first-order methods for these function classes are relatively well-understood under deterministic assumptions (i.e., when exact gradients are available), considerably less attention has been devoted to the stochastic setting, where only noisy gradient estimates are provided. In this paper, we study the task of minimizing a differentiable function \( f : \R^d \to \R \)
	\begin{equation}\label{eq:pb}
		\min_{\bx \in \sX}\  f(\bx),
	\end{equation}
	where \(\sX\) is either \(\R^d\) or a compact, convex subset of \(\R^d\), depending on the function class under consideration. The optimal value \( f^* \coloneqq \min_{\bx \in \sX} f(\bx) \) is assumed to be finite and the set of global minimizers, denoted by $\sX^*$, to be convex. The objective is to identify a point \(\bx^*\) in the set of minimizers, \(\bx^* \in \arg \min_{\bx \in \sX} f(\bx)\). Our focus is on algorithms that proceed via iterative queries to an oracle supplying information about \(f\) in the form of stochastic unbiased gradients. Given a function class \(\sF\) and a fixed budget of (potentially noisy) gradient evaluations, we aim to establish lower bounds on the achievable optimization error, quantified by the suboptimality gap $f(\hat{\bx}) - f^*$, where \(\hat{\bx}\) is the solution returned by the algorithm.
	
	We consider the classes of $L$-smooth functions\footnote{We say a function is $L$-smooth if the gradient is $L$-Lipschitz, i.e., $\norm{\nabla f(x)-\nabla f(y)}\leq L\norm{x-y}$.} satisfying one or a combination of the following properties:\footnote{In this paper $\norm{\cdot}$ and $\langle \cdot, \cdot \rangle$ denote the Euclidean norm and scalar product respectively. $B_d(\bx, r)$ denotes the ball centered at $\bx$ with radius $r$.} Quasar-Convexity ($\QC$), Quadratic Growth ($\QG$), and Restricted-Secant-Inequality ($\RSI$) with $\tau \in (0,1]$ and $\mu >0$:
	\begin{align*}
		\left(\tau-\QC\right)&: \forall \bx \in \R^d \quad	f(\bx) - f^* \le \frac{1}{\tau} \langle \nabla f(\bx) , \bx - \bx_p \rangle\\
		\left(\mu-\QG\right)&: \forall \bx \in \R^d \quad	f(\bx) - f^* \ge \frac{\mu}{2} \norm{\bx - \bx_p}^2\\
		\left(\mu-\RSI\right)&: \forall \bx \in \R^d \quad	\langle \nabla f(\bx), \bx-\bx_p \rangle \ge \mu \norm{\bx - \bx_p}^2,
	\end{align*}
	where $\bx_p$ is the projection of $\bx$ onto the set of global minimizers of $f$. 
	
	Given a function class $\sF$, an algorithm $\sA$ that solves this problem interacts iteratively with a first-order oracle, sequentially querying noisy gradients at points in $\sX$. In this work, our main objective is to establish lower bounds on the information-based complexity, specifically focusing on how many noisy gradient evaluations are needed to achieve a given accuracy for any function in the considered class. Moreover, we aim at obtaining lower bounds that are tight in all the relevant parameters of each function class (up to logarithmic terms).
	
	Our approach is motivated by techniques in the statistical literature for deriving minimax lower bounds on the risks of estimation procedures~\citep{tsybakov2003introduction}. We reformulate the optimization task as one of function identification by constructing a finite set of functions. 
	The challenge in identifying the correct function arises from the relative entropy—more precisely, the Kullback-Leibler divergence—between the feedback distributions corresponding to any pair of distinct functions. Ideally, this divergence should be minimized. However, to obtain tighter performance guarantees, the functions must be sufficiently separated so that misidentifying the underlying objective leads to a significant error. The function construction is therefore carefully chosen to balance these competing requirements.

	For the class of convex functions, lower bounds are established using quadratic and piece-wise linear functions that ensure a relative entropy between the distributions of oracle feedbacks at a given point that is uniform (independent of the queried point)~\citep{nemirovskij1983problem, agarwal2009information, raginsky2011information}. Such constructions suffice to obtain sharp lower bounds for convex functions. 
	
	In contrast, the function classes we consider here offer greater flexibility in constructing a ``difficult'' subclass of functions. Specifically, the non-convexity allows us to create functions whose curvature varies across different regions, leading to a relative entropy between the oracle feedback distributions that depend on where the query is made. We leverage this by designing functions that exhibit large curvature locally around the minimizer and smaller curvature away from it. Such a construction guarantees that misidentifying the function incurs a substantial error due to the large local curvature, while the overall relative entropy between feedback distributions remains limited because the algorithm must query near the minimizer (unknown a priori) in order to observe a higher relative entropy compared to other regions.
	This intuition is captured by the following result, known as \emph{divergence decomposition} in the bandits/active learning literature~\citep{lattimore2020bandit}. To the best of our knowledge, this technique has not been previously leveraged in the analysis of stochastic first-order optimization, and its application here provides a novel way to derive lower bounds.
	
	\paragraph{Divergence Decomposition.}
	Let $\mathbb{P}_0^T$ and $\mathbb{P}_1^T$ denote the probability distribution of the complete query--feedback transcript, including the learner's independent random seed and reported output under two different objective functions $f_0$ and $f_1$, and let $\E_0[\cdot]$ be the expectation with respect to $\mathbb{P}_0$. Let $\mathbb{P}_{i,\bx}$, for $i \in \{0,1\}$, be the distribution of the oracle feedback when querying $\bx \in \sX$ with the objective function $f_i$, assumed constant across rounds. Define $R \subset \sX$ and let $N_R$ be the number of rounds in which the algorithm queries a point $\bx \in R$. Then, we have\footnote{Formal definitions and a proof for this bound are presented in Section~\ref{sec:can_model}.}
	\[
	\text{KL}\left(\mathbb{P}_0^T, \mathbb{P}_1^T\right)
	\le 
	\E_0[N_{R}] \,\sup_{\bx \in R}\text{KL}\left(\mathbb{P}_{0,\bx}, \mathbb{P}_{1,\bx}\right)
	+ \E_0[T-N_{R}]\,
	\sup_{\bx \notin R} \text{KL}\left(\mathbb{P}_{0,\bx}, \mathbb{P}_{1,\bx}\right)~.
	\]
	
	The same decomposition applies to auxiliary feedback models whose means are measurable reference fields rather than gradients. We choose $R$ to be a ball centered at the minimizer of the objective function, where the relative entropy is large according to our construction. Because the environment selects the objective function uniformly at random from the constructed subclass, the average contribution of relative entropy in the region near the minimizer can be kept small relative to its contribution outside this region. In other words, there is a minimal cost for slightly larger relative entropy across all rounds, while the penalty from misidentifying the objective function remains significant. However, to ensure this argument holds, the subclass must be sufficiently large, and the regions $R$ around the functions' minimizers must be disjoint. Achieving this requires that the ambient dimension $d$ exceed a certain threshold determined by the problem parameters.

	\paragraph{Main Results.}
	We tailor the above analysis to different classes of non-convex functions, yielding the following new results. Throughout, $\sigma^2$ denotes the bound on the variance of the stochastic gradients provided by the oracle. The bounds apply to any randomized algorithm whose query at round $t$ can depend on all past observations and the algorithm’s independent internal randomness. The reported solution $\hat{\bx}$ may be any measurable function of the complete history and need not coincide with the final iterate.
	\begin{theorem}{(Informal)}
		Let $\sX = B_d(\bzero, D)$. Suppose that $d = \Omega\left(\log(2/\tau)\right) $ and $T$ exceeds a threshold depending on $\sigma, L, D,$ and $\tau$. For any optimization algorithm having access to unbiased stochastic gradient with variance bounded by $\sigma^2$, there exists an $L$-smooth, $\tau$-QC function $f$ such that
		\[
		\E[f(\hat{\bx})]-f^* = \Omega\left( \frac{D\sigma}{\tau\sqrt{\ln(2/\tau) T}} \right)~.
		\]
	\end{theorem}
	This lower bound is optimal up to a $\sqrt{\log(2/\tau)}$ factor, since Stochastic Gradient Descent (SGD) guarantees $\E[f(\hat{\bx})]-f^* = \scO\left(\frac{D\sigma}{\tau \sqrt{T}}\right)$. Within the QC class, faster rates are possible when the objective function also satisfies the QG condition.
	\begin{theorem}{(Informal)}
		Let $\sX = \R^d$. Suppose that $L/\mu\ge202$ and $d = \Omega\left(\log(2/\tau)\right)$. For any optimization algorithm having access to unbiased stochastic gradient with variance bounded by $\sigma^2$, there exists an $L$-smooth, $\tau$-QC, and $\mu$-QG function $f$ such that
		\[
		\E[f(\hat{\bx})]-f^* = \Omega\left(  \frac{\sigma^2}{\mu\tau^2\log(2/\tau)T}\right)~.
		\]
	\end{theorem}
	In Section~\ref{sec:comparisons}, we demonstrate that any function exhibiting both QC and QG properties is strongly quasar-convex, implying that the previously stated result also extends to this class of functions. In Section~\ref{sec:ub}, we also show an upper bound for SGD differing only by a $\log(2/\tau)$ factor. Next, we introduce a corresponding lower bound for the class of RSI functions. Additionally, based on the comparisons in Section~\ref{sec:comparisons}, this lower bound is also valid for other non-convex function classes.
	\begin{theorem}{(Informal)}\label{thm:rsi_l}
		Let $\sX = \R^d$, $0<\mu\le L$, and $\kappa = L/\mu$. Suppose that $d = \Omega\left(\log(2\kappa)\right)$. For any optimization algorithm having access to unbiased stochastic gradient with variance bounded by $\sigma^2$, there exists an $L$-smooth, $\mu$-RSI function $f$ such that
		\[
		\E[f(\hat{\bx})]-f^* = \Omega\left( \frac{L\sigma^2}{\mu^2\log(2\kappa)T}\right)~.
		\]
	\end{theorem}
	A matching upper bound, up to a $\log(\kappa)$ factor, is shown in Section~\ref{sec:ub} via SGD. Note that the $\mu$-RSI class contains $\mu$-strongly convex functions, revealing a factor-$\kappa$ gap between the best possible rates for these two classes.
	
	Finally, we present an argument indicating that the necessity for $d$ to exceed a certain problem-dependent threshold is an intrinsic characteristic of the optimization problem itself rather than merely a byproduct of our analysis. Specifically, in Section~\ref{sec:ub}, we propose a novel algorithm tailored for the one-dimensional setting and demonstrate the following better upper bound. While previous lower-bound constructions for locating stationary points in non-convex smooth functions also required high dimensionality \citep{arjevani2023lower}, their arguments demand a dimension that grows polynomially with the problem parameters and the time horizon $T$. By contrast, our bound is satisfied once the dimension merely exceeds a logarithmic function of the problem parameters, with no dependence on $T$.
	\begin{theorem}{(Informal)}\label{thm:rsi_u}
		Let $f:\mathbb{R} \to \mathbb{R}$ be an $L$-smooth, $\mu$-RSI function, with access to a $\sigma$-subGaussian oracle. For $T$ sufficiently large (depending on the problem parameters), there is an algorithm whose output satisfies\footnote{$\tilde{\scO}(\cdot)$ hides poly-logarithmic factors in the problem parameters.} with probability at least $1-\delta$:
		\[
		f(\hat{\bx}_T) - f^*  =\tilde{\scO}\left(\frac{\sigma^2\log(1/\delta)}{\mu T}\right)~. 
		\]
		If $f$ is $\mu-\QG$ and $\tau-\QC$, then with probability at least $1-\delta$:
		\[
		f(\hat{\bx}_T) - f^*  =\tilde{\scO}\left(\frac{\sigma^2\log(1/\delta)}{\mu\tau T}\right)~.
		\]
	\end{theorem}

	\paragraph{Paper outline.} The remainder of this paper is organized as follows. In Section~\ref{sec:comparisons}, we review related work and position our contributions within the existing literature and provide a comparative analysis to highlight their relevance to our framework. In Section~\ref{sec:lb}, we present a proof sketch establishing a lower bound on the performance of any algorithm in this setting, demonstrating inherent limitations. Section~\ref{sec:ub} complements this result by deriving matching upper bounds, up to logarithmic terms. Finally, we conclude with a discussion on potential extensions and implications of our results.

	\section{Related Work and Classes Comparisons}\label{sec:comparisons}
	Various classes of non-convex functions have received attention in the optimization literature. Notable examples include the Polyak-Łojasiewicz (PL) condition~\citep{polyak1963gradient}, Error Bounds (EB)~\citep{luo1993error}, Quadratic Growth (QG)~\citep{anitescu2000degenerate}, Essential Strong Convexity (ESC)~\citep{liu2014asynchronous}, Restricted Secant Inequality (RSI)~\citep{polyak1963gradient}, Restricted Strong Convexity (RSC)~\citep{zhang2013gradient}, and Quasi-Strong Convexity~\citep{necoara2019linear}. All these classes can be seen as gradual weakening of the strong convexity (SC) assumption, until including non-convex functions~\citep{karimi2016linear}.
	
	Other classes of non-convex functions have been introduced as relaxations of convexity, as the quasar-convex functions that we consider~\citep{hardt2018gradient, hinder2020near}.\footnote{Also referred to as weakly quasi-convex.}
	This class is characterized by a parameter $\tau \in (0,1]$ (see the definition in Section~\ref{sec:intro}), where the special case $\tau = 1$ corresponds to the class of star-convex functions \citep{nesterov2006cubic}. 	In the deterministic setting, \citet{guminov2023accelerated} established an upper bound of the order $\mathcal{O}(L\norm{\bx_0 - \bx^*}^2/(\tau^2 T^2))$ under the assumption of $L$-smoothness and quasar-convexity; \citet{nesterov2021primal} obtained a rate of $\mathcal{O}(L\norm{\bx_0 - \bx^*}^2/(\tau^3 T^2))$. The lower bound of \citet{hinder2020near} matches the former rate up to logarithmic factors.    
	A subclass of quasar-convex functions, known as Strongly Quasar-Convex (SQC) functions, was also analyzed in \citet{hinder2020near} in the deterministic setting. 
	Notably, the class of strongly quasar-convex functions includes quasar-convex functions that satisfy the quadratic growth condition, as we show later in this section.
	
	Most of the previous work on both upper and lower bounds has focused on the deterministic setting, i.e., when the gradients are exact, or in the finite-sum setting.
	
	

	
	In the stochastic optimization setting,   guarantees for SGD were developed for $L$-smooth $\mu$-PL objective functions by\footnote{The guarantee in \citet{karimi2016linear} needs the additional assumption of bounded stochastic gradients.} \cite{li2021second} and \citet{khaled2020better}, independently and at the same time. They showed that using a step size $\alpha_t = \mathcal{O}(1/(\mu t))$, the guarantee $\mathbb{E}[f(\bx_t)] - f^* = \mathcal{O}(L\sigma^2/(\mu^2 T))$ can be achieved. This is in contrast to the class of $\mu$-strongly convex functions, where the known optimal guarantees are $\mathcal{O}(\sigma^2/(\mu T))$~\citep{agarwal2009information, raginsky2011information}. 
	
	The presence of the gap between the rates of PL and strongly convex functions begs the question if these non-convex optimization rates are optimal or not.
	Yet, to the best of our knowledge, we are not aware of any result for the stochastic case that covers the classes of functions we list above. Our work fills the knowledge gap by providing the first lower bounds for some of the above class of functions and at the same time proving their (almost) optimality.
	
	A separate line of research, represented by \citet{arjevani2023lower}, derives lower bounds for smooth stochastic non-convex problems with the goal of reaching an $\epsilon$-stationary point. In contrast, our work focuses on minimizing the optimization error itself by seeking a global minimizers within the function classes we study.

	
	\paragraph{Classes Comparison.}
	To better understand our results and the relationship with known upper bounds, it is important to take into account the relationships between these different classes. Indeed, it is known that these classes are included one into the other,   
	as proved in Theorem 5 of \citet{karimi2016linear}.
	Hence, we recall the comparison of several non-convex function classes,
	including functions that satisfy Star Strong Convexity (*SC)---we give the precise definitions of the classes in Section~\ref{sec:compare_a_1}:
	\begin{equation}\label{eq:inclusions}
		(\text{SC}) 
		\rightarrow 
		(\text{*SC}) 
		\rightarrow 
		(\text{RSI}) 
		\rightarrow 
		(\text{EB}) \equiv (\text{PL}) 
		\rightarrow 
		(\text{QG}),
	\end{equation}
	where $\rightarrow$ indicates an implication and $\equiv$ indicates an equivalence with potentially difference constants. Note that a more detailed comparison, incorporating class parameters, can be found in \cite{guille2021study}. Moreover, we complement these comparisons with additional results presented in Lemma~\ref{lem:compare} below (proved in Section~\ref{sec:compare_a_2}).
	\begin{lemma}\label{lem:compare}
		Let $\mu >0$ and $\tau \in (0,1]$. We have:
		\begin{itemize}
			\item $\mu\text{-RSI} \subsetneq \mu\text{-EB}$.
			\item $\tau\text{-QC} \cap \mu\text{-QG} \subsetneq (\frac{\tau}{2}, \frac{\mu}{2})\text{-SQC}$.
			\item $\exists f:\R^2\rightarrow\R$, $f\in\mu\text{-EB}$ such that $f\notin a\text{-RSI}$ for any $a>0$.
			\item $\exists f:\R\rightarrow\R$, $f\in \mu\text{-RSI}$ such that $f\notin a\text{-*SC}$ for any $a>0$.
		\end{itemize}
	\end{lemma}
	In particular, Lemma~\ref{lem:compare} shows that the lower bound derived for the class of $\mu$-RSI functions also holds for the class of $\mu$-EB functions. Likewise, the lower bound for the class of $\tau$-QC $\cap$ $\mu$-QG  functions extends to the class of $(\tau/2, \mu/2)$-strongly quasar functions. Furthermore, it confirms that the inclusion of RSI functions is strict, as is the inclusion of star strongly convex functions within RSI functions. Equipped with this lemma and the inclusions in \eqref{eq:inclusions}, we can summarize  upper and lower bounds in Table~\ref{tab:complexity}, including our new results.

	\begin{table}[t]
		\centering
		\caption{Complexity bounds for various classes under the \( L \)-smoothness condition. The notation \( \tilde{\Omega} \) hides a \( 1/\sqrt{\log(2/\tau)} \) factor for \( \tau \)-QC, a \( 1/\log(2/\tau) \) factor for \( \tau \)-QC+\(\mu\)-QG, and a \( 1/\log(2\kappa) \) factor for \( \mu \)-RSI. Here, \( \sigma^2 \) represents the variance of the stochastic gradient. For $\tau$-QC in the last column, \( D \) denotes the diameter of the optimization domain. }
		\label{tab:complexity}
		\resizebox{1\textwidth}{!}{
			\begin{tabular}{c c c c c c | c}
				\toprule
				& \multicolumn{5}{c|}{\textbf{Fast Rate}} & \textbf{Slow Rate} \\
				\midrule
				\textbf{Setting} & \textbf{$\mu$-SC}  & \textbf{$\tau$ QC+$\mu$-QG} & \textbf{$\mu$-RSI} & \textbf{$\mu$-EB} & \textbf{$\mu$-PL} & \textbf{$\tau$-QC} \\
				\midrule
				Upper Bound & 
				\makecell{\(\mathcal{O}(\frac{\sigma^2}{\mu T})\) \\ \textcolor{blue}{\citep{hazan2014beyond}}}& \makecell{
					\(\mathcal{O}(\frac{\sigma^2}{\tau^2\mu T})\)\\(Theorem~\textcolor{blue}{\ref{thm:ub1}})} & 
				\makecell{\(\mathcal{O}(\frac{\kappa\sigma^2}{\mu T})\) \\ (Theorem~\textcolor{blue}{\ref{thm:ub1}})}& 
				\makecell{\(\mathcal{O}(\frac{\kappa^3\sigma^2}{\mu T})\) \\  \textcolor{blue}{\citep{li2021second}}}& 
				\makecell{\(\mathcal{O}(\frac{\kappa\sigma^2}{\mu T})\) \\  \textcolor{blue}{\citep{li2021second}}} & 
				\makecell{\(\mathcal{O}(\frac{D\sigma}{\tau \sqrt{T}})\) \\ \textcolor{blue}{\citep{jin2020convergence}}}  \\ 
				
				Lower Bound & 
				\makecell{\(\Omega(\frac{\sigma^2}{\mu T})\) \\ \textcolor{blue}{\citep{agarwal2009information}}} & 
				\makecell{\(\tilde{\Omega}(\frac{\sigma^2}{\tau^2\mu T})\) \\ (Theorem~\textcolor{blue}{\ref{thm:1}})} & 
				\makecell{\(\tilde{\Omega}(\frac{\kappa\sigma^2}{\mu T})\) \\ (Theorem~\textcolor{blue}{\ref{thm:2}})}& 
				\makecell{\(\tilde{\Omega}(\frac{\kappa\sigma^2}{\mu T})\) \\ (Theorem~\textcolor{blue}{\ref{thm:2}})} & 
				\makecell{\(\Omega(\frac{\sigma^2}{\mu T})\) \\ \textcolor{blue}{\citep{agarwal2009information}}} & 
				\makecell{\(\tilde{\Omega}(\frac{D\sigma}{\tau \sqrt{T}})\) \\ (Theorem~\textcolor{blue}{\ref{thm:3}})} \\
				\bottomrule
			\end{tabular}
		}
	\end{table}

	\section{Lower Bounds}\label{sec:lb}
	
	In this section, we present the formal statement of the lower bound results and a detailed sketch of the proof of one of the results. We start by giving a definition of the stochastic oracle considered, then we introduce the quantity of interest, that is the minimax optimization error. We adopt some of the notation used in \cite{agarwal2009information}. 
	
	Let $\sF$ be a class of objective functions. An algorithm $\mathcal{A}$ solving the problem~\eqref{eq:pb} interacts iteratively with a stochastic first-order oracle. At each iteration $t = 1, \dots, T$, the algorithm selects a query point $\bx_t \in \mathcal{X}$ and receives an oracle response $\phi(\bx_t, f)$, which represents a noisy gradient of $f$ at $\bx_t$. The noise is assumed to be zero-mean with bounded variance. Based on the sequence of past responses, $\{\phi(\bx_1, f), \dots, \phi(\bx_t, f)\}$, the algorithm determines the next query point $\bx_{t+1}$. After $T$ rounds the algorithm outputs $\hat{\bx}_T \in \mathcal{X}$. This output is not restricted to the final iterate; it may be any function of the observed history $\{\phi(\bx_1, f), \dots, \phi(\bx_T, f)\}$ to the set $\mathcal{X}$. Our analysis focuses on the information-theoretic complexity of this procedure, aiming to determine the number of oracle queries required to achieve a given accuracy for any function in $\mathcal{F}$.
	
	Let $\A_T$ denote the class of optimization algorithms that perform $T$ oracle queries. For any algorithm $\sA \in \A_T$, we define the optimization error on the function $f:\sX \to \R$ as  
	\[
	\epsilon_T(\mathcal{A}, f, \mathcal{X}, \phi) := f(\hat{\bx}_T) - f^*,
	\]
	where $\hat{\bx}_T$ is the final output of $\sA$ after $T$ queries. When the oracle responses are stochastic, this error becomes a random variable due to the inherent noise in the oracle outputs. In this case, the quantity of interest is the expected optimization error, given by  $\E_{\phi}\left[\epsilon_T(\mathcal{A}, f, \mathcal{X}, \phi)\right]$.
	
	For the given function class $\mathcal{F}$ over $\mathcal{X}$ and the set of all optimization algorithms with $T$ oracle queries, we define the minimax optimization error as  
	\[
	\epsilon^*_T(\sF, \sX; \phi) := \inf_{\mathcal{A} \in \A_T} \sup_{f \in \sF} \ \E_{\phi}\left[\epsilon_T(\sA, f, \sX, \phi)\right].
	\]
	Throughout the paper, when the oracle function is clear from the context, we simplify the notation to $\epsilon^*_T(\mathcal{F}, \mathcal{X})$.
	Below, we give the definition of stochastic first-order oracles considered in this work.
	
	\begin{definition}\label{def:oracle}
		Given a set $\sX \subseteq \R^d$, a class of functions $\sF$, and $\sigma>0$, the class of first-order stochastic oracles consists of random mappings $\phi: \sX \times \sF \to \R^d$ such that $\forall \bx \in \sX$
		\[
		\E[\phi(\bx, f)] = \nabla f(\bx) \quad \text{and} \quad \E[\norm{\phi(\bx, f)-\nabla f(\bx)}^2] \le \sigma^2~. 
		\]
        Each oracle call uses fresh randomness, independent of the past and of the learner's internal randomness. In particular, if $\mathcal{H}_t$ denotes the information available immediately before the response to query $\bx_t$, including the randomness used to select that query, and $\bg_t=\phi(\bx_t,f)$, then
        \[
        \mathbb{E}[\bg_t\mid\mathcal{H}_t]=\nabla f(\bx_t),\qquad
        \mathbb{E}[\|\bg_t-\nabla f(\bx_t)\|^2\mid\mathcal{H}_t]\le\sigma^2.
        \]
        In the one-dimensional results, a $\sigma$-subGaussian oracle additionally satisfies, for all $\lambda\in\R$,
        \[
        \mathbb{E}\!\left[\exp\bigl(\lambda(g_t-f'(x_t))\bigr)\mid\mathcal{H}_t\right]
        \le\exp(\lambda^2\sigma^2/2).
        \]
	\end{definition}
	Let $\mathbb{O}_{\sigma}$ denote the family of first-order stochastic oracles introduced in Definition \ref{def:oracle}. In all forthcoming lower-bound results, the learner observes a multivariate Gaussian vector with independent components whose total variance does not exceed $\sigma^{2}$. The corresponding lower bounds for each problem class are stated below.
	\begin{theorem}\label{thm:1}
		Let $\mu, L>0$ and $\tau \in (0,1]$. Suppose that $d \ge 3\log_{5/4}(2/\tau)$ and $L/\mu \ge 202$. Let $\sX = \R^d$ and $\QC\cap \QG$ denote the set of $\tau$-Quasar convex and $L$-smooth functions satisfying the $\mu$-Quadratic growth condition. For some universal constant $c>0$, we have 
		\[
		\sup_{\phi \in \mO_{\sigma}} \ \epsilon^*_T(\QC\cap \QG, \sX; \phi) \ge c \cdot \frac{\sigma^2}{\mu\tau^2\log(2/\tau)T}~.
		\]
	\end{theorem}
	\begin{theorem}\label{thm:2}
		Let $0<\mu\le L$ and $\kappa=L/\mu$. Suppose that $d \ge 2\log_{5/4}(2\kappa)$. Let $\sX = \R^d$ and $\RSI$ denote the set of $L$-smooth functions satisfying the $\mu$-RSI condition. For some universal constant $c>0$, we have 
		\[
		\sup_{\phi \in \mO_{\sigma}} \ \epsilon^*_T(\RSI, \sX; \phi) \ge c \cdot \frac{L\sigma^2}{\mu^2\log(2\kappa)T}~.
		\]
	\end{theorem}	
	\begin{theorem}\label{thm:3}
		Let $D, L>0$ and $\tau \in (0,1]$. Suppose that $d \ge \log_{16/15}(4/\tau^2)$ and $T \ge \frac{35 \sigma^2}{L^2D^2\tau^2\log_{16/15}(2/\tau)}$. Let $\sX = B_d(\bzero, D)$ and $\QC$ denote the set of $\tau$-Quasar convex and $L$-smooth functions. For some universal constant $c>0$, we have
		\[
		\sup_{\phi \in \mO_{\sigma}} \ \epsilon^*_T(\QC, \sX; \phi) \ge c \cdot \frac{D\sigma}{\tau\sqrt{\ln(2/\tau) T}}~.
		\]
	\end{theorem}
	
	In the following, we give a detailed proof sketch of Theorem~\ref{thm:1}. Its complete proof, as well as the proofs of the other results, can be found in Sections~\ref{sec:proof1},~\ref{sec:proof2}, and \ref{sec:proof3}. The main difference between the proofs lies in the function construction.
	\paragraph{Proof Sketch of Theorem~\ref{thm:1}.} 
	Let \(L, \mu > 0\), \(\tau \in (0,1]\), and \(\sigma > 0\). Suppose \(\kappa := \tfrac{L}{\mu} \ge 202\) and \(d \ge 3 \log_{5/4}\!\bigl(2/\tau\bigr)\). We begin by constructing a finite family of ``difficult'' functions. Next, we introduce an explicit form of the stochastic oracle. Then, we reduce the optimization problem to a function identification one. Finally, we employ the divergence decomposition lemma together with Pinsker's inequality to derive the lower bound.

	We introduce the following additional notation. For any integer \(n \ge 1\), let \(\mathcal{F}_d\) denote the class of real-valued functions on \(\mathbb{R}^d\) that are \(L\)-smooth, $\tau-\QC$, and satisfy the $\mu-\QG$. Let \(A := \{1, \dots, d_0\}\) and \(\bar{A} := [d] \setminus A\). For any \(\bx \in \mathbb{R}^d\), let \(\bx_A \in \mathbb{R}^{d_0}\) be the vector of its first \(d_0\) coordinates, and \(\bx_{\bar{A}} \in \mathbb{R}^{d - d_0}\) be the vector of its remaining coordinates. For two 
	vectors \(\mathbf{a} \in \mathbb{R}^{d_0}\) and \(\mathbf{b} \in \mathbb{R}^{d - d_0}\), we write \((\mathbf{a}, \mathbf{b})\) to denote the vector in \(\mathbb{R}^d\) whose first \(d_0\) components are \(\mathbf{a}\) and last \(d - d_0\) components are \(\mathbf{b}\).
	
	\textbf{Functions Construction.} Define \(d_0 := \lceil 2\log_{5/4}(2/\tau)\rceil\). For \(d \ge 3\log_{5/4}(2/\tau)\), we construct a function \(F: \R^d \to \R\) that depends only on the first \(d_0\) coordinates. Concretely, there is a function \(f: \R^{d_0} \to \R\) such that \(F(\bx) = f(\bx_A)\). Because \(F\) does not depend on the remaining \(d - d_0\) coordinates, its partial derivatives in those directions are zero. Hence, the gradient of \(F\) is \(\nabla F(\bx) = \bigl(\nabla f(\bx_A), \bzero_{\bar{A}}\bigr)\), so it matches \(\nabla f(\bx_A)\) in the first \(d_0\) entries and is zero elsewhere. Before detailing the expression of \(f\), we present the following lemma (proved in Appendix~\ref{sec:tech}), which ensures that the properties of \(f\) carry over to \(F\).
	\begin{lemma}\label{lem:0}
		If \(f \in \mathcal{F}_{d_0}\) and \(f\) has a unique global minimizer, then \(F \in \mathcal{F}_d\).
	\end{lemma}
	Let \(\Delta > 0\) and \(a \in (0, 1)\) be parameters to be specified later. Let \(m \ge 2\) denote the size of the function subclass we are constructing. We pick \(m\) elements \(\bz_1, \dots, \bz_m\) in \(B_{d_0}(\bzero_A, 5\Delta)\), and define a set of functions \(f_{1}, \dots, f_{m}\) so that each \(f_{i}\) belongs to \(\sF_{d_0}\) and admits \(\bz_i\) as its unique global minimizer. Following the insight from Section~\ref{sec:intro}, we design each \(f_{i}\) so that its curvature is large near \(\bz_i\) but smaller elsewhere, subject to the constraint that \(f_{i}\) must lie in \(\sF_{d_0}\). For each \(i \in [m]\), let \(f_{i}\) be a function whose value at \(\bz_i\) is zero and whose gradient for all \(\bx \in \mathbb{R}^{d_0}\) is given by
	\[
	\nabla f_{i}(\bx) =
	\begin{cases}
		\displaystyle 169\mu(\bx - \bz_i), & \text{if } \|\bx - \bz_i\| < \Delta, \\[1em]
		\displaystyle -169\mu\Bigl(\bx - \bz_i - \Delta \tfrac{\bx - \bz_i}{\|\bx - \bz_i\|}\Bigr) + 169\mu\Delta \tfrac{\bx - \bz_i}{\|\bx - \bz_i\|}, & \text{if } \Delta \leq \|\bx - \bz_i\| < (1+a)\Delta, \\[1em]
		\displaystyle 169\mu(1-a)\Delta  \tfrac{\bx - \bz_i}{\|\bx - \bz_i\|} , & \text{if } \bx \in B_{d_0}(\boldsymbol{0}, 8\Delta)\setminus B_{d_0}(\bz_i, (1+a)\Delta), \\[1em]
		\displaystyle  169\mu\left(\bx-8\Delta \tfrac{\bx}{\norm{\bx}}\right)+169\mu(1-a)\Delta  \tfrac{\bx - \bz_i}{\|\bx - \bz_i\|}, & \text{if } \norm{\bx} > 8\Delta~.
	\end{cases}
	\] 
	The corresponding explicit form of \(f_i\) is provided in Section~\ref{sec:proof1}. We now specify the choice of the parameter \(a\). Our goal is to make the norm of the differences among the gradients \(\left( \nabla f_i\right)_{i \in [m]}\) small so that the associated relative entropy between feedback distributions remains small. To achieve this, we choose \(a\) to be as large as possible while assuring that the functions \(f_i\) are \(\tau\)-QC and \(\mu\)-QG. This requirement corresponds to choosing $a$ as the positive root of the function \(r \mapsto 1 - \tfrac{\tau}{2} - \tau r - \bigl(1 - \tfrac{\tau}{2}\bigr)r^2\), which implies \(1 - a \le \tau\). To gain a geometric intuition, we present in Figure~\ref{fig:fz} a plot of \(f_i\) in the one-dimensional case.
	
	\begin{figure}
		\centering
		\includegraphics[width=0.9\linewidth]{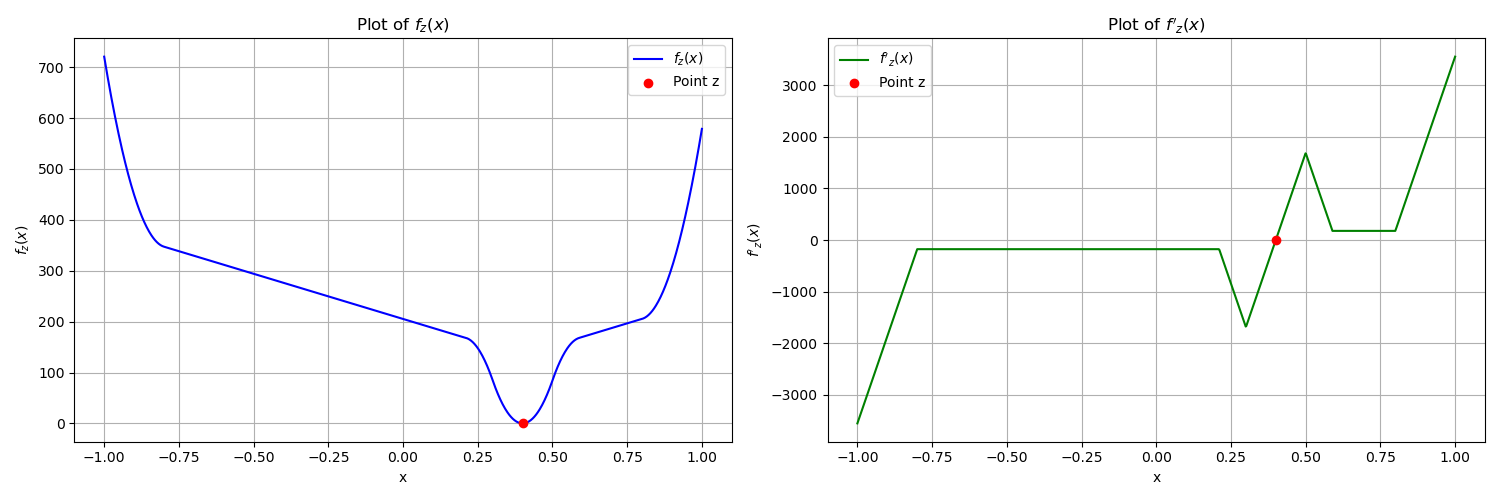}
		\caption{Plots of $f_{z}(x)$ and its derivative $f'_{z}(x)$ in one dimension, with $\mu=100, \Delta = 0.1, \tau = 0.2, z = 0.4$. The red dot marks the global minimizer.}
		\label{fig:fz}
	\end{figure}
	
	In Lemma~\ref{lem:f}, we prove that each \(f_i\) lies in \(\sF_{d_0}\). Consequently, by Lemma~\ref{lem:0}, the functions \(F_i: \mathbb{R}^d \to \mathbb{R}\) defined by \(F_i(\bx) = f_i(\bx_A)\) for \(i \in [m]\) belong to \(\sF_d\). Observe that each \(F_i\) achieves its minimum value of 0, and its set of global minimizers is \(\{(\bz_i, \by): \by \in \R^{d - d_0}\}\), which is convex.
	
	\textbf{Oracle Construction.} We now specify the stochastic oracle \(\phi: \R^d \times \sF_d \to \R^d\) used in the proof. At each oracle call, let \(\bxi\) be a fresh independent sample from a \(d_0\)-dimensional normal distribution with zero mean and covariance matrix \(\frac{\sigma^2}{d_0}\bI_{d_0}\), i.e., \(\bxi \sim \sN_{d_0}\bigl(\bzero_A, \frac{\sigma^2}{d_0}\bI_{d_0}\bigr)\). Given an input \(\bx \in \R^d\) and a function \(H \in \sF_d\), we define \(\phi(\bx, H) = \nabla H(\bx) + (\bxi, \mathbf{0}_{\bar{A}})\). Thus, \(\phi\) is a stochastic oracle belonging to $\mathbb{O}_{\sigma}$ as specified in Definition~\ref{def:oracle}. Moreover, for each \(i \in [m]\), we have \(\phi(\bx, F_i) = \bigl(\nabla f_i(\bx_A) + \bxi, \bzero_{\bar{A}}\bigr)\).

	\textbf{Divergence Decomposition.} We reduce the optimization problem to one of function identification. To that end, consider a reference field \(V: \R^d \to \R^d\) of the form \(V(\bx)=(v(\bx_A),\bzero_{\bar A})\), where \(v:\R^{d_0}\to\R^{d_0}\) is a measurable field. These reference fields need not be continuous or gradients of functions. We define, for any \(\bx \in \R^{d_0}\),
	\[
	v(\bx) =
	\begin{cases}
		\displaystyle \bzero, & \text{if }\|\bx\| < 8\Delta, \\[1em]
		\displaystyle 169\mu \left( \bx-8\Delta \bx/\norm{\bx}\right), & \text{if }\|\bx \| \ge 8\Delta~.
	\end{cases}
	\]
	Let \(\hat{\bx} \in \R^d\) denote the output of the optimization algorithm. For each \(i \in [m]\), let \(\fP_i\) and \(\Q\) denote the probability distributions of the complete transcript (the random seed, the \(T\) query--feedback pairs, and the reported output)\footnote{For a rigorous definition of these quantities, see Section~\ref{sec:can_model} where we define the canonical model.} under objective \(F_i\) and under the auxiliary Gaussian feedback \(V(\bx)+(\bxi,\bzero_{\bar A})\), respectively (we omit the dependence on \(T\) in our notation). Also, let \(\E_i[\cdot]\) and \(\E[\cdot]\) denote the expectations with respect to \(\fP_i\) and \(\Q\), respectively. For each \(i \in [m]\), we define the “good identification event” as \(\sE_i := \{\hat{\bx}_A \in B_{d_0}(\bz_i, 2\Delta)\}\). In the remainder of the proof, we derive a lower bound on the misidentification event \(\neg \sE_i\) and then use this bound to relate misidentification to the optimization error using the bound below. For any $i \in [m]$, we have
	\begin{align}
		\E_i[F_i(\hat{\bx})] &= \E_i[f_i(\hat{\bx}_A)]\nonumber\\
		&\ge \fP_i(\sE_i) \inf_{\bx \in B_{d_0}(\bz_i, 2\Delta)} \ \{f_{i}(\bx) \}+(1-\fP_i(\sE_i))\inf_{\bx \notin B_{d_0}(\bz_i, 2\Delta)} \ \{ f_{i}(\bx)\}\nonumber\\ 
		&\ge (1-\fP_i(\sE_i))\inf_{\bx \notin B_{d_0} \ (\bz_i, 2\Delta)} \{ f_{i}(\bx)\}\nonumber\\
		&\ge (1-\fP_i(\sE_i))\frac{169}{2} \mu \Delta^2~.\label{eq:opt}
	\end{align}
	Next, we specify the choice of \(m\) and \(\bz_i\). We select \(m\) to ensure that the balls \(B_{d_0}(\bz_i, 2\Delta)\) for \(i \in [m]\) remain disjoint. Recall that \(\bz_i\) is defined as a sequence of elements within \(B_{d_0}(\bzero_A, 5\Delta)\). The largest possible \(m\) ensuring disjointness of these balls corresponds to the packing number of \(B_{d_0}(\bzero_A, 5\Delta)\) with radius \(2\Delta\). A lower bound for this number is provided in Lemma~\ref{lem:cover}, that implies that it suffices to take \(m = \left\lceil \frac{1}{2} \left(\frac{5}{4}\right)^{d_0} \right\rceil\). Also, we select \(\bz_i\) as a sequence of elements such that \(B_{d_0}(\bz_i, 2\Delta) \cap B_{d_0}(\bz_j, 2\Delta) = \emptyset\) for \(i \neq j\).
	
	\noindent To derive an upper bound on \(\fP_i(\sE_i)\) (or equivalently, a lower bound on the probability of \(\neg \sE_i\)), we apply Pinsker’s inequality. This gives  
	\[
	\frac{1}{m} \sum_{i=1}^{m} \fP_i(\sE_i) \le \frac{1}{m} \sum_{i=1}^{m} \Q(\sE_i) + \sqrt{\frac{1}{2m} \sum_{i=1}^{m} \KL\left(\Q, \fP_i\right)}~.
	\]
	Since the events \(\sE_i\) for \(i \in [m]\) are disjoint, it follows that  
	\begin{equation}\label{eq:pins}
		\frac{1}{m} \sum_{i=1}^{m} \bigl(1 - \fP_i(\sE_i)\bigr) \ge 1 - \frac{1}{m} - \sqrt{\frac{1}{2m} \sum_{i=1}^{m} \KL\left(\Q, \fP_i\right)}~.
	\end{equation}
	
	\noindent At this step, we develop an upper bound on the average \(\frac{1}{m} \sum_{i=1}^{m} \KL\left(\Q, \fP_i\right)\). Let \(N_i\) denote the number of times the algorithm queries a point \(\bx\) such that \(\bx_A \in B_{d_0}(\bz_i, 2\Delta)\), $\bI_{d_0}$ denote the $d_0\times d_0$ identity matrix, and $\sN_{d_0}(\cdot, \cdot)$ the normal distribution in dimension $d_0$. Using Lemma~\ref{lem:chain_rule} in the Appendix, we obtain  
	\begin{align*}
		\text{KL}\left(\Q, \fP_i\right)  &\le \E[N_i] \sup_{\bx \in B_{d_0}(\bz_i, 2\Delta)} \KL\left( \sN_{d_0}\left(v(\bx), \frac{\sigma^2}{d_0} \bI_{d_0}\right), \sN_{d_0}\left(\nabla f_i(\bx), \frac{\sigma^2}{d_0} \bI_{d_0}\right)\right)\\
		&\qquad+ \E[T-N_i] \sup_{\bx \notin B_{d_0}(\bz_i, 2\Delta)} \KL\left( \sN_{d_0}\left(v(\bx), \frac{\sigma^2}{d_0} \bI_{d_0}\right), \sN_{d_0}\left(\nabla f_i(\bx), \frac{\sigma^2}{d_0} \bI_{d_0}\right)\right)\\
		&\le \frac{d_0}{2\sigma^2}\mathbb{E}[N_{i}] \sup_{\bx \in B_{d_0}(\bz_i, 2\Delta)} \|v(\bx)- \nabla f_i(\bx)\|^2\\
		&\qquad + \frac{d_0}{2\sigma^2} \mathbb{E}[T-N_i] \sup_{\bx \notin B_{d_0}(\bz_i, 2\Delta)} \|v(\bx)- \nabla f_i(\bx)\|^2.	
	\end{align*}
	Applying the definitions of \(v\) and \(f_i\) to bound \(\|v(\bx) - \nabla f_i(\bx)\|\) in the above expression, we obtain  
	\[
	\KL\left(\Q, \fP_i\right) \le \frac{169^2 d_0 \mu^2 \Delta^2}{2\sigma^2} \E[N_{i}] + \frac{169^2 d_0\tau^2 \mu^2 \Delta^2}{2\sigma^2}T.
	\]  
	Averaging over \(m\) and using the fact that \(\sum_{i=1}^{m} N_i \le T\), which holds due to our choice of the sequence \((\bz_i)\), we derive the bound  
	\begin{equation}\label{eq:kl_avr}
		\frac{1}{m} \sum_{i=1}^{m} \KL\left(\Q, \fP_i\right) \le \frac{169^2 d_0 \mu^2 \Delta^2}{2\sigma^2} \left(\tau^2+\frac{1}{m}\right)T.
	\end{equation}  
	The factor \(\tau^2+\frac{1}{m}\) is of the order \(\tau^2 + \left(\frac{4}{5}\right)^{d_0}\). The term \(\tau^2\) arises from the relative entropy when querying a point \(\bx\) such that \(\|\bx_A - \bz_i\| \ge 2\Delta\), while the term \(\left(\frac{4}{5}\right)^{d_0}\) originates from querying a point near the global minimum, weighted by the probability of selecting a given objective function from the subclass of size \(m\). This underscores the necessity of choosing \(d_0\) on the order of \(\log(2/\tau)\), as this ensures that \(\tau^2\) dominates.  
	
	Combining \eqref{eq:pins} with \eqref{eq:kl_avr} yields a lower bound on the average misidentification error \(\frac{1}{m} \sum_{i=1}^{m} (1-\fP_i(\sE_i))\). This, in turn, leads to the following lower bound when applying \eqref{eq:opt}:  
	\[
	\frac{1}{m} \sum_{i=1}^{m} \E_i[F_i(\hat{\bx})] \ge \frac{169}{2}\mu \Delta^2 \left(1-\frac{1}{m} - \sqrt{\frac{169^2 d_0 \mu^2\Delta^2}{4\sigma^2}\left(\tau^2+\frac{1}{m}\right)T } \right).
	\]  
	Substituting the expressions for \(d_0\) and \(m\) and optimizing this expression with respect to \(\Delta > 0\) leads to the final lower bound.

	\section{Upper Bounds}\label{sec:ub}
	In this section, we establish upper bounds for two classes of \(L\)-smooth functions, those satisfying the \(\mu\)-RSI condition, and those satisfying both the \(\tau\)-QC and \(\mu\)-QG conditions. We show that stochastic gradient descent (SGD) with decreasing step size \(\eta_{t} = \mathcal{O}\bigl(\tfrac{1}{\mu \, t}\bigr)\) and \(\eta_{t} = \mathcal{O}\bigl(\tfrac{1}{\mu \tau \, t}\bigr)\), given access to stochastic gradients of bounded variance \(\sigma^2\), achieves rates of order \(\mathcal{O}\bigl(\tfrac{L\,\sigma^2}{\mu^2 \,T}\bigr)\) for the first class and \(\mathcal{O}\bigl(\tfrac{\sigma^2}{\mu\,\tau^2\,T}\bigr)\) for the second class, respectively. These bounds match the lower bounds in Theorems~\ref{thm:2} and~\ref{thm:1}, up to logarithmic factors in \(L/\mu\) and \(1/\tau\), respectively.
	
	Next, we analyze the problem of minimizing an \(L\)-smooth function satisfying the \(\mu\)-RSI condition in the one-dimensional setting. We propose a novel procedure that, with high probability, converges at a rate of \(\tilde{\mathcal{O}}\bigl(\tfrac{\sigma^2\log(1/\delta)}{\mu\,T}\bigr)\), up to logarithmic factors in \(T\) and other parameters. This result suggests that, in one dimension, the dependence on the condition number \(\kappa\)—which appears in the lower bound when \(d \ge \Omega(\log(2\kappa))\)—can be avoided for \(\mu\)-RSI functions, and that for \(\mu\)-QG and \(\tau\)-QC functions, the dependence on \(1/\tau^2\) can be improved to \(1/\tau\). The proof of Theorem~\ref{thm:ub1} is in Section~\ref{sec:proof_ub1}, and the proof of Theorem~\ref{thm:dim1} is in Section~\ref{sec:proof_dim1}.

	\begin{theorem}\label{thm:ub1}
		Let $f:\R^d \to \R$ be a $L$-smooth function. Run SGD with initial point $\bx_1$ and step sizes $(\eta_t)$ for $T$ iterations, given access to a stochastic gradient oracle with variance bounded $\sigma^2$.
		\begin{itemize}
			\item If $f$ is $\mu-\RSI$ and $\eta_t = \frac{2}{\mu (t+\frac{2L^2}{\mu^2}+1)}$, then
			\[
			\E[f(\hat{\bx}_T)]-f^\star
			\le\frac{\mu^2L^3+2L^5}{\mu^4 T(T+1)}\norm{\bx_1-\bx_1^*}^2+\frac{2L\sigma^2}{\mu^2(T+1)},	
			\]
			where $\bx^*_1$ is the projection of $\bx_1$ on the set of global minimizers. 
			\item If $f$ is $\tau-\QC$ and $\mu-\QG$, and $\eta_t = \frac{4}{\tau \mu \left(t+\frac{16L}{\tau^2\mu}\right)}$, then
			\[
			\E[f(\hat{\bx}_T)]-f^\star \le \frac{\mu\left(1+\frac{16L}{\tau^2\mu}\right)^2}{2T(T+1)}\norm{\bx_1-\bx_1^*}^2+ \frac{16\sigma^2}{\tau^2\mu (T+1)}~.
			\]
		\end{itemize}
		In both cases, $\hat{\bx}_T$ is sampled from $\bx_1, \dots, \bx_T$ with weights $w_t = \frac{2t}{T(T+1)}$.
	\end{theorem}

	Next, we present an algorithm for solving the optimization problem in the class of \( \mu \)-RSI and \( L \)-smooth functions in the one-dimensional case, providing guarantees that hold with high probability given access to a subGaussian oracle. The primary motivation is to demonstrate that the dependency of the optimization error on \( L \) and \( \mu \) can be improved, at least in the one-dimensional setting.
	
	Consider a function \( f:\R \to \R \) that is \( L \)-smooth and \( \mu \)-RSI. Suppose we are given \( D>0 \) such that the interval \( [-D/2, D/2] \) contains at least one minimizer of \( f \). If such prior knowledge is unavailable, one can determine \( D \) using the observation in Remark~\ref{rem:1}. We assume access to \( T \) queries from a \( \sigma \)-subGaussian oracle for the derivatives of \( f \).
	
	Finding a minimizer of \( f \) is equivalent to locating a point \( x \) such that \( f'(x) = 0 \). This suggests formulating the problem as a stochastic root-finding problem, for which a natural approach is to apply a dichotomic search over the interval \( [-D/2, D/2] \) based on estimates of the derivatives. However, due to the stochastic nature of the oracle, the precision of these estimates is subject to deviations of order \( \mathcal{O}(\sqrt{\log(1/\delta)/T}) \).
	
	This raises a key challenge: given a point \( x \) where the derivative estimate is minimal, we cannot rule out the possibility that \( f(x) \) remains large. Indeed, due to the potential non-convexity of \( f \), the derivative \( f' \) may attain large values between \( x \) and a true minimizer \( x^* \), leading to a significant optimization gap \( f(x) - f^* \). This suggests that solving the problem requires not only ensuring that the derivative estimates at the output are small but also verifying that the derivative estimates in the surrounding region are similarly small.
	
	Building on this insight, we develop a procedure performing a dichotomic search over the interval \( [-D/2, D/2] \) using aggregates of the derivative estimates at each iteration point. The connection between the aggregated derivative estimates and the optimization gap is established via the fundamental theorem of calculus. The pseudocode of the algorithm is presented in Section~\ref{sec:proof_dim1}.

	\begin{theorem}\label{thm:dim1}
		Let $f:\R \to \mathbb{R}$ be a $L$-smooth and $\mu$-RSI function.
		For $D,\sigma>0$ and $\delta\in(0,1)$, set
		\[
		K=4\ceil*{\frac{L}{4\mu}},\qquad
		n=\max\left\{1,\ceil*{\log_{4/3}\left(\frac{L\mu D^2T}{192\sigma^2}\right)}\right\}.
		\]
		Suppose that $T\ge 2(K+1)n$, and for some minimizer of $f$ we have $\abs{x^*}\le D/2$. Then, given access to $T$ queries from a $\sigma$-subGaussian oracle, the output of Algorithm~\ref{algo:n1} with input $(T,D,\mu,L,\sigma,\delta)$ satisfies with probability at least $1-\delta$
		\[
		f(a)-f(x^*)\le \frac{1024\sigma^2}{\mu T}\,n\log\left(\frac{2n(K+1)}{\delta}\right)~.
		\]
	\end{theorem}
	\begin{remark}\label{rem:1}
		\begin{itemize}
			\item The requirement that \( D/2 \) serves as an upper bound on a minimizer of \( f \) is not restrictive. One way to ensure such a bound is by allocating a fraction of the total budget \( T \) to querying derivative samples at \( 0 \) and concentrating the empirical estimate. Then, leveraging the RSI property of \( f \), we use \( \abs{x^*} \leq \frac{1}{\mu} \abs{f'(0)} \) for some minimizer \( x^* \).
			\item A direct corollary of Theorem~\ref{thm:dim1} follows from the fact that in one dimension, a function that is \( \tau \)-QC and \( \mu \)-QG is also \( (\mu \tau/2) \)-RSI (see Lemma~\ref{lem:rsi-qc}).
		\end{itemize}
	\end{remark}
	
	The dependency on the problem parameters \( L \) and \( \mu \) is improved compared to the upper bound obtained using SGD. The bound in Theorem~\ref{thm:dim1} can be used to derive an upper bound on the expected optimization gap, for instance, by setting \( \delta = 1/T^2 \) and leveraging the fact that the function \( f \) is bounded on the considered interval. Developing an algorithm for the case of a bounded variance oracle remains an open direction for future work.

	\section{Conclusion}
	
	In this work, we establish new lower bounds for first-order stochastic optimization for several non-convex function classes. Specifically, we derive lower bounds on the number of noisy gradient queries necessary to minimize $L$-smooth functions satisfying quasar-convexity plus quadratic growth or restricted secant inequalities. Our approach leverages a \emph{divergence decomposition} technique, allowing us to construct hard instances that lead to nearly tight lower bounds on optimization complexity.
	
	Despite these advances, several open directions remain for future work. Our lower bounds include logarithmic factors in problem parameters (e.g., $\log(2/\tau)$) and removing these terms remains an open challenge. Finally,
	closing the gap in the Polyak-Łojasiewicz and error bound function classes would provide a more complete understanding of stochastic first-order optimization in non-convex settings.

	\paragraph{Acknowledgments}
	We thank Weiming Ou for helpful comments that led us to clarify the notation and the role of the auxiliary reference measure in our lower-bound proofs.
	\bibliographystyle{plainnat}
	\bibliography{bib_db}
	
	\appendix
	
	\newpage
	\section{Results on Classes Comparisons}
	Let $f:\R^d\to \R$ be such that $f^* := \inf_{\R^d} f > -\infty$.
	\subsection{Definitions of Classes}\label{sec:compare_a_1}     
	
		\begin{definition}
    Suppose that $S$ is the set of minimizers of $f$.
    \begin{itemize}
         \item We say that $f$ is $\mu$-strongly convex if for all $\bx,\by\in\R^d$
            \[
    		f(\by)\ge f(\bx) +\langle \nabla f(\bx), \by-\bx \rangle+\frac{\mu}{2}\norm{\by-\bx}^2~.
    	\]      
        \item We say that $f$ is $\mu$-star strongly-convex if for all $\bx$
            \[
    		f(\bx_p)\ge f(\bx) +\langle \nabla f(\bx), \bx_p-\bx \rangle+\frac{\mu}{2}\norm{\bx_p-\bx}^2,
    	\]
        where $\bx_p$ is $\text{Proj}_S(\bx)$.
        \item We say that $f$ satisfies $\mu$-Error-Bound (EB) if for all $\bx$
            \[
    		\norm{\nabla f(\bx)}\geq \mu\norm{\bx-\bx_p},
    	\]
        where $\bx_p$ is $\text{Proj}_S(\bx)$.
        \item We say that $f$ satisfies $\mu$-Polyak-Łojasiewicz (PL) inequality if for all $\bx$
            \[
    		\frac{1}{2}\norm{\nabla f(\bx)}^2\ge \mu(f(\bx)-f(\bx_p))~.
    	\]

        \item We say that $f$ satisfies $(\tau,\mu)$-Strongly-Quasar-Convexity if for all $\bx\in\R^d$
        \[
		f(\bx)-f^* \le \frac{1}{\tau}\langle \nabla f(\bx), \bx-\bx_p \rangle-\frac{\mu}{2}\norm{\bx_p-\bx}^2,
		\]   
        where $\bx_p$ is $\text{Proj}_S(\bx)$.        
    \end{itemize}
    \end{definition}

    \subsection{Proof of Lemma~\ref{lem:compare}}\label{sec:compare_a_2}

\begin{figure}
    \centering
    \includegraphics[width=0.9\linewidth]{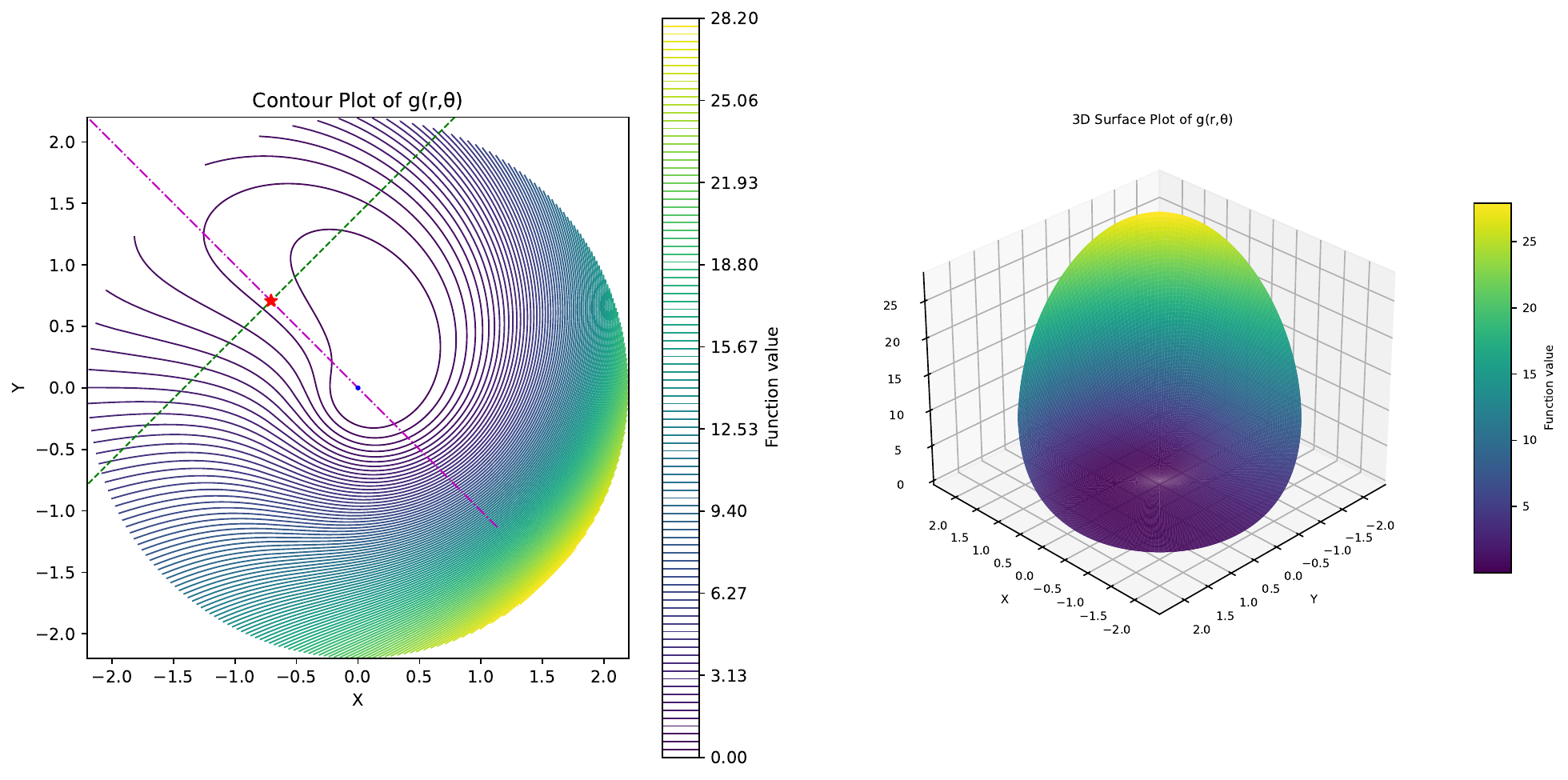}
    \caption{Contour lines (left) and surface plot (right) of $f$. At the red point $p=(-\sqrt{2}/2,\sqrt{2}/2)$, the gradient is perpendicular to the displacement from the minimizer at the origin.}
    \label{fig:eb_not_rsi}
\end{figure}

\begin{proof}
 \begin{itemize}
\item The first statement directly follows from Cauchy–Schwarz inequality:
			\[
				\norm{\nabla f(\bx)}\norm{\bx-\bx_p}\geq\langle \nabla f(\bx),\bx-\bx_p\rangle\geq\mu \norm{\bx-\bx_p}^2~.
			\]
			\item  To prove the strict inclusion and statement three, we consider the following function in $\mathbb{R}^2$:
			\begin{equation}\label{eq:counter}
				f(x,y)=\Big(\sqrt{2}\sqrt{x^2+y^2}+x\sin\big(\log\big(\sqrt{x^2+y^2}\big)\big)-y\cos\big(\log\big(\sqrt{x^2+y^2}\big)\big)\Big)^2,\qquad f(0,0)=0 .
			\end{equation}
			In polar coordinates $(x,y)=(r\cos\theta,r\sin\theta)$, we have $f=g(r,\theta):=r^2B^2$, where $\varphi:=\log r-\theta$ and $B:=\sqrt{2}+\sin\varphi\ge\sqrt{2}-1>0$.
			
			Since $(\sqrt{2}-1)^2r^2\le f\le(\sqrt{2}+1)^2r^2$, the origin is the unique minimizer, $f$ is differentiable at the origin with $\nabla f(\boldsymbol{0})=\boldsymbol{0}$, and $f$ is smooth on $\mathbb{R}^2\setminus\{\boldsymbol{0}\}$.
			
			We have $g_r=2rB(B+\cos\varphi)$ and $g_\theta = -2r^2B \cos\varphi$. Using $xf_x+yf_y=rg_r$ and $f_x^2+f_y^2=g_r^2+g_\theta^2/r^2$, we get, for $\boldsymbol{x} \neq \boldsymbol{0}$,
			\[
			\frac{\langle\nabla f(\boldsymbol{x}),\boldsymbol{x}\rangle}{\|\boldsymbol{x}\|^2}=2B\,(B+\cos\varphi),
			\qquad
			\frac{\|\nabla f(\boldsymbol{x})\|}{\|\boldsymbol{x}\|}=2B\sqrt{(B+\cos\varphi)^2+\cos^2\varphi}=:E(\varphi) .
			\]
			Since $B+\cos\varphi=\sqrt{2}\,\big(1+\sin(\varphi+\pi/4)\big)$, the first ratio vanishes when $\varphi=5\pi/4$, e.g., at $r=1$, $\theta=3\pi/4$ (red point in Figure~2). Hence $f$ is not $a$-RSI for any $a>0$.
			The function $E$ is continuous, $2\pi$-periodic, and never vanishes: if $\cos\varphi=0$, then $B+\cos\varphi=B>0$. Hence $c:=\min_{\varphi}E(\varphi)>0$, and $f$ is $c$-EB on $\mathbb{R}^2$. Finally, $(\mu/c)f$ is $\mu$-EB and not $a$-RSI for any $a>0$, which proves statement three and $\mu$-RSI $\subsetneq$ $\mu$-EB.

\item For statement two, suppose $f$ is $\tau$-quasar-convex and $\mu$-QG, that is 
			\[f^\star-f(x)\geq \frac{1}{\tau}\langle \nabla f(x),x^\star-x\rangle~,\]
			\[f(x)-f^\star\geq \frac{\mu}{2}\norm{x-x^\star}^2~.\]
			Then 
			\begin{align*}
				\frac{1}{2}(f^\star-f) &= f^\star-f(x)+\frac{1}{2}(f(x)-f^\star)\\
				&\geq \frac{1}{\tau}\langle \nabla f(x),x^\star-x\rangle +\frac{\mu}{4}\norm{x-x^\star}^2~,
			\end{align*}
			which shows $f$ is $(\frac{\tau}{2},\frac{\mu}{2})$-strongly quasar-convex. The inclusion is strict: $f(x)=\mu x^2/4$ is $(\tau/2,\mu/2)$-SQC for $\tau\in(0,1]$, but is not $\mu$-QG. 
			\item We consider an instance of the function used in the lower bound proof of Theorem~\ref{thm:2} in dimension $1$. Its expression is given by
			\[
			f(x) =
			\begin{cases}
				\frac{3}{2}x^2, & \text{if } \abs{x} < 1, \\[1em]
				-\frac{3}{2}x^2+6\abs{x}-3, & \text{if } 1 \leq \abs{x} < \frac{3}{2}, \\[1em]
				\frac{3}{2}+\frac{1}{2}x^2, & \text{otherwise }~.
			\end{cases}
			\]
			$f$ is $1$-RSI, has a minimum value $0$ attained at the origin. At the point $x=3/2$ we have $f(3/2) = \frac{21}{8}$ and $f^\prime(3/2)3/2 = 9/4$, which shows that it cannot be star strongly convex at all. Multiplying this function by $\mu$ gives the stated example for every $\mu>0$.
		\end{itemize}
	\end{proof}

	\section{Complete Proof of Theorem~\ref{thm:1}}\label{sec:proof1}
	Let \(L, \mu > 0\), \(\tau \in (0,1]\), and \(\sigma > 0\). Suppose \(\kappa := \tfrac{L}{\mu} \ge 202\) and \(d \ge 3 \log_{5/4}\!\bigl(2/\tau\bigr)\), let $d_0 := \lceil 2\log_{5/4}(2/\tau)\rceil$.
	
	\paragraph{Additional Notation.} For any integer \(n \ge 1\), let \(\mathcal{F}_n\) denote the class of real-valued functions on \(\mathbb{R}^n\) that are \(L\)-smooth, $\tau-\QC$, and satisfy the $\mu-\QG$. Let \(A := \{1, \dots, d_0\}\) and \(\bar{A} := [d] \setminus A\). For any \(\bx \in \mathbb{R}^d\), let \(\bx_A \in \mathbb{R}^{d_0}\) be the vector of its first \(d_0\) coordinates, and \(\bx_{\bar{A}} \in \mathbb{R}^{d - d_0}\) be the vector of its remaining coordinates. For two 
	vectors \(\mathbf{a} \in \mathbb{R}^{d_0}\) and \(\mathbf{b} \in \mathbb{R}^{d - d_0}\), we write \((\mathbf{a}, \mathbf{b})\) to denote the vector in \(\mathbb{R}^d\) whose first \(d_0\) components are \(\mathbf{a}\) and last \(d - d_0\) components are \(\mathbf{b}\).
	
	\subsection{Functions Construction.} For \(d \ge 3\log_{5/4}(2/\tau)\), we construct a function \(F: \R^d \to \R\) that depends only on the first \(d_0= \lceil 2\log_{5/4}(2/\tau)\rceil\) coordinates. Concretely, there is a function \(f: \R^{d_0} \to \R\) such that \(F(\bx) = f(\bx_A)\). Lemma~\ref{lem:dim} shows that if $f$ is in $\sF_{d_0}$, then $F$ is in $\sF_d$.
	
	\noindent Let \(\Delta > 0\) be a parameter to be specified later. Let $a$ be defined by
	\[
	a := \frac{-\tau +\sqrt{2\tau^2-4\tau+4}}{2-\tau}~.
	\]
	Since $\tau \in (0, 1]$, we have $a \in [\sqrt{2}-1, 1]$ and $1-a\le \tau$.
	
	Let \(m \ge 2\) denote the size of the function subclass we are constructing. We pick \(m\) elements \(\bz_1, \dots, \bz_m\) in \(B_{d_0}(\bzero_A, 5\Delta)\), and define a set of functions \(f_{1}, \dots, f_{m}\) so that each \(f_{i}\) belongs to \(\sF_{d_0}\). For each \(i \in [m]\), let \(f_{i}\) be a function whose value at \(\bz_i\) is zero and whose gradient for all \(\bx \in \R^{d_0}\) is given by
	\[
	\nabla f_{i}(\bx) =
	\begin{cases}
		\displaystyle 169\mu(\bx - \bz_i), & \text{if } \|\bx - \bz_i\| < \Delta, \\[1em]
		\displaystyle -169\mu\Bigl(\bx - \bz_i - \Delta \frac{\bx - \bz_i}{\|\bx - \bz_i\|}\Bigr) + 169\mu\Delta \frac{\bx - \bz_i}{\|\bx - \bz_i\|}, & \text{if } \Delta \leq \|\bx - \bz_i\| < (1+a)\Delta, \\[1em]
		\displaystyle 169\mu(1-a)\Delta  \frac{\bx - \bz_i}{\|\bx - \bz_i\|} , & \text{if } \bx \in B_{d_0}(\boldsymbol{0}, 8\Delta)\setminus B_{d_0}(\bz_i, (1+a)\Delta), \\[1em]
		\displaystyle  169\mu\left(\bx-8\Delta \frac{\bx}{\norm{\bx}}\right)+169\mu(1-a)\Delta  \frac{\bx - \bz_i}{\|\bx - \bz_i\|}, & \text{if } \norm{\bx} > 8\Delta~.
	\end{cases}
	\] 
	Let $r:= \norm{\bx-\bz_i}$ and define
	\[
	f_{i}(\bx) \;=\;
	\begin{cases}
		\displaystyle
		\frac{169\,\mu}{2}\,r^2,
		& \text{if } r < \Delta,
		\\[1.2em]
		\displaystyle
		-\,\frac{169\,\mu}{2}\,r^2
		\;+\;
		338\,\mu\,\Delta\,r
		\;-\;
		169\,\mu\,\Delta^2,
		& \text{if } \Delta \;\le\; r < (1+a)\,\Delta,
		\\[1.2em]
		\displaystyle
		169\,\mu\,(1-a)\,\Delta\,r
		\;+\;
		\frac{169\,\mu}{2}\,\Delta^2\,\bigl(a^2 + 2\,a - 1\bigr),
		& \text{if } \bx \in B_{d_0}(\boldsymbol{0}, 8\Delta)\setminus B_{d_0}(\bz_i, (1+a)\Delta),
		\\[1.2em]
		\displaystyle
		169\mu\,\Bigl(\tfrac12\,\|\bx\|^2 - 8\,\Delta\,\|\bx\| +32\Delta^2\Bigr)\\
		\;+\;
		169\,\mu\,(1-a)\,\Delta\,r
		\;+\;
		\frac{169\,\mu}{2}\,\Delta^2\,\bigl(a^2 + 2\,a - 1\bigr),
		& \text{if } \|\bx\| > 8\,\Delta.
	\end{cases}
	\]
	
	Lemma below shows that the constructed functions $f_{\bz_i}$ are in $\sF_{d_0}$.
	\begin{lemma}\label{lem:f}
		The functions $(f_{i})_{i\in [m]}$ are $\tau$-WQC, $\mu$-QG and $L$-smooth.
	\end{lemma}
	\begin{proof}
		Let $i \in [m]$, in the following, we will use that $\tau \in (0,1]$, therefore following the expression of $a$ we have
		\begin{equation*}
			1 - \frac{\tau}{2} - \tau a - \Bigl(1 - \frac{\tau}{2}\Bigr)a^2=0 \qquad \text{ and } \quad  a \in [\sqrt{2}-1,1]~.
		\end{equation*}
		
		\noindent \textbf{Verifying $L$-smoothness:}
		First, we will show that $\nabla f_{i}$ is $L$-Lipschitz on each region in its expression, then we will conclude using Lemma~\ref{lem:lip}. We have
		\begin{itemize}
			\item Case $\bx \in B_{d_0}(\bz_i, \Delta)$: since $L \ge 169\mu$, the expression of $\nabla f_{i}$ shows that it is $L$-Lipschitz.
			\item Case $\bx \in B_{d_0}(\bz_i, (1+a)\Delta)\setminus B_{d_0}(\bz_i, \Delta)$: we have $1+a \le 2$, therefore using Lemma~\ref{lem:lip2} $\nabla f_{i}$ is $L$-Lispchitz.
			\item Case $\bx \in B_{d_0}(\boldsymbol{0},8\Delta)\setminus B_{d_0}(\bz_i, (1+a)\Delta)$: recall that for $\bx \notin B_{d_0}(\bz_i, (1+a)\Delta)$ we have that $\bx \to \Delta \frac{\bx-\bz_i}{\norm{\bx-\bz_i}}$ is the projection of $\bx$ onto $B_{d_0}(\bz_i, \Delta)$ minus its center $\bz_i$. Therefore, $\nabla f_{i}$ is $L$-Lipschitz.
			\item Case $\bx \notin B_{d_0}(\boldsymbol{0}, 8\Delta)$: We first calculate the Hessian divided by $169\mu$:
			\[
			\bI\left(1-\frac{8\Delta}{\norm{\bx}}\right)+\frac{8\Delta}{\norm{\bx}^3}\bx{\bx}^{\mathrm{T}}+(1-a)\Delta\left(\frac{\bI}{\norm{\bx-\bz_i}}-\frac{(\bx-\bz_i)(\bx-\bz_i)^{\mathrm{T}}}{\norm{\bx-\bz_i}^3}\right)~.
			\]
			The matrix $H_1=I(1-\frac{8\Delta}{\norm{\bx}})+\frac{8\Delta}{\norm{\bx}^3}\bx{\bx}^{\mathrm{T}}$ has eigenvalues $1$ and $1-8\Delta/\norm{\bx}$.\\
			The matrix $H_2=(1-a)\Delta(\frac{I}{\norm{\bx-\bz_i}}-\frac{(\bx-\bz_i)(\bx-\bz_i)^{\mathrm{T}}}{\norm{\bx-\bz_i}^3})$ has eigenvalues $\frac{(1-a)\Delta}{\norm{\bx-\bz_i}}$ and $0$.
			Since $169\mu(1+\frac{(1-a)\Delta}{\norm{\bx-\bz_i}})\leq 202\mu\leq L$, it is $L$-Lipschitz.

		\end{itemize}
		The conclusion follows from the fact that $\nabla f_{i}$ is continuous and Lemma~\ref{lem:lip}.
		
		\noindent \textbf{Verifying $\tau$-WQC:}
		Observe that the minimizer of $f_{i}$ is $\bz_i$ and the minimum is $0$. We will show that $f_{i}$ satisfies $\tau$-WQC on each of the four regions in the definition of $\nabla f_{i}$. 
		\begin{itemize}
			\item Let $\bx \in B_{d_0}(\boldsymbol{0}, 8\Delta)$. To ease notation define $r:= \norm{\bx-\bz_i}$.
			Then, the expression of $f_i$ is given by
			\[
			f_{i}(\bx) \;=\;
			\begin{cases}
				\displaystyle
				\frac{169\,\mu}{2}\,r^2,
				& \text{if } r < \Delta,
				\\[1.2em]
				\displaystyle
				-\,\frac{169\,\mu}{2}\,r^2
				\;+\;
				338\,\mu\,\Delta\,r
				\;-\;
				169\,\mu\,\Delta^2,
				& \text{if } \Delta \;\le\; r < (1+a)\,\Delta,
				\\[1.2em]
				\displaystyle
				169\,\mu\,(1-a)\,\Delta\,r
				\;+\;
				\frac{169\,\mu}{2}\,\Delta^2\,\bigl(a^2 + 2\,a - 1\bigr),
				& \text{if } \bx \in B_{d_0}(\boldsymbol{0}, 8\Delta)\setminus B_{d_0}(\bz_i, (1+a)\Delta)~.
			\end{cases}
			\]
			Therefore, we have
			\begin{align*}
				\langle &\nabla f_{i}(\bx), \bx - \bz_i \rangle - \tau f_{i}(\bx) \\
				&=\begin{cases}
					\displaystyle
					169\mu r^2 \left(1 - \frac{\tau}{2}\right), & \text{if } r < \Delta, \\[1em]
					\displaystyle
					169\mu \left[\tau \Delta^2 + 2\Delta r (1 - \tau) + r^2 \left(\frac{\tau}{2} - 1\right)\right], & \text{if } \Delta \leq r < (1+a)\Delta,\\
					169 \mu (1-a)\Delta r (1-\tau) - \frac{169 \tau \mu \Delta^2}{2}(a^2+2a-1), & \text{if }  \bx \in B_{d_0}(\boldsymbol{0}, 8\Delta)\setminus B_{d_0}(\bz_i, (1+a)\Delta)~.
				\end{cases}
			\end{align*}
			
			Since $\tau \le 1$, we have that $\langle \nabla f_{i}(\bx), \bx-\bz_i \rangle -\tau f_{i}(\bx) \ge 0$ in the first region. For the second region, observe that $Q:r \to \tau \Delta^2 + 2\Delta r (1 - \tau) + r^2 \left(\frac{\tau}{2} - 1\right)$ is increasing $(-\infty, \frac{1-\tau}{1-\frac{\tau}{2}}\Delta]$ and decreasing on $[\frac{1-\tau}{1-\frac{\tau}{2}}\Delta, +\infty)$, therefore for $r \in [\Delta, (1+a)\Delta]$, we have:
			\[
			Q(r) \ge \min\{Q(\Delta), Q((1+a)\Delta)\}. 
			\]
			We have 
			\begin{align*}
				Q(\Delta) &= \tau \Delta^2 + 2\Delta^2 (1 - \tau) + \Delta^2 \left(\frac{\tau}{2} - 1\right)\\
				&= \left( 1-\frac{\tau}{2}\right)\Delta^2 \ge 0,
			\end{align*}
			and 
			\begin{align*}
				Q((1+a)\Delta) &= \tau \Delta^2 + 2\Delta^2 (1+a) (1 - \tau) + (1+a)^2 \Delta^2\left(\frac{\tau}{2} - 1\right)\\
				&= \Delta^2 \left(1 - \frac{\tau}{2} - \tau a - \Bigl(1 - \frac{\tau}{2}\Bigr)a^2\right)\\
				&=0~.
			\end{align*}
			Therefore $Q(r) \ge 0$ for $r \in [\Delta, (1+a)\Delta]$, and
			for any $\bx$ in the second region we have
			\begin{align*}
				\langle \nabla f_{i}(\bx), \bx - \bz_i \rangle - \tau f_{i}(\bx) = 169\mu Q(r)
				\ge0~.
			\end{align*}
			For the third region, observe that since $\nabla f_{i}$ is continuous, so is $f_{i}$ is $\langle\nabla f_{i}(\bx), \bx - z \rangle - \tau f_{i}(\bx)$. Moreover, $\langle\nabla f_{i}(\bx), \bx - z \rangle - \tau f_{i}(\bx)$ is increasing with respect to $r$ in the third region, therefore $\langle\nabla f_{i}(\bx), \bx - z \rangle - \tau f_{i}(\bx) \ge 0$, which shows that $f_{i}$ is $\tau$-WQC in the third region.
			
			\item Let $\bx \notin B_{d_0}(\boldsymbol{0}, 8\Delta)$. We have
			\begin{align*}
				\bigl\langle &\nabla f_{i}(\bx),\,\bx - \bz_i \bigr\rangle 
				\;-\;
				\tau\,f_{i}(\bx)\\
				&= \left\langle 169\mu\left(\bx-8\Delta \frac{\bx}{\norm{\bx}}\right) + 169\mu(1-a)\Delta  \frac{\bx - \bz_i}{\|\bx - \bz_i\|} , \bx-\bz_i\right\rangle\\
				&\qquad  - \tau \left( 169\mu \left[\frac{1}{2}\norm{\bx}^2-8\Delta \norm{\bx}+32\Delta^2\right]\right)\\
				&\qquad -\tau\left(  169\mu (1-a)\Delta \norm{\bx-\bz_i}+\frac{169\mu}{2}\Delta^2(a^2+2a-1)\right)\\
				&= \underbrace{\left\langle 169\mu\left(\bx-8\Delta \frac{\bx}{\norm{\bx}}\right), \bx-\bz_i \right\rangle -169\tau  \mu \left[\frac{1}{2}\norm{\bx}^2-8\Delta \norm{\bx}+32\Delta^2\right]}_{\text{Term 1}}\\
				& + \underbrace{\left\langle 169\mu(1-a)\Delta  \frac{\bx - \bz_i}{\|\bx - \bz_i\|}, \bx-\bz_i \right\rangle -169\mu\tau \left(    (1-a)\Delta \|\bx-\bz_i\|+\frac{\Delta^2}{2}(a^2+2a-1)\right)}_{\text{Term 2}}~.
			\end{align*}
			We have
			\begin{align*}
				\text{Term 1} &=169\mu
				\left[
				(\|\bx\|^2 - 8\Delta\|\bx\|)
				- \left\langle \bx-8\Delta \frac{\bx}{\norm{\bx}}, \bz_i \right\rangle
				- \tau\bigl(\tfrac{1}{2}\|\bx\|^2 - 8\Delta\|\bx\| + 32\Delta^2)
				\right]\\
				&\ge 169\mu
				\Biggl[
				\bigl(\|\bx\|^2 - 8\Delta\|\bx\|\bigr)
				- \norm{\bz_i} \norm{\bx-8\Delta \frac{\bx}{\norm{\bx}}}
				- \tau\bigl(\tfrac{1}{2}\|\bx\|^2 - 8\Delta\|\bx\| + 32\Delta^2\bigr)
				\Biggr]\\
				&\ge 169\mu
				\Biggl[
				\bigl(\|\bx\|^2 - 8\Delta\|\bx\|\bigr)
				- 5\Delta \left(\norm{\bx}-8\Delta \right)
				- \tau\bigl(\tfrac{1}{2}\|\bx\|^2 - 8\Delta\|\bx\| + 32\Delta^2\bigr)
				\Biggr]\\
				&= 169\mu
				\Biggl[\bigl(1 - \tfrac{\tau}{2}\bigr)\|\bx\|^2+
				\Delta\bigl(-13 + 8\tau\bigr)\|\bx\|+
				\bigl(40 - 32\tau\bigr)\Delta^2\Biggr],
			\end{align*}
			where in the second inequality we used the fact that $\|\bz_i\|\leq 5 \Delta$.
			The last expression is positive for $\norm{\bx} \ge 8\Delta$ (the two roots of the quadratic above are $ 8\Delta$ and $ \frac{10-8\tau}{2-\tau}\Delta\le 5\Delta$).
			
			For the second term, we have
			\begin{align*}
				\text{Term 2} &= 169 \left[ 
				\mu (1-a) \, \Delta \, \|\bx - \bz_i\| \, (1 - \tau) 
				- \frac{\tau}{2} \, \mu\Delta^2 \, (a^2 + 2a - 1) 
				\right]\\
				&\ge  169\mu \Delta^2\left[ 
				3 (1-a) \,  \,  (1 - \tau) 
				- \frac{\tau}{2} \,  \, (a^2 + 2a - 1) 
				\right],
			\end{align*}
			where we used $\norm{\bx-\bz_i} \ge \norm{\bx}-\norm{\bz_i}\ge 3\Delta$.
			Using the defining quadratic for $a$, we have $3(1-a)(1-\tau)-\frac{\tau}{2}(a^2+2a-1)=(2-a)(1-a)(1-\tau)\ge0$ for $\tau\in(0,1]$. As a conclusion, if $\bx \notin B_{d_0}(\boldsymbol{0}, 8\Delta)$, we have $\langle \nabla f_{i}(\bx), \bx-\bz_i \rangle - \tau f_{i}(\bx) \ge 0$. Therefore $f_{i}$ is $\tau$-WQC.
		\end{itemize}
		
		\noindent \textbf{Verifying $\mu$-QG:}
		For any $\bx \in \R^{d_0}$, we want to show that $f_{i}(\bx) \ge \frac{\mu}{2}\norm{\bx-\bz_i}^2$. Let $\bu := \frac{\bx-\bz_i}{\norm{\bx-\bz_i}}$,
		let $h:\R_{\geq 0}\to \R$ defined as $h(t) := f_{i}(\bz_i+t\cdot \bu)$. Therefore, for any $\bx \in \R^{d_0}: f_{i}(\bx)-\frac{\mu}{2}\norm{\bx-\bz_i}^2 = h(\norm{\bx-\bz_i})-\frac{\mu}{2}\norm{\bx-\bz_i}^2$. The expression of the derivative of $h$ is $h'(t) = \langle \bu, \nabla f_{i}(\bz_i+t\cdot \bu) \rangle$. Let $t_0$ be the (unique) number such that $\norm{\bz_i+t_0 \cdot \bu}=8\Delta$. Then, we have
		\[
		h'(t) =
		\begin{cases}
			\displaystyle 
			169\,\mu\,t,  
			& \text{if } 0 \;\le\; t < \Delta, \\[8pt]
			\displaystyle
			169\,\mu\,\bigl(2\,\Delta - t\bigr),
			& \text{if } \Delta \;\le\; t < (1+a)\,\Delta, \\[8pt]
			\displaystyle
			169\,\mu\,(1-a)\,\Delta,
			& \text{if } (1+a)\,\Delta \;\le\; t \;\le\; t_0, \\ 
			\displaystyle
			169\mu\,\left(1 - \frac{8\Delta}{\|\bz_i + t\,\bu\|}\right)\,\bigl\langle \bu,\;\bz_i + t\,\bu\bigr\rangle
			\;+\;
			169\,\mu\,(1-a)\,\Delta,
			& \text{if } t > t_0~.
		\end{cases}
		\]
		It is easy to see that the QG is verified in $B_{d_0}(\bz_i, \Delta)$, hence we only have to consider the other regions.
		
		\begin{itemize}
			\item Suppose that $\bx \in B_{d_0}(\bz_i, (1+a)\Delta)\setminus B_{d_0}(\bz_i, \Delta)$. Therefore, $\Delta \le \norm{\bx-\bz_i}\le (1+a)\Delta$. From the continuity of $h'$, we have
			\begin{align*}
				f_{i}(\bx) &= h(\norm{\bx-\bz_i})-h(0)\\
				&= \int_{0}^{\norm{\bx-\bz_i}} h'(s) \diff s\\
				&= \int_{0}^{\Delta} 169\mu s \diff s + \int_{\Delta}^{\norm{\bx-\bz_i}} 169\mu (2\Delta-s)\diff s\\
				&\ge \frac{169}{2}\mu \Delta^2 \ge \frac{\mu}{2}\norm{\bx-\bz_i}^2, 
			\end{align*}
			where we used $a\le 1$, therefore $2\Delta-s \ge 0$ for $s \in [\Delta, \norm{\bx-\bz_i}]$.
			\item Suppose that $\bx \in B_{d_0}(\boldsymbol{0}, 8\Delta)\setminus B_{d_0}(\bz_i, (1+a)\Delta)$. Therefore, $(1+a)\Delta \le \norm{\bx-\bz_i} $ and $\norm{\bx}\le 8\Delta$. Recall that since $\norm{\bx} = \norm{\bz_i+\norm{\bx-\bz_i}\bu} \le 8\Delta$, by the definition of $t_0$ we have $\norm{\bx-\bz_i}\le t_0$. We have
			\begin{align*}
				f_{i}(\bx) &= h(\norm{\bx-\bz_i})-h(0)\\
				&= \int_{0}^{\norm{\bx-\bz_i}} h'(s)\diff s\\
				&= \int_{0}^{\Delta} 169\mu s \diff s + \int_{\Delta}^{(1+a)\Delta} 169\mu (2\Delta-s)\diff s+ \int_{(1+a)\Delta}^{\norm{\bx-\bz_i}}h'(s) \diff s~.
			\end{align*} 	
			From the expression of $h'$ we have  $h'(s) \ge 0$ for $s \in [(1+a)\Delta, \norm{\bx-\bz_i}]$. Therefore, we obtain
			\begin{align}
				f_{i}(\bx) 
				&\geq \int_{0}^{\Delta} 169\mu s \diff s + \int_{\Delta}^{(1+a)\Delta} 169\mu (2\Delta-s)\diff s \nonumber \\
				&= \frac{169}{2}\mu \Delta^2 + 169\mu \Delta^2 \left(a-\frac{a^2}{2}\right)\nonumber\\
				&\ge \frac{169}{2}\mu \Delta^2 + 169(2\sqrt{2}-\frac{5}{2})\mu \Delta^2 \nonumber\\
				&\ge \frac{\mu}{2} (13\Delta)^2 + 55\mu \Delta^2 \nonumber \\
				&\ge \frac{\mu}{2} \norm{\bx-\bz_i}^2+55\mu \Delta^2,\label{eq:t0a}
			\end{align}
			where we used $a-\frac{a^2}{2}\ge 2\sqrt{2}-\frac{5}{2}$ for $a\in [\sqrt{2}-1, 1]$, and the fact that $\norm{\bx-\bz_i} \le \norm{\bx}+\norm{\bz_i} \le 13\Delta$.
			\item 	Suppose now that $\norm{\bx}> 8\Delta$. Recall that $t_0$ is the positive number such that $\norm{\bz_i+t_0 \bu} = 8\Delta$. Observe that $\norm{\bz_i+t_0 \bu} \le \norm{\bz_i}+t_0 \le 5\Delta +t_0$, therefore $t_0 \ge 3\Delta$. Moreover, $t_0 - \norm{\bz_i} \le \norm{\bz_i+t_0\bu}$, therefore $t_0 \le 13\Delta$. We conclude that $t_0 \in [3\Delta, 13\Delta]$. For any $t \ge t_0$, we have
			\begin{align}
				h'(t)-\mu t 
				&= 169\mu\left(1-\frac{8\Delta}{\norm{\bz_i+t\bu}}\right)\langle \bu, \bz_i+t\bu \rangle+169\mu (1-a)\Delta -\mu t\nonumber\\
				&= 169\mu\left(1-\frac{8\Delta}{\norm{\bz_i+t\bu}}\right)\left(\langle \bu, \bz_i \rangle +t\right)+169\mu (1-a)\Delta -\mu t\nonumber~.		
			\end{align}
			Recall that $169\mu (1-a)\Delta \ge 0$,  $\norm{\bx} = \norm{\bz_i+t\bu} \ge 8\Delta$ and $\langle \bu, \bz_i\rangle  \ge -\norm{\bz_i} \ge -5\Delta$. Therefore, we obtain
			\begin{align*}
				h'(t)-\mu t &\ge 169\mu\left(1-\frac{8\Delta}{\norm{\bz_i+t\bu}}\right)\left(t-5\Delta\right) -\mu t\\
				&= 169\mu\left(1-\frac{8\Delta}{\norm{\bz_i+t_0\bu + (t-t_0)\bu}}\right)\left(t-5\Delta\right) -\mu t\\
				&= 169\mu\left(1-\frac{8\Delta}{\sqrt{\norm{\bz_i+t_0\bu}^2+(t-t_0)^2+2(t-t_0)\langle \bu, \bz_i+t_0\bu \rangle }}\right)\left(t-5\Delta\right) -\mu t~.
			\end{align*}

			From the definition of $t_0$, we have $\norm{\bz_i+t_0\bu} = 8\Delta$, therefore $\norm{\bz_i}^2+t_0^2+2t_0\langle \bz_i, \bu \rangle = 64\Delta^2$. Using $\norm{\bz_i} \le 5\Delta$ and $t_0 \ge 0$, we get: $\frac{t_0}{2}+\langle \bu, \bz_i \rangle \ge 0$. Hence, 
			we have
			\[
			\langle \bu, \bz_i + t_0 \bu\rangle 
			= \langle \bu, \bz_i\rangle + t_0
			\geq \langle \bu, \bz_i\rangle + t_0/2 \geq 0~.
			\]
			Therefore, for any $t\geq\max\{t_0,5\Delta\}$, we have
			\begin{align}
				h'(t)-\mu t 
				&\ge 169\mu\left(1-\frac{8\Delta}{\sqrt{\norm{\bz_i+t_0\bu}^2+(t-t_0)^2+2(t-t_0)\langle \bu, \bz_i+t_0\bu \rangle }}\right)\left(t-5\Delta\right) -\mu t \nonumber \\
				&\ge 169\mu\left(1-\frac{8\Delta}{\sqrt{\norm{\bz_i+t_0\bu}^2+(t-t_0)^2}}\right)\left(t-5\Delta\right) -\mu t \nonumber\\
				&= 169\mu\left(1-\frac{8\Delta}{\sqrt{64\Delta^2+(t-t_0)^2}}\right)\left(t-5\Delta\right) -\mu t~.\label{eq:int1a}
			\end{align}
			We distinguish between the following cases:
			\begin{itemize}
				\item If $t \in [t_0, t_0+3\Delta]$, recall that we have $h'(t) -\mu t\ge -\mu t$. So, if we take the integral between $t_0$ and $t$, we get
				\[
				h(t)- \frac{\mu}{2}t^2 \ge h(t_0)-\frac{\mu}{2}t_0^2 +\int_{t_0}^t (-\mu s)\diff s~.                
				\]
				Recall that \eqref{eq:t0a} gives $h(t_0)-\frac{\mu}{2}t_0^2 = f_{\bz_i}(\bz_i+t_0 \bu)-\frac{\mu}{2}t_0^2 \ge 55\mu \Delta^2$. Therefore, the inequality above gives
				\begin{align}
					h(t)- \frac{\mu}{2}t^2 &\ge h(t_0)-\frac{\mu}{2}t_0^2 +\int_{t_0}^t (-\mu s)\diff s \nonumber\\
					&\ge 55\mu \Delta^2 -\frac{\mu}{2}(t^2-t_0^2)\nonumber\\
					&\ge  55\mu \Delta^2 -\frac{\mu}{2}\left(\left(t_0+3\Delta\right)^2-t_0^2\right)\nonumber\\
					&=55\mu \Delta^2 -\frac{\mu}{2}(9\Delta^2+6 t_0 \Delta)\nonumber\\
					&\ge  55\mu \Delta^2 -\frac{\mu 87}{2}\Delta^2
					\ge 0,\label{eq:int2a}
				\end{align} 
				where we used $t_0 \leq 13\Delta$ in the second to last inequality.
				\item If $t \ge t_0+3\Delta$, given that $t_0 \ge 3\Delta$, we have $t \ge 6\Delta$. So, using \eqref{eq:int1a}, we have for any $t \ge t_0+3\Delta$
				\begin{align*}
					h'(t)-\mu t &\ge 169\mu\left(1-\frac{8}{\sqrt{73}}\right)\left(t-5\Delta\right) -\mu t\\
					&\ge 10 \mu (t-5\Delta)-\mu t\\
					&= 9\mu t -50 \mu \Delta \ge 4\mu \Delta > 0~.
				\end{align*}
				Therefore, we have
				\begin{align*}
					h(t)-\frac{\mu}{2}t^2 
					&=  h(t_0+3\Delta)-\frac{\mu}{2}(t_0+3\Delta)^2 + \int_{t_0+3\Delta}^{t} \left( h'(s) -\mu s \right)\diff s\\
					&\ge h(t_0+3\Delta)-\frac{\mu}{2}(t_0+3\Delta)^2\\
					&\ge 0,
				\end{align*}
				where we used \eqref{eq:int2a} with $t =t_0+3\Delta$ in the last inequality.
			\end{itemize}
			
			We conclude that for any $t \ge t_0$ we have $h(t)-\frac{\mu}{2}t^2  \ge 0$ which proves that $f_{i}$ is $\mu$-QG.
		\end{itemize}
	\end{proof}
	We conclude using the lemma above that each \(f_i\) lies in \(\sF_{d_0}\). Consequently, by Lemma~\ref{lem:0}, the functions \(F_i: \mathbb{R}^d \to \mathbb{R}\) defined by \(F_i(\bx) = f_i(\bx_A)\) for \(i \in [m]\) belong to \(\sF_d\). Observe that each \(F_i\) achieves its minimum value of $0$, and its set of global minimizers is \(\{(\bz_i, \by): \by \in \R^{d - d_0}\}\), which is convex.
	
	\subsection{Oracle Construction and Information Theoretic tools.}
	
	\textbf{Oracle Construction.} We now specify the stochastic oracle \(\phi: \R^d \times \sF_d \to \R^d\) used in the proof. At each oracle call, let \(\bxi\) be a fresh independent sample from a \(d_0\)-dimensional normal distribution with zero mean and covariance matrix \(\frac{\sigma^2}{d_0}\bI_{d_0}\), i.e., \(\bxi \sim \sN_{d_0}\bigl(\bzero_A, \frac{\sigma^2}{d_0}\bI_{d_0}\bigr)\). Given an input \(\bx \in \R^d\) and a function \(H \in \sF_d\), we define \(\phi(\bx, H) = \nabla H(\bx) + (\bxi, \mathbf{0}_{\bar{A}})\). Thus, \(\phi\) is a stochastic oracle belonging to $\mathbb{O}_{\sigma}$ as specified in Definition~\ref{def:oracle}. Moreover, for each \(i \in [m]\), we have \(\phi(\bx, F_i) = \bigl(\nabla f_i(\bx_A) + \bxi, \bzero_{\bar{A}}\bigr)\).

	\noindent\textbf{Information theoretic tools.} We reduce the optimization problem to one of function identification. To that end, consider a reference field \(V: \R^d \to \R^d\) of the form \(V(\bx)=(v(\bx_A),\bzero_{\bar A})\), where \(v:\R^{d_0}\to\R^{d_0}\) is a measurable field. These reference fields need not be continuous or gradients of functions. We define, for any \(\bx \in \R^{d_0}\),
	\[
	v(\bx) =
	\begin{cases}
		\displaystyle \bzero, & \text{if }\|\bx\| < 8\Delta, \\[1em]
		\displaystyle 169\mu \left( \bx-8\Delta \frac{\bx}{\norm{\bx}}\right), & \text{if }\|\bx \| \ge 8\Delta~.
	\end{cases}
	\]
	We introduce the following technical lemma, which will be instrumental in computing the relative entropy between feedback distributions.
	\begin{lemma}\label{lem:gd_bounds}
		Let $i \in [m]$. For any $\bx\in \R^d$, we have
		\[
		\norm{\nabla F_{i}(\bx) -V(\bx)} \le
		\begin{cases}
			\displaystyle 169\mu \Delta, & \text{if }\bx_A \in B_{d_0}(\bz_i, 2\Delta), \\[1em]
			\displaystyle 169 \mu \tau \Delta, & \text{if }\bx_A \notin B_{d_0}(\bz_i, 2\Delta).
		\end{cases}
		\]
	\end{lemma}
	\begin{proof}
		Let $i \in [m]$. This result is a consequence of the expressions of $\nabla f_{i}$ and $v$. We have
		\begin{itemize}
			\item If $\bx_A \in B_{d_0}(\bz_i, (1+a)\Delta)$, given that $a\le 1$ we have $B_{d_0}(\bz_i, (1+a)\Delta) \subseteq B_{d_0}(\bz_i, 2\Delta)$. Therefore, 
			\begin{align*}
				\norm{\nabla F_i(\bx) - V(\bx)} &= \norm{\nabla f_{i}(\bx_A) -v(\bx_A)}\\
				&= \norm{\nabla f_{i}(\bx_A)}\\
				&\le 169\mu \Delta~. 
			\end{align*}
			\item If $\bx_A\in B_{d_0}(\boldsymbol{0}, 8\Delta)\setminus B_{d_0}(\bz_i, (1+a)\Delta)$, then
			\begin{align*}
				\norm{\nabla F_i(\bx) - V(\bx)} &= \norm{\nabla f_{i}(\bx_A)-v(\bx_A)}\\
				&= \norm{\nabla f_{i}(\bx_A)}\\
				&= 169\mu (1-a)\Delta \le 169 \mu \tau \Delta~.
			\end{align*} 
			\item If $\bx_A \in \R^{d_0} \setminus B_{d_0}(\boldsymbol{0}, 8\Delta)$, then
			\begin{align*}
				&\norm{\nabla F_{i}(\bx) - V(\bx)}\\
				&= \norm{\nabla f_{i}(\bx_A) -v(\bx_A)}\\ 
				&= \norm{ 169\mu \left(\bx_A-8\Delta \frac{\bx_A}{\norm{\bx_A}}\right)+169\mu (1-a)\Delta \frac{\bx_A-\bz_i}{\norm{\bx_A-\bz_i}}- 169\mu \left(\bx_A-8\Delta \frac{\bx_A}{\norm{\bx_A}}\right)}\\
				&=169\mu (1-a)\Delta 
				\le 169 \mu \tau \Delta~.
			\end{align*}
		\end{itemize}
	\end{proof}
	Let \(\hat{\bx} \in \R^d\) denote the output of the optimization algorithm. For each \(i \in [m]\), let \(\fP_i\) and \(\Q\) denote the probability distributions of the complete transcript (the random seed, the \(T\) query--feedback pairs, and the reported output)\footnote{For a rigorous definition of these quantities, see Section~\ref{sec:can_model} where we define the canonical model.} under objective \(F_i\) and under the auxiliary Gaussian feedback \(V(\bx)+(\bxi,\bzero_{\bar A})\), respectively (we omit the dependence on \(T\) in our notation). Also, let \(\E_i[\cdot]\) and \(\E[\cdot]\) denote the expectations with respect to \(\fP_i\) and \(\Q\), respectively. For each \(i \in [m]\), we define the “good identification event” as \(\sE_i := \{\hat{\bx}_A \in B_{d_0}(\bz_i, 2\Delta)\}\). In the remainder of the proof, we derive a lower bound on the misidentification event \(\neg \sE_i\) and then use this bound to relate misidentification to the optimization error using the bound below. For any $i \in [m]$, we have
	\begin{align}
		\E_i[F_i(\hat{\bx})] &= \E_i[f_i(\hat{\bx}_A)]\nonumber\\
		&\ge \fP_i(\sE_i) \inf_{\bx \in B_{d_0}(\bz_i, 2\Delta)} \ \{f_{i}(\bx) \}+(1-\fP_i(\sE_i))\inf_{\bx \notin B_{d_0}(\bz_i, 2\Delta)} \ \{ f_{i}(\bx)\}\nonumber\\ 
		&\ge (1-\fP_i(\sE_i))\inf_{\bx \notin B_{d_0} \ (\bz_i, 2\Delta)} \{ f_{i}(\bx)\}\nonumber\\
		&\ge (1-\fP_i(\sE_i))\frac{169}{2} \mu \Delta^2~.\label{eq:opta}
	\end{align}
	Next, we specify the choice of \(m\) and \(\bz_i\). We select \(m\) to ensure that the balls \(B_{d_0}(\bz_i, 2\Delta)\) for \(i \in [m]\) remain disjoint. Recall that \(\bz_i\) is defined as a sequence of elements within \(B_{d_0}(\bzero_A, 5\Delta)\). The largest possible \(m\) ensuring disjointness of these balls corresponds to the packing number of \(B_{d_0}(\bzero_A, 5\Delta)\) with radius \(2\Delta\). A lower bound for this number is provided in Lemma~\ref{lem:cover}, that implies that it suffices to take \(m = \left\lceil \frac{1}{2} \left(\frac{5}{4}\right)^{d_0} \right\rceil\). Also, we select \(\bz_i\) as a sequence of elements such that \(B_{d_0}(\bz_i, 2\Delta) \cap B_{d_0}(\bz_j, 2\Delta) = \emptyset\) for \(i \neq j\).
	
	\noindent To derive an upper bound on \(\fP_i(\sE_i)\) (or equivalently, a lower bound on the probability of \(\neg \sE_i\)), we apply Pinsker’s inequality. This gives  
	\[
	\frac{1}{m} \sum_{i=1}^{m} \fP_i(\sE_i) \le \frac{1}{m} \sum_{i=1}^{m} \Q(\sE_i) + \sqrt{\frac{1}{2m} \sum_{i=1}^{m} \KL\left(\Q, \fP_i\right)}~.
	\]
	Since the events \(\sE_i\) for \(i \in [m]\) are disjoint, it follows that  $\sum_{i=1}^{m} \Q(\sE_i) \le 1$, therefore the bound above gives
	\begin{equation}\label{eq:pinsa}
		\frac{1}{m} \sum_{i=1}^{m} \bigl(1 - \fP_i(\sE_i)\bigr) \ge 1 - \frac{1}{m} - \sqrt{\frac{1}{2m} \sum_{i=1}^{m} \KL\left(\Q, \fP_i\right)}~.
	\end{equation}
	
	\noindent At this step, we develop an upper bound on the average \(\frac{1}{m} \sum_{i=1}^{m} \KL\left(\Q, \fP_i\right)\). Let \(N_i\) denote the number of times the algorithm queries a point \(\bx\) such that \(\bx_A \in B_{d_0}(\bz_i, 2\Delta)\). Using Lemma~\ref{lem:chain_rule} in the Appendix, we obtain  
	\begin{align*}
		\text{KL}\left(\Q, \fP_i\right)  &\le \frac{d_0}{2\sigma^2}\mathbb{E}[N_{i}] \sup_{\bx_A \in B_{d_0}(\bz_i, 2\Delta)} \|V(\bx)- \nabla F_i(\bx)\|^2\\
		&\qquad + \frac{d_0}{2\sigma^2} \mathbb{E}[T-N_{i}] \sup_{\bx_A \notin B_{d_0}(\bz_i, 2\Delta)} \|V(\bx)- \nabla F_i(\bx)\|^2~.	
	\end{align*}
	Using the bounds in Lemma~\ref{lem:gd_bounds} , we obtain  
	\[
	\KL\left(\Q, \fP_i\right) \le \frac{169^2 d_0 \mu^2 \Delta^2}{2\sigma^2} \E[N_{i}] + \frac{169^2 d_0\tau^2 \mu^2 \Delta^2}{2\sigma^2}T.
	\]  
	Averaging over \(m\) and using the fact that \(\sum_{i=1}^{m} N_i \le T\), which holds due to our choice of the sequence \((\bz_i)\), we derive the bound  
	\begin{equation}\label{eq:kl_avra}
		\frac{1}{m} \sum_{i=1}^{m} \KL\left(\Q, \fP_i\right) \le \frac{169^2 d_0 \mu^2 \Delta^2}{2\sigma^2} \left(\tau^2+\frac{1}{m}\right)T.
	\end{equation}    
	
	Combining \eqref{eq:pinsa} with \eqref{eq:kl_avra} yields a lower bound on the average misidentification error \(\frac{1}{m} \sum_{i=1}^{m} (1-\fP_i(\sE_i))\). This, in turn, leads to the following lower bound when applying \eqref{eq:opta}:  
	\[
	\frac{1}{m} \sum_{i=1}^{m} \E_i[F_i(\hat{\bx})] \ge \frac{169}{2}\mu \Delta^2 \left(1-\frac{1}{m} - \sqrt{\frac{169^2 d_0 \mu^2\Delta^2}{4\sigma^2}\left(\tau^2+\frac{1}{m}\right)T } \right).
	\]  
	To conclude, we use the fact that $m = \ceil{\frac{1}{2}(5/4)^{d_0}}$, and $d_0 = \ceil{\log_{5/4}(4/\tau^2)}$, which gives $m \ge \frac{2}{\tau^2}$.  
	\begin{align*}
		\frac{1}{m}\sum_{i=1}^{m}\mathbb{E}_i[F_{i}(\hat{\bx})] 
		&\ge \frac{169}{2}\mu \Delta^2 \left(1- \frac{\tau^2}{2}-\sqrt{10711 \frac{\log_{5/4}(5/\tau^2)}{\sigma^2} \mu^2\tau^2\Delta^2T }\right)\\
		&\ge \frac{169}{2}\mu \Delta^2 \left(\frac{1}{2}-\sqrt{10711 \frac{\log_{5/4}(5/\tau^2)}{\sigma^2} \mu^2\tau^2\Delta^2T }\right)~. 
	\end{align*}
	We choose
	\[
	\Delta^*
	= 
	\frac{1}{4\sqrt{10711}\, \mu \tau} 
	\cdot 
	\frac{1}{\sqrt{
			\frac{\log_{5/4}\left(\frac{5}{\tau^2}\right)}{\sigma^2} \cdot T
	}}~.
	\] 
	This leads to
	\[
	\frac{1}{m}\sum_{i=1}^{m}\mathbb{E}_i[F_{i}(\hat{\bx})] \ge c\cdot \frac{\sigma^2}{\mu \tau^2\log(2/\tau) T}~,
	\] 
	where $c$ is a numerical constant.

	\section{Complete proof of Theorem~\ref{thm:2}}\label{sec:proof2}
	Let \(L, \mu > 0\) and \(\sigma > 0\). Denote \(\kappa := \tfrac{L}{\mu}\), $d_0 := \lceil 2 \log_{5/4}(2\kappa)\rceil$, and suppose that $\kappa\ge1$ and \(d \ge d_0\).
	
	\paragraph{Additional notation.} For any integer \(n \ge 1\), let \(\mathcal{G}_n\) denote the class of real-valued functions on \(\mathbb{R}^n\) that are \(L\)-smooth and $\mu-\RSI$. Let \(A := \{1, \dots, d_0\}\) and \(\bar{A} := [d] \setminus A\). For any \(\bx \in \mathbb{R}^d\), let \(\bx_A \in \mathbb{R}^{d_0}\) be the vector of its first \(d_0\) coordinates, and \(\bx_{\bar{A}} \in \mathbb{R}^{d - d_0}\) be the vector of its remaining coordinates. For two 
	vectors \(\mathbf{a} \in \mathbb{R}^{d_0}\) and \(\mathbf{b} \in \mathbb{R}^{d - d_0}\), we write \((\mathbf{a}, \mathbf{b})\) to denote the vector in \(\mathbb{R}^d\) whose first \(d_0\) components are \(\mathbf{a}\) and last \(d - d_0\) components are \(\mathbf{b}\).
	
	\subsection{Functions construction.} For \(d \ge d_0\), we construct a function \(F: \R^d \to \R\) that depends only on the first \(d_0\) coordinates. Concretely, there is a function \(f: \R^{d_0} \to \R\) such that \(F(\bx) = f(\bx_A)\). Lemma~\ref{lem:dim} shows that if $f$ is in $\sG_{d_0}$, then $F$ is in $\sG_d$.
	
	\noindent Let \(\Delta > 0\) be a parameter to be specified later. Let $a$ be defined by
	\[
	a := \frac{\kappa-1}{\kappa+1}~.
	\]
	Since $\kappa \ge 1$, we have $a \in [0,1)$ and $1-a=2/(\kappa+1)$.
	
	Let \(m \ge 2\) denote the size of the function subclass we are constructing. We pick \(m\) elements \(\bz_1, \dots, \bz_m\) in \(B_{d_0}(\bzero_A, 5\Delta)\), and define a set of functions \(f_{1}, \dots, f_{m}\) so that each \(f_{i}\) belongs to \(\sG_{d_0}\). For each \(i \in [m]\), let \(f_{i}\) be a function whose value at \(\bz_i\) is zero and whose gradient for all \(\bx \in \R^{d_0}\) is given by
	\[
	\nabla f_{i}(\bx) =
	\begin{cases}
		\displaystyle L(\bx - \bz_i), & \text{if }\|\bx - \bz_i\| < \Delta, \\[1em]
		\displaystyle -L\Bigl(\bx - \bz_i - \Delta \frac{\bx - \bz_i}{\|\bx - \bz_i\|}\Bigr) + L\Delta \frac{\bx - \bz_i}{\|\bx - \bz_i\|}, & \text{if }\Delta \leq \|\bx - \bz_i\| \le (1+a)\Delta, \\[1em]
		\displaystyle \mu\Bigl(\bx - \bz_i\Bigr) , & \text{if }\bx \notin  B_{d_0}(\bz_i, (1+a)\Delta)~.
	\end{cases}
	\]
	The expression of $f_{i}$ is given by
	\[
	f_{i}(\bx) =
	\begin{cases}
		\displaystyle
		\frac{L}{2}\,\|\bx - \bz_i\|^2, 
		& \text{if }\|\bx - \bz_i\| < \Delta, \\[1.5em]
		\displaystyle
		-\frac{L}{2}\,\|\bx - \bz_i\|^2 
		+ 2\,L\,\Delta\,\|\bx - \bz_i\|
		- L\,\Delta^2,
		& \text{if }\Delta \leq \|\bx - \bz_i\| < (1+a)\,\Delta, \\[1.5em]
		\displaystyle
		L\,\Delta^2\,\frac{1 + 2a - a^2}{2}
		+ \frac{\mu}{2}\,\bigl(\|\bx - \bz_i\|^2 - (1+a)^2\,\Delta^2\bigr),
		& \text{if }\|\bx - \bz_i\| \geq (1+a)\,\Delta~.
	\end{cases}
	\]
	Lemma below shows that the constructed functions $f_{i}$ are in $\sG_{d_0}$.
	\begin{lemma}\label{lem:g}
		The functions $(f_{i})_{i\in [m]}$ are $\mu-\RSI$ and $L$-smooth.
	\end{lemma}
	\begin{proof}
		Let $i \in [m]$, we will first verify that $f_i$ is $\mu-\RSI$ then we will check for smoothness.
		
		\noindent \textbf{Verifying $\mu$-RSI.}
		Recall that the minimizer of $f_{i}$ is $\bz_i$, and the minimum is $0$.  We have for any $\bx \in \R^{d_0}$
		\[
		\bigl\langle \nabla f_{i}(\bx),\,\bx-\bz_i \bigr\rangle - \mu\,\|\bx-\bz_i\|^2
		\;=\;
		\begin{cases}
			\displaystyle
			(L - \mu)\,\|\bx-\bz_i\|^2,
			& \text{if }0 \le \|\bx-\bz_i\| < \Delta,\\[1em]
			\displaystyle
			-\,(L + \mu)\,\|\bx - \bz_i\|^2 
			+ 2\,L\,\Delta\,\|\bx - \bz_i\|,
			& \text{if }\Delta \le \|\bx-\bz_i\| < (1+a)\,\Delta,\\[1em]
			0,
			& \text{if }\|\bx-\bz_i\| \ge (1+a)\,\Delta~.
		\end{cases}
		\]
		This expression is clearly non-negative if $\bx \in B_{d_0}(\bz_i, \Delta)$ or $\bx \notin B_{d_0}(\bz_i, (1+a)\Delta)$. If $\bx \in B_{d_0}(\bz_i, (1+a)\Delta)\setminus B_{d_0}(\bz_i, \Delta)$, we have that $r \to -(L+\mu)r^2+2L\Delta r$ is decreasing on $[ \frac{L}{L+\mu}\Delta, +\infty)$. The last interval contains $[\Delta, (1+a)\Delta]$, therefore for each $\bx \in  B_{d_0}(\bz_i, (1+a)\Delta)\setminus B_{d_0}(\bz_i, \Delta)$ we have
		\begin{align*}
			\bigl\langle \nabla f_{i}(\bx),\,\bx-\bz_i \bigr\rangle - \mu\,\|\bx-\bz_i\|^2 
			\ge -(L+\mu)(1+a)^2\Delta^2+2L(1+a)\Delta^2
			= 0~.
		\end{align*} 
		
		\noindent \textbf{Verifying $L$-smooth:}
		It is straightforward that $\nabla f_{i}$ is $L$-Lipschitz on the regions $B_{d_0}(\bz_i, \Delta)$ and $\mathbb{R}^{d_0} \setminus B_{d_0}(\bz_i, (1+a)\Delta)$. Using Lemma~\ref{lem:lip2}, $\nabla f_{i}$ is $L$-Lipschitz on $B_{d_0}(\bz_i, (1+a)\Delta)\setminus B_{d_0}(\bz_i, \Delta)$. The conclusion follows using Lemma~\ref{lem:lip}.
	\end{proof}
	We conclude using the lemma above that each \(f_i\) lies in \(\sG_{d_0}\). Consequently, by Lemma~\ref{lem:dim}, the functions \(F_i: \mathbb{R}^d \to \mathbb{R}\) defined by \(F_i(\bx) = f_i(\bx_A)\) for \(i \in [m]\) belong to \(\sG_d\). Observe that each \(F_i\) achieves its minimum value of $0$, and its set of global minimizers is \(\{(\bz_i, \by): \by \in \R^{d - d_0}\}\), which is convex.
	\subsection{Oracle construction and information-theoretic tools.}
	
	The oracle we use is the same one introduced in the proof of Theorem~\ref{thm:1}.
	
	\noindent \textbf{Information theoretic tools.} We reduce the optimization problem to one of function identification. To that end, consider a reference field \(V: \R^d \to \R^d\) of the form \(V(\bx)=(v(\bx_A),\bzero_{\bar A})\), where \(v:\R^{d_0}\to\R^{d_0}\) is a measurable field. These reference fields need not be continuous or gradients of functions. We define, for any \(\bx \in \R^{d_0}\),
	\[
	v(\bx) = \mu\bx~.
	\]
	We introduce the following technical lemma, which will be instrumental in computing the relative entropy between feedback distributions.
	\begin{lemma}\label{lem:gd_bounds2}
		Let $i \in [m]$. For any $\bx\in \R^d$, we have
		\[
		\norm{\nabla F_{i}(\bx) -V(\bx)} \le
		\begin{cases}
			\displaystyle 5L \Delta, & \text{if }\bx_A \in B_{d_0}(\bz_i, 2\Delta), \\[1em]
			\displaystyle 5 \mu \Delta, & \text{if }\bx_A \notin B_{d_0}(\bz_i, 2\Delta).
		\end{cases}
		\]
	\end{lemma}
	\begin{proof}
		Let $i \in [m]$. This result is a consequence of the expressions of $\nabla f_{i}$ and $v$. We have 
		\begin{itemize}
			\item If $\bx_A \in B_{d_0}(\bz_i, 2\Delta)$, the expression of $\nabla f_i$ and $1+a=2L/(L+\mu)$ give $\norm{\nabla f_i(\bx_A)-\mu(\bx_A-\bz_i)}\le (L-\mu)\Delta$. Therefore,
			\begin{align*}
				\norm{\nabla F_i(\bx) - V(\bx)} &= \norm{\nabla f_i(\bx_A)-v(\bx_A)}\\
				&\le \norm{\nabla f_i(\bx_A)-\mu(\bx_A-\bz_i)}+\mu\norm{\bz_i}\\
				&\le (L-\mu)\Delta+5\mu\Delta \le 5L\Delta~. 
			\end{align*}
			\item If $\bx_A \notin B_{d_0}(\bz_i, 2\Delta)$, then
			\begin{align*}
				\norm{\nabla F_i(\bx) - V(\bx)} &= \norm{\nabla f_{i}(\bx_A)-v(\bx_A)}\\
				&= \norm{\nabla f_{i}(\bx_A)-\mu \bx_A}\\
				&= \mu \norm{\bz_i}
				\le 5\mu\Delta~. 
			\end{align*} 
		\end{itemize}
	\end{proof}
	We follow similar steps as in the proof of Theorem~\ref{thm:2}. Let \(\hat{\bx} \in \R^d\) denote the output of the optimization algorithm. For each \(i \in [m]\), let \(\fP_i\) and \(\Q\) denote the probability distributions of the complete transcript (the random seed, the \(T\) query--feedback pairs, and the reported output)\footnote{For a rigorous definition of these quantities, see Section~\ref{sec:can_model} where we define the canonical model.} under objective \(F_i\) and under the auxiliary Gaussian feedback \(V(\bx)+(\bxi,\bzero_{\bar A})\), respectively (we omit the dependence on \(T\) in our notation). Also, let \(\E_i[\cdot]\) and \(\E[\cdot]\) denote the expectations with respect to \(\fP_i\) and \(\Q\), respectively. For each \(i \in [m]\), we define the “good identification event” as \(\sE_i := \{\hat{\bx}_A \in B_{d_0}(\bz_i, 2\Delta)\}\). For any $i \in [m]$, we have
	\begin{align}
		\E_i[F_i(\hat{\bx})] &= \E_i[f_i(\hat{\bx}_A)]\nonumber\\
		&\ge \fP_i(\sE_i) \inf_{\bx \in B_{d_0}(\bz_i, 2\Delta)} \ \{f_{i}(\bx) \}+(1-\fP_i(\sE_i))\inf_{\bx \notin B_{d_0}(\bz_i, 2\Delta)} \ \{ f_{i}(\bx)\}\nonumber\\ 
		&\ge (1-\fP_i(\sE_i))\inf_{\bx \notin B_{d_0} \ (\bz_i, 2\Delta)} \{ f_{i}(\bx)\}\nonumber\\
		&\ge (1-\fP_i(\sE_i)) L\Delta^2 \frac{1+2a-a^2}{2} \nonumber\\
		&\ge (1-\fP_i(\sE_i)) \frac{L}{2} \Delta^2~.\label{eq:opt_rsi}
	\end{align}
	Next, we specify the choice of \(m\) and \(\bz_i\). We select \(m\) to ensure that the balls \(B_{d_0}(\bz_i, 2\Delta)\) for \(i \in [m]\) remain disjoint. Using Lemma~\ref{lem:cover}, we choose $m = \ceil{\frac{1}{2} \left( \frac{5}{4} \right)^{d_0}}$, and \(\bz_i\) as a sequence of elements such that \(B_{d_0}(\bz_i, 2\Delta) \cap B_{d_0}(\bz_j, 2\Delta) = \emptyset\) for \(i \neq j\).
	
	\noindent We use Pinsker's inequality to derive an upper bound on \(\fP_i(\sE_i)\). This gives  
	\[
	\frac{1}{m} \sum_{i=1}^{m} \fP_i(\sE_i) \le \frac{1}{m} \sum_{i=1}^{m} \Q(\sE_i) + \sqrt{\frac{1}{2m} \sum_{i=1}^{m} \KL\left(\Q, \fP_i\right)}~.
	\]
	Since the events \(\sE_i\) for \(i \in [m]\) are disjoint, it follows that  $\sum_{i=1}^{m} \Q(\sE_i) \le 1$, therefore the bound above gives
	\begin{equation}\label{eq:pins_rsi}
		\frac{1}{m} \sum_{i=1}^{m} \bigl(1 - \fP_i(\sE_i)\bigr) \ge 1 - \frac{1}{m} - \sqrt{\frac{1}{m} \sum_{i=1}^{m} \KL\left(\Q, \fP_i\right)}~.
	\end{equation}
	
	\noindent Let us develop an upper bound on the average \(\frac{1}{m} \sum_{i=1}^{m} \KL\left(\Q, \fP_i\right)\). Let \(N_i\) denote the number of times the algorithm queries a point \(\bx\) such that \(\bx_A \in B_{d_0}(\bz_i, 2\Delta)\). Using Lemma~\ref{lem:chain_rule} in the Appendix, we obtain  
	\begin{align*}
		\text{KL}\left(\Q, \fP_i\right)  &\le \frac{d_0}{2\sigma^2}\mathbb{E}[N_{i}] \sup_{\bx_A \in B_{d_0}(\bz_i, 2\Delta)} \|V(\bx)- \nabla F_i(\bx)\|^2\\
		&\qquad + \frac{d_0}{2\sigma^2} \mathbb{E}[T-N_{i}] \sup_{\bx_A \notin B_{d_0}(\bz_i, 2\Delta)} \|V(\bx)- \nabla F_i(\bx)\|^2~.	
	\end{align*}
	Using the bounds in Lemma~\ref{lem:gd_bounds2} , we obtain  
	\[
	\KL\left(\Q, \fP_i\right) \le \frac{25d_0 L^2 \Delta^2}{2\sigma^2} \E[N_{i}] + \frac{25d_0 \mu^2 \Delta^2}{2\sigma^2}T~.
	\]  
	Averaging over \(m\) and using the fact that \(\sum_{i=1}^{m} N_i \le T\), which holds due to our choice of the sequence \((\bz_i)\), we derive the bound  
	\begin{equation}\label{eq:kl_avr_rsi}
		\frac{1}{m} \sum_{i=1}^{m} \KL\left(\Q, \fP_i\right) \le \frac{ d_0 \mu^2 \Delta^2}{2\sigma^2} \left(\frac{25\kappa^2}{m}+25\right)T~.
	\end{equation}    
	
	\noindent Combining \eqref{eq:pins_rsi} with \eqref{eq:kl_avr_rsi} yields a lower bound on the average misidentification error \(\frac{1}{m} \sum_{i=1}^{m} (1-\fP_i(\sE_i))\). This, in turn, leads to the following lower bound when applying \eqref{eq:opt_rsi}  
	\[
	\frac{1}{m} \sum_{i=1}^{m} \E_i[F_i(\hat{\bx})] \ge \frac{L}{2}\Delta^2 \left(1-\frac{1}{m} - \sqrt{\frac{d_0 \mu^2 \Delta^2}{2\sigma^2} \left(\frac{25\kappa^2}{m}+25\right)T } \right).
	\]  
	To conclude, we use the fact that $m = \ceil{\frac{1}{2}(5/4)^{d_0}}$, and $d_0 = \ceil{\log_{5/4}(4\kappa^2)}$, which gives $m \ge 2\kappa^2$.  
	\begin{align*}
		\frac{1}{m}\sum_{i=1}^{m}\mathbb{E}_i[F_{i}(\hat{\bx})] 
		&\ge \frac{L}{2} \Delta^2 \left(1- \frac{1}{2\kappa^2}-\sqrt{38 \frac{\log_{5/4}(5\kappa^2)}{\sigma^2} \mu^2\Delta^2T }\right)\\
		&\ge \frac{L}{2} \Delta^2 \left(\frac{1}{2}-\sqrt{38\frac{\log_{5/4}(5\kappa^2)}{\sigma^2} \mu^2\Delta^2T }\right)~. 
	\end{align*}
	We choose
	\[
	\Delta^*= \frac{1}{4\mu} \sqrt{\frac{\sigma^2}{38 \log_{5/4}(5\kappa^2)T}}~.
	\]
	This leads to
	\[
	\frac{1}{m}\sum_{i=1}^{m}\mathbb{E}_i[F_{i}(\hat{\bx})] \ge c\cdot \frac{L\sigma^2}{\log(2\kappa)\mu^2 T}~,
	\] 
	where $c$ is a numerical constant.

	\section{Complete Proof of Theorem~\ref{thm:3}}\label{sec:proof3}
	
	Let $D,L >0$ and $\tau \in (0,1]$. Denote $d_0 := \ceil{2\log_{16/15}(2/\tau)}$ and suppose that $d \ge d_0$ and $T \ge \frac{35\sigma^2}{L^2D^2\tau^2\log_{16/15}(2/\tau)}$.
	
	\paragraph{Additional Notation.} For any integer \(n \ge 1\), let \(\mathcal{H}_n\) denote the class of real-valued functions on \(\mathbb{R}^n\) that are \(L\)-smooth and $\tau-\QC$. Let \(A := \{1, \dots, d_0\}\) and \(\bar{A} := [d] \setminus A\). For any \(\bx \in \mathbb{R}^d\), let \(\bx_A \in \mathbb{R}^{d_0}\) be the vector of its first \(d_0\) coordinates, and \(\bx_{\bar{A}} \in \mathbb{R}^{d - d_0}\) be the vector of its remaining coordinates. For two 
	vectors \(\mathbf{a} \in \mathbb{R}^{d_0}\) and \(\mathbf{b} \in \mathbb{R}^{d - d_0}\), we write \((\mathbf{a}, \mathbf{b})\) to denote the vector in \(\mathbb{R}^d\) whose first \(d_0\) components are \(\mathbf{a}\) and last \(d - d_0\) components are \(\mathbf{b}\).
	
	\subsection{Functions Construction.} For \(d \ge d_0\), we construct a function \(F: \R^d \to \R\) that depends only on the first \(d_0\) coordinates. Lemma~\ref{lem:dim} shows that if $f$ is in $\sH_{d_0}$, then $F$ is in $\sH_d$.
	
	\noindent Let \(\Delta \in (0, D/16)\) be a parameter to be specified later. Let $Q:\R \to \R$ be defined by $Q(r) = \left(\frac{\tau}{2}-1\right)r^2 + \left((1-\tau) - \frac{D}{4\Delta}\right)r+(1-\tau)\frac{D}{4\Delta}+ \frac{\tau}{2}$. Observe that $Q$ is a concave parabola and $Q(0)>0$. Let $a$ be the positive root for $Q$. We prove in Lemma~\ref{lem:tech2} that
	\[
	a \in [1-\tau,1]~.
	\]
	
	Let \(m \ge 2\) denote the size of the function subclass we are constructing. We pick \(m\) elements \(\bz_1, \dots, \bz_m\) in \(B_{d_0}(\bzero_A, 2D/3)\), and define a set of functions \(f_{1}, \dots, f_{m}\) so that each \(f_{i}\) belongs to \(\sH_{d_0}\). For each \(i \in [m]\), let \(f_{i}\) be a function whose value at \(\bz_i\) is zero and whose gradient for all \(\bx \in \R^{d_0}\) is given by
	\[
	\nabla f_{i}(\bx) =
	\begin{cases}
		\displaystyle L(\bx - \bz_i), & \text{if } \|\bx - \bz_i\| < \Delta, \\[1em]
		\displaystyle L\Delta \frac{\bx - \bz_i}{\|\bx - \bz_i\|}, & \text{if } \Delta \leq \|\bx - \bz_i\| < \frac{D}{4}, \\[1em]
		\displaystyle -L\left(\bx-\bz_i - \frac{D}{4} \frac{\bx-\bz_i}{\norm{\bx-\bz_i}}\right)+L\Delta  \frac{\bx - \bz_i}{\|\bx - \bz_i\|} , & \text{if } \frac{D}{4} \leq \|\bx - \bz_i\| < \frac{D}{4}+a\Delta, \\[1em]
		\displaystyle  L(1-a)\Delta  \frac{\bx - \bz_i}{\|\bx - \bz_i\|}, & \text{if } \bx \notin B_{d_0}(\bz_i, D/4+a\Delta)~.
	\end{cases}
	\]
	
	\noindent Let $r := \norm{\bx-\bz_i}$ and define
	\[
	f_{i}(\bx)
	\;=\;
	\begin{cases}
		\dfrac{L}{2}\,r^2,
		& 
		\text{ if }0 \,\le\, r < \Delta,
		\\[1.1em]
		L\,\Delta\,r \;-\;\dfrac{L}{2}\,\Delta^2,
		& 
		\text{ if }\Delta \,\le\, r < \dfrac{D}{4},
		\\[1.1em]
		L \Bigl(\dfrac{D}{4}\,r \;+\;\Delta\,r \;-\;\tfrac{r^2}{2}\Bigr)
		\;-\;
		L\Bigl(\tfrac{\Delta^2}{2} +\tfrac{D^2}{32}\Bigr),
		& 
		\text{ if }\dfrac{D}{4} \,\le\, r < \dfrac{D}{4} + a\,\Delta,
		\\[1.1em]
		L\,(1-a)\,\Delta\,r \;+\; \frac{a}{4} LD\Delta - \frac{1-a^2}{2} L\Delta^2 ,
		& 
		\text{ if }r \,\ge\, \dfrac{D}{4} + a\,\Delta,
	\end{cases}
	\]
	\begin{lemma}\label{lem:tech2}
		Consider the parabola $Q(r) = \left(\frac{\tau}{2}-1\right)r^2+ \left((1-\tau)-\frac{D}{4\Delta}\right)r+(1-\tau)\frac{D}{4\Delta}+\frac{\tau}{2}$. Let $a$ be the positive root of $Q$. Then, we have
		\[
		1-\tau \le a \le 1~.
		\]
	\end{lemma}
	\begin{proof}
		We have:
		\begin{itemize}
			\item Proof of $1-\tau \le a$: 
			\begin{align*}
				Q(1-\tau) &= \left(\frac{\tau}{2}-1\right)\left(1-\tau\right)^2+ \left((1-\tau)-\frac{D}{4\Delta}\right)\left(1-\tau\right)+(1-\tau)\frac{D}{4\Delta}+\frac{\tau}{2}\\
				&= \frac{\tau}{2} (1-\tau)^2+ \frac{\tau}{2} \ge 0~.
			\end{align*}
			Therefore $a \ge 1-\tau$.	
			\item Proof of $a \le 1$:
			\begin{align*}
				Q(1) = \frac{\tau}{2}-1+1-\tau - \frac{D}{4\Delta}+(1-\tau)\frac{D}{4\Delta}+\frac{\tau}{2}
				= -\tau \frac{D}{4\Delta} \le 0~.
			\end{align*}
			Therefore $a \le 1$.
		\end{itemize}
	\end{proof}
	Lemma below shows that the constructed functions $f_{i}$ are in $\sH_{d_0}$.
	\begin{lemma}\label{lem:f3}
		For each $i \in [m]$, the function $f_{i}$ is $\tau$-WQC and $L$-smooth.
	\end{lemma}
	\begin{proof}
		Let $i\in [m]$, in the following, we will use that $\tau \in (0,1]$ and $\Delta \in (0, D/16)$, and from the definition of $a$ we have
		\begin{equation}\label{eq:condc}
			\left(\frac{\tau}{2}-1\right)a^2 -\frac{D}{4\Delta}a+(1-\tau)\frac{D}{4\Delta}+(1-\tau)a+\frac{\tau}{2}=0 \qquad \text{ and } \quad a \in [1-\tau,1]~.
		\end{equation}
		
		\noindent \textbf{Verifying $L$-smoothness:}
		First, we will show that $\nabla f_{i}$ is $L$-Lipschitz on each region in its expression, then we will conclude using Lemma~\ref{lem:lip}. We have
		\begin{itemize}
			\item Case $\bx \in B_{d_0}(\bz_i, \Delta)$: The expression of $\nabla f_{i}$ shows that it is $L$-Lipschitz
			\[
			\norm{L(\bx-\bz_i)-L(\by-\bz_i)}
			= L\norm{\bx-\by} ~.
			\]
			\item Case $\bx \in B_{d_0}(\bz_i, D/4)\setminus B_{d_0}(\bz_i, \Delta)$: recall that for $\bx \notin B_{d_0}(\bz_i, \Delta)$ we have that $\bx \to \Delta \frac{\bx-\bz_i}{\norm{\bx-\bz_i}}$ is the projection of $\bx$ onto $B_{d_0}(\bz_i, \Delta)$ minus its center $\bz_i$. Therefore, $\nabla f_{i}$ is $L$-Lipschitz.
			\item Case $\bx \in B_{d_0}(\bz_i, D/4+a\Delta)\setminus B_{d_0}(\bz_i, D/4)$: writing $r=\norm{\bx-\bz_i}$, the radial and tangential Hessian eigenvalues are $-L$ and $L(D/4+\Delta)/r-L$, respectively. Their absolute values are at most $L$, since $r\ge D/4$ and $\Delta\le D/16$, so $\nabla f_i$ is $L$-Lipschitz. 
			\item Case $\bx \notin B_{d_0}(\bz_i, D/4+a\Delta)$:  since for $\bx \notin B_{d_0}(\bz_i, 
			\frac{D}{4}+a\Delta)$ we have that the mapping $\bx \to \Delta(1-a) \frac{\bx-\bz_i}{\norm{\bx-\bz_i}}$ is the projection of $\bx$ onto $B_{d_0}(\bz_i, (1-a)\Delta)$ minus its center $\bz_i$. Therefore, $\nabla f_{i}$ is $L$-Lipschitz.
		\end{itemize}
		The conclusion follows from the fact that $\nabla f_{i}$ is continuous and Lemma~\ref{lem:lip}.
		
		\textbf{Verifying $\tau$-WQC:}
		Observe that the minimizer of $f_{i}$ is $\bz_i$ and the minimum is $0$. We will show that $f_{i}$ satisfies $\tau$-WQC on each of the four regions in the definition of $\nabla f_{i}$. Define $r:= \norm{\bx-\bz_i}$.
		We have
		\begin{align*}
			\langle &\nabla f_{i}(\bx), \bx - \bz_i \rangle - \tau f_{i}(\bx) \\
			&=\begin{cases}
				\displaystyle
				L\,r^2\bigl(1 - \tfrac{\tau}{2}\bigr),
				& 
				\text{ if } 0 \,\le\, r < \Delta,
				\\[1.25em]
				\displaystyle
				L\,\Delta\,r\,(1-\tau)+\tfrac{\tau L}{2}\,\Delta^2,
				&
				\text{ if } \Delta \,\le\, r < \tfrac{D}{4}, 
				\\[1.25em]
				\displaystyle
				L\left(\frac{\tau}{2}-1\right) r^2+ L(1-\tau) \left(\Delta+\frac{D}{4}\right)r + L\tau \left(\frac{D^2}{32}+\frac{\Delta^2}{2}\right),
				&
				\text{ if } \tfrac{D}{4} \,\le\, r < \tfrac{D}{4}+a\,\Delta,
				\\[1.25em]
				\displaystyle
				L\,(1-a)\,(1-\tau)\,\Delta\,r + L\,\tau\,\Bigl[\tfrac{1-a^2}{2}\,\Delta^2 - \tfrac{a}{4}\,D\,\Delta
				\Bigr],
				&
				\text{ if } r \,\ge\, \tfrac{D}{4} + a\,\Delta~.
			\end{cases}
		\end{align*}
		Since $\tau \le 1$, we have that $\langle \nabla f_{i}(\bx), \bx-\bz_i \rangle -\tau f_{i}(\bx) \ge 0$ in the first region. For the second region, we use the fact that $\langle \nabla f_{i}(\bx), \bx - \bz_i \rangle - \tau f_{i}(\bx)$ is non-decreasing with respect to $r$ and that by continuity of the last expression, it is non-negative for $r=\Delta$. For the third region, observe that the function $R:r \to L\left(\frac{\tau}{2}-1\right) r^2+ L(1-\tau) \left(\Delta+\frac{D}{4}\right)r + L\tau \left(\frac{D^2}{32}+\frac{\Delta^2}{2}\right)$ is concave, therefore for $r \in [D/4, D/4+a\Delta]$, we have
		\[
		R(r) \ge \min\{R(D/4), R(D/4+a\Delta)\}~.
		\]
		For the first term, we have
		\begin{align*}
			R(D/4) &= L\left(\frac{\tau}{2}-1\right) \frac{D^2}{16}+ L(1-\tau) \left(\Delta+\frac{D}{4}\right) \frac{D}{4} + L\tau \left(\frac{D^2}{32}+\frac{\Delta^2}{2}\right)\\
			&= L(1-\tau) \frac{D\Delta}{4}+ L\tau \frac{\Delta^2}{2}\\
			&\ge 0~.
		\end{align*}
		For the second one, we have
		\begin{align*}
			R(D/4+a\Delta) &= L\left(\frac{\tau}{2}-1\right) \left(\frac{D}{4}+a\Delta\right)^2+ L(1-\tau) \left(\Delta+\frac{D}{4}\right)\left(\frac{D}{4}+a\Delta\right) + L\tau \left(\frac{D^2}{32}+\frac{\Delta^2}{2}\right)\\
			&= L\left(\frac{\tau}{2}-1\right)\Delta^2 a^2+ L \left((1-\tau) \Delta^2- \frac{1}{4}D\Delta\right)a + L(1-\tau)\frac{D\Delta}{4} + L\tau \frac{\Delta^2}{2}\\
			&= L\Delta^2 \left[\left(\frac{\tau}{2}-1\right)a^2 + \left((1-\tau)-\frac{D}{4\Delta}\right)a+(1-\tau)\frac{D}{4\Delta}+ \frac{\tau}{2} \right]\\
			&=0,
		\end{align*}
		where we used in the last line the fact that $a$ satisfies \eqref{eq:condc}. We therefore conclude that $R(r) \ge 0$ for any $r \in [D/4, D/4+a\Delta]$. For the last region, we use the fact that $\langle\nabla f_{i}(\bx), \bx - \bz_i \rangle - \tau f_{i}(\bx)$ is continuous in $\mathbb{R}^{d_0}$ and increasing in $r$ for $r \ge D/4+a\Delta$  (since we have $a \le 1$).
	\end{proof}
	We conclude using the lemma above that each \(f_i\) lies in \(\sH_{d_0}\). Consequently, by Lemma~\ref{lem:dim}, the functions \(F_i: \mathbb{R}^d \to \mathbb{R}\) defined by \(F_i(\bx) = f_i(\bx_A)\) for \(i \in [m]\) belong to \(\sH_d\). Observe that each \(F_i\) achieves its minimum value of $0$, and its set of global minimizers is \(\{(\bz_i, \by): \by \in \R^{d - d_0}\}\), which is convex.
	\subsection{Oracle Construction and Information-Theoretic Tools.}
	
	The oracle we use is the same introduced in the proof of Theorem~\ref{thm:1}.
	
	\noindent \textbf{Information Theoretic Tools.} As for the previous sections, we reduce the optimization problem to one of function identification. We use the zero reference fields \(v:\R^{d_0}\to\R^{d_0}\) and \(V:\R^d\to\R^d\), with \(V(\bx)=(v(\bx_A),\bzero_{\bar A})=\bzero\). We introduce the following technical lemma, which will be instrumental in computing the relative entropy between feedback distributions.
	\begin{lemma}\label{lem:gd_bounds3}
		Let $i \in [m]$. Then, for any $\bx\in \R^d$, we have
		\[
		\norm{\nabla F_{i}(\bx) -V(\bx)} \le
		\begin{cases}
			\displaystyle L \Delta, & \text{if } \bx_A \in B_{d_0}(\bz_i, 5D/16), \\[1em]
			\displaystyle L \tau \Delta, & \text{if } \bx_A \notin B_{d_0}(\bz_i, 5D/16)~.
		\end{cases}
		\]
	\end{lemma}
	\begin{proof}
		Let $i \in [m]$. This result is a consequence of the expression of $\nabla f_{i}$. We have for any $\bx \in \R^d$
		\begin{itemize}
			\item If $\bx_A \in B_{d_0}(\bz_i, D/4+ a\Delta)$, given that $a\le 1$ and $\Delta \le D/16$, we have $B_{d_0}(\bz_i, D/4+a\Delta) \subset B_{d_0}(\bz_i, 5D/16)$. Therefore, 
			\begin{align*}
				\norm{\nabla F_i(\bx) - V(\bx)} &= \norm{\nabla f_{i}(\bx_A)}\\
				&= \norm{\nabla f_{i}(\bx_A)}\\
				&\le L \Delta~. 
			\end{align*}
			\item If $D/4+a\Delta\le\|\bx_A-\bz_i\|\le5D/16$, then $\|\nabla F_i(\bx)-V(\bx)\|=L(1-a)\Delta\le L\Delta$.
			\item If $\bx_A \in \R^{d_0}\setminus B_{d_0}(\bz_i, 5D/16)$, then
			\begin{align*}
				\norm{\nabla F_{i}(\bx)-V(\bx)} &= \norm{\nabla f_{i}(\bx_A)}
				= L(1-a) \Delta \le  L \tau \Delta~.
			\end{align*} 
		\end{itemize}
	\end{proof}
	We follow similar steps as in the proof of Theorem~\ref{thm:2}. Let \(\hat{\bx} \in \R^d\) denote the output of the optimization algorithm. For each \(i \in [m]\), let \(\fP_i\) and \(\Q\) denote the probability distributions of the complete transcript (the random seed, the \(T\) query--feedback pairs, and the reported output)\footnote{For a rigorous definition of these quantities, see Section~\ref{sec:can_model} where we define the canonical model.} under objective \(F_i\) and under the auxiliary Gaussian feedback \(V(\bx)+(\bxi,\bzero_{\bar A})\), respectively (we omit the dependence on \(T\) in our notation). Also, let \(\E_i[\cdot]\) and \(\E[\cdot]\) denote the expectations with respect to \(\fP_i\) and \(\Q\), respectively. For each \(i \in [m]\), we define the “good identification event” as \(\sE_i := \{\hat{\bx}_A \in B_{d_0}(\bz_i, 5D/16)\}\). For any $i \in [m]$, we have
	\begin{align}
		\E_i[F_i(\hat{\bx})] &= \E_i[f_i(\hat{\bx}_A)]\nonumber\\
		&\ge \fP_i(\sE_i) \inf_{\bx \in B_{d_0}(\bz_i, 5D/16)} \ \{f_{i}(\bx) \}+(1-\fP_i(\sE_i))\inf_{\bx \notin B_{d_0}(\bz_i, 5D/16)} \ \{ f_{i}(\bx)\}\nonumber\\ 
		&\ge (1-\fP_i(\sE_i))\inf_{\bx \notin B_{d_0} \ (\bz_i, 5D/16)} \{ f_{i}(\bx)\}\nonumber\\
		&\ge (1-\fP_i(\sE_i)) \left(\frac{5-a}{16}LD\Delta-\frac{1-a^2}{2}L\Delta^2\right) \nonumber\\
		&\ge (1-\fP_i(\sE_i))\cdot \frac{7}{32} LD\Delta~.\label{eq:opt_qc}
	\end{align}
	where in the last inequality we used the fact that $a\in [0,1]$ and $\Delta \le \frac{D}{16}$.
	Next, we specify the choice of \(m\) and \(\bz_i\). We select \(m\) to ensure that the balls \(B_{d_0}(\bz_i, 5D/16)\) for \(i \in [m]\) remain disjoint. Using Lemma~\ref{lem:cover}, we choose $m = \ceil{\frac{1}{2} \left( \frac{16}{15} \right)^{d_0}}$, and \(\bz_i\) as a sequence of elements in $B_{d_0}(\bzero, 2D/3)$ such that \(B_{d_0}(\bz_i, 5D/16) \cap B_{d_0}(\bz_j, 5D/16) = \emptyset\) for \(i \neq j\).
	
	\noindent We use Pinsker's inequality to derive an upper bound on \(\fP_i(\sE_i)\). This gives  
	\[
	\frac{1}{m} \sum_{i=1}^{m} \fP_i(\sE_i) \le \frac{1}{m} \sum_{i=1}^{m} \Q(\sE_i) + \sqrt{\frac{1}{2m} \sum_{i=1}^{m} \KL\left(\Q, \fP_i\right)}~.
	\]
	Since the events \(\sE_i\) for \(i \in [m]\) are disjoint, it follows that  $\sum_{i=1}^{m} \Q(\sE_i) \le 1$, therefore the bound above gives
	\begin{equation}\label{eq:pins_qc}
		\frac{1}{m} \sum_{i=1}^{m} \bigl(1 - \fP_i(\sE_i)\bigr) \ge 1 - \frac{1}{m} - \sqrt{\frac{1}{m} \sum_{i=1}^{m} \KL\left(\Q, \fP_i\right)}~.
	\end{equation}
	
	\noindent Next, we develop an upper bound on the average \(\frac{1}{m} \sum_{i=1}^{m} \KL\left(\Q, \fP_i\right)\). Let \(N_i\) denote the number of times the algorithm queries a point \(\bx\) such that \(\bx_A \in B_{d_0}(\bz_i, 5D/16)\). Using Lemma~\ref{lem:chain_rule} in the Appendix, we obtain  
	\begin{align*}
		\text{KL}\left(\Q, \fP_i\right)  &\le \frac{d_0}{2\sigma^2}\mathbb{E}[N_{i}] \sup_{\bx_A \in B_{d_0}(\bz_i, 5D/16)} \|V(\bx)- \nabla F_i(\bx)\|^2\\
		&\qquad + \frac{d_0}{2\sigma^2} \mathbb{E}[T-N_{i}] \sup_{\bx_A \notin B_{d_0}(\bz_i, 5D/16)} \|V(\bx)- \nabla F_i(\bx)\|^2~.	
	\end{align*}
	Using the bounds in Lemma~\ref{lem:gd_bounds3} , we obtain  
	\[
	\KL\left(\Q, \fP_i\right) \le \frac{ d_0 L^2 \Delta^2}{2\sigma^2} \E[N_{i}] + \frac{d_0 L^2\tau^2 \Delta^2}{2\sigma^2}T~.
	\]  
	Averaging over \(m\) and using the fact that \(\sum_{i=1}^{m} N_i \le T\), which holds due to our choice of the sequence \((\bz_i)\), we derive the bound  
	\begin{equation}\label{eq:kl_avr_qc}
		\frac{1}{m} \sum_{i=1}^{m} \KL\left(\Q, \fP_i\right) \le \frac{ d_0 L^2 \Delta^2}{2\sigma^2} \left(\frac{1}{m}+\tau^2\right)T~.
	\end{equation}    
	
	\noindent Combining \eqref{eq:pins_qc} with \eqref{eq:kl_avr_qc} yields a lower bound on the average misidentification error \(\frac{1}{m} \sum_{i=1}^{m} (1-\fP_i(\sE_i))\). This, in turn, leads to the following lower bound when applying \eqref{eq:opt_qc}  
	\[
	\frac{1}{m} \sum_{i=1}^{m} \E_i[F_i(\hat{\bx})] \ge \frac{7}{32}LD\Delta \left(1-\frac{1}{m} - \sqrt{\frac{d_0 L^2 \Delta^2}{2\sigma^2} \left(\frac{1}{m}+\tau^2\right)T } \right).
	\]  
	To conclude, we use the fact that $m = \ceil{\frac{1}{2}(16/15)^{d_0}}$, and $d_0 = \ceil{\log_{16/15}(4/\tau^2)}$, which gives $m \ge \frac{2}{\tau^2}$.  
	\begin{align*}
		\frac{1}{m}\sum_{i=1}^{m}\mathbb{E}_i[F_{i}(\hat{\bx})] 
		&\ge \frac{7}{32}LD\Delta \left(1- \frac{\tau^2}{2}-\sqrt{3 \frac{\log_{16/15}(5/\tau^2)}{4\sigma^2} L^2\tau^2\Delta^2T }\right)\\
		&\ge \frac{7}{32}LD\Delta \left(\frac{1}{2}-\sqrt{3\frac{\log_{16/15}(5/\tau^2)}{4\sigma^2} L^2\tau^2\Delta^2T }\right)~. 
	\end{align*}
	We choose
	\[
	\Delta^*= \frac{2\sigma}{4\sqrt{3}L\tau \sqrt{8\log_{16/15}(5/\tau^2) T}}~.
	\]
	Since we have $T \ge \frac{35 \sigma^2}{L^2D^2\tau^2\log_{16/15}(2/\tau)}$, we necessarily have $\Delta^* \le D/16$. We conclude that 
	\[
	\frac{1}{m}\sum_{i=1}^{m}\mathbb{E}_i[f_{\bz_i}(\hat{\bx})] \ge c\cdot \frac{D \,\sigma}{\tau \sqrt{\log(2/\tau)T}}~,
	\] 
	where $c$ is a numerical constant.

	\section{Technical Results}\label{sec:tech}
	\begin{lemma}\label{lem:dim}
		Let $f: \R^{d_0} \to \R$ be a differentiable function having a unique global minimizer. For $d\ge d_0$ and any \(\bx \in \mathbb{R}^{d}\), let \(\bx_A \in \mathbb{R}^{d_0}\) be the vector of its first \(d_0\) coordinates, and define $F:\R^d \to \R$ by $F(\bx) = f(\bx_A)$. We have:
		\begin{itemize}
			\item If $f$ is $\tau-\QC$ then $F$ is $\tau-\QC$.
			\item If $f$ is $\mu-\QG$ then $F$ is $\mu-\QG$.
			\item If $f$ is $\mu-\RSI$ then $F$ is $\mu-\RSI$.
			\item If $f$ is $L$-smooth then $F$ is $L$-smooth.
		\end{itemize} 
	\end{lemma}
	\begin{proof}
		Let $\bz \in \R^{d_0}$ be the global minimizer of $f$, given the expression of $F$, we have that the set of global minimizers of $F$ is given by $\sX^* = \{(\bz, \by) \mathrm{s.t.} \by \in \R^{d-d_0}\}$. Observe that $\sX^*$ is convex, and the minimum of $F$ denoted $F^*$ is the same as the minimum of $f$ denoted $f^*$. It is straightforward that if $f$ is $L$-smooth, then so is $F$. Let us verify $\mu-\QG$ property, for any $\bx \in \R^d$, the projection of $\bx$ onto $\sX^*$, denoted $\bx_p$ is given by $\bx_p = (\bz, \bx_{\bar{A}})$.
		
		\noindent \emph{Verifying $\mu-\QG$} we have for any $\bx \in \R^d$
		\begin{align*}
			F(\bx)-F^* &= f(\bx_A)-f^*\\
			&\ge \frac{\mu}{2} \norm{\bx_A - \bz}^2\\
			&= \frac{\mu}{2} \norm{\bx - (\bz, \bx_{\bar{A}})}^2\\
			&= \frac{\mu}{2} \norm{\bx - \bx_p}^2~,
		\end{align*} 
		where we used the $\mu-\QG$ property of $f$ in the second line. This proves that $F$ is also $\mu-\QG$.
		
		\noindent \emph{Verifying $\tau-\QC$} we have for any $\bx \in \R^d$
		\begin{align*}
			F(\bx)-F^* &= f(\bx_{A})-f^*\\
			&\le \frac{1}{\tau} \langle \nabla f(\bx_{A}), \bx_{A}-\bz \rangle\\
			&= \frac{1}{\tau} \langle \left(\nabla f(\bx_{A}), \bzero_{\bar{A}}\right), (\bx_{A}, \bx_{\bar{A}})-(\bz, \bx_{\bar{A}}) \rangle\\
			&= \frac{1}{\tau} \langle \nabla F(\bx), \bx -\bx_p \rangle~.
		\end{align*}
		\noindent \emph{Verifying $\mu-\RSI$} we have for any $\bx \in \R^d$
		\begin{align*}
			\langle \nabla F(\bx), \bx-\bx_p \rangle  &= \langle \nabla f(\bx_A), \bx_A-\bz \rangle\\
			&\ge \mu \norm{\bx_A - \bz}^2\\
			&= \mu \norm{\bx - \bx_p}^2~. 
		\end{align*}
		
		\noindent \emph{Verifying $L$-smoothness} we have for any $\bx, \by \in \R^d$
		\[
		\norm{\nabla F(\bx)-\nabla F(\by)} = \norm{\nabla f(\bx_A) - \nabla f(\by_A)}
		\le L \norm{\bx_A - \by_A}
		\le L \norm{\bx-\by}~.
		\]

	\end{proof}
	\begin{lemma}\label{lem:lip}
		Let $R>0$ and $f: B_d(\boldsymbol{0}, R) \to \R^d$ such that for some $\bz \in B_d(\boldsymbol{0},R)$ and $r < R$, we have $B_d(\bz,r) \subseteq B_d(\boldsymbol{0},R)$, $f$ is $L$-Lipschitz on $B_d(\bz,r)$ and on $B_d(\boldsymbol{0},R)\setminus B_d(\bz,r)$, and $f$ is continuous on $B_d(\boldsymbol{0},R)$. Then, we have $f$ is $L$-Lipschitz on $B_d(\boldsymbol{0},R)$. 
	\end{lemma}
	\begin{proof}
		Let $\bx, \by \in B_d(\boldsymbol{0},R)$. If $\bx, \by$ are both in $B_d(\bz,\br)$ or in $B_d(\boldsymbol{0},R)\setminus B_d(\bz,r)$ the result is straightforward. Suppose that $\bx \in B_d(\bz,r)$ and $\by \in B_d(\boldsymbol{0},R)\setminus B_d(\bz,r)$, let $\bx'$ be the point on the segment $[\bx, \by]$ such that $\norm{\bx'-\bz} = r$, and let $\lambda \in [0,1]$ such that $\bx' = \lambda \bx + (1-\lambda) \by$. We have
		\begin{align*}
			\norm{f(\bx)-f(\by)} 
			&\le \norm{f(\bx)-f(\bx')}+\norm{f(\bx')-f(\by)}\\
			&\le L \norm{\bx - \bx'}+L\norm{\bx'-\by}\\
			&= L \norm{\bx - \by},
		\end{align*}
		where in the second inequality we use the $L$-Lipschitz and continuity property of $f$ in both domains $B_d(\bz,r)$ and $B_d(\boldsymbol{0},R)\setminus B_d(\bz,r)$, and in the equality the fact that $\bx' \in [\bx, \by]$.
	\end{proof}
	
	\begin{lemma}\label{lem:lip2}
		The function $f:B_d(\boldsymbol{0},2)\setminus B_d(\boldsymbol{0},1) \to \R^d$ defined as $f(\bx) = -\bx+2\frac{\bx}{\norm{\bx}}$, is $1$-Lipschitz.
	\end{lemma}
	\begin{proof}
		Let $\bx, \by \in B_d(\boldsymbol{0},2)\setminus B_d(\boldsymbol{0},1)$. We have
		\begin{align*}
			\norm{f(\bx)-f(\by)} 
			&= \norm{-\bx+2\frac{\bx}{\norm{\bx}}+\by-2\frac{\by}{\norm{\by}}}\\
			&= \norm{(2-\norm{\bx})\frac{\bx}{\norm{\bx}} - (2-\norm{\by})\frac{\by}{\norm{\by}}}~.
		\end{align*}
		To ease the notation, let $\bu = \frac{\bx}{\norm{\bx}}$, $\bv = \frac{\by}{\norm{\by}}$ and $c := \langle \bu, \bv \rangle \in [-1,1]$. Therefore, we have
		\begin{align*}
			\norm{f(\bx)-f(\by)}^2-\norm{\bx-\by}^2 &= \norm{(2-\norm{\bx})\bu - (2-\norm{\by})\bv}^2-\norm{\norm{\bx}\bu-\norm{\by}\bv}^2\\
			&= 4(2-(\norm{\bx}+\norm{\by}))(1-c)~.
		\end{align*}
		Then, the conclusion follows from $c \le 1$ and $\norm{\bx}+\norm{\by} \ge 2$.
	\end{proof}
	
	We consider the following Pinsker inequality from \cite{gerchinovitz2020fano}.
	\begin{lemma}\label{lem:pinsker}
		Given an underlying measurable space, for all probability pairs $\mathbb{P}_i$, $\mathbb{Q}_i$ and for all $[0,1]$-valued random variables $Z_i$ defined on this measurable space, with $i\in \{1, \dots, N\}$, where
		\[
		0 < \frac{1}{N}\sum_{i=1}^{N}\mathbb{E}_{\mathbb{P}_i}[Z_i] <1,
		\] 
		we have
		\[
		\frac{1}{N}\sum_{i=1}^{N}\mathbb{E}_{\mathbb{Q}_i}[Z_i] \le \frac{1}{N}\sum_{i=1}^{N}\mathbb{E}_{\mathbb{P}_i}[Z_i] + \sqrt{\frac{\frac{1}{N}\sum_{i=1}^{N}\text{KL}\left(\mathbb{Q}_i, \mathbb{P}_i\right)}{-\log\left( \frac{1}{N}\sum_{i=1}^{N}\mathbb{E}_{\mathbb{P}_i}[Z_i]\right)}},
		\]
		and
		\[
		\frac{1}{N}\sum_{i=1}^{N}\mathbb{E}_{\mathbb{Q}_i}[Z_i] \le \frac{1}{N}\sum_{i=1}^{N}\mathbb{E}_{\mathbb{P}_i}[Z_i] + \sqrt{\frac{\frac{1}{N}\sum_{i=1}^{N}\text{KL}\left(\mathbb{Q}_i, \mathbb{P}_i\right)}{2}}~.
		\]
	\end{lemma}

	\begin{lemma}\label{lem:cover}
		Let $\Delta >0$, the maximal number $m$ of elements $(\bz_i)_{i \in [m]}$ in $\R^d$ such that $\norm{\bz_i} \le 5\Delta$ and $\norm{\bz_i - \bz_j} >4\Delta$ for each $i\neq j$ satisfies
		\[
		m \ge \left(\frac{5}{4}\right)^d~.
		\]
	More generally, for centers in $B_d(\bzero,R)$ separated by more than $2r$, the maximal number is at least $(R/(2r))^d$.
	\end{lemma}
	\begin{proof}
        A maximal set of centers with pairwise distances greater than $2r$ covers $B_d(\bzero,R)$ by balls of radius $2r$. Comparing volumes gives $R^d\le m(2r)^d$. Taking $R=5\Delta$ and $r=2\Delta$ yields the displayed bound.
    \end{proof}
	
	\section{Canonical Model and Divergence Decomposition}\label{sec:can_model}
	
	We introduce a formal framework that enables a rigorous development of information-theoretic tools, adapting the one considered in  \cite{garivier2019explore,lattimore2020bandit, hadiji2019polynomial}. 
	
	Let \(h_0,h_1,\dots,h_m:\mathbb{R}^{d_0}\to\mathbb{R}^{d_0}\) be measurable feedback mean fields. In our lower-bound proofs, \(h_i=\nabla f_i\) for \(i\ge1\), while \(h_0=v\) is the reference field, which need not be continuous or a gradient. The reference law is used only for comparison and need not correspond to an objective in the function class. Suppose there are \(T\) rounds in total. Let $U\in[0,1]$ be an independent random seed with the same law under every hypothesis, representing the learner's internal randomness, and let $\Omega_0=[0,1]$. For each \(t\), define the measurable space $
	\Omega_t \;=\; [0,1]\times\prod_{s=1}^{t} \bigl(\mathbb{R}^{d_0} \times \mathbb{R}^{d_0}\bigr),
	$
	equipped with the usual Borel \(\sigma\)-algebra. Let
	$
	\mathbf{y}_t \;=\; (U,\mathbf{x}_1, \mathbf{g}_1, \ldots, \mathbf{x}_t, \mathbf{g}_t) \;\in\; \Omega_t
	$
	represent all information available to the learner up to time \(t\).
	
	We consider an optimization algorithm \(\mathcal{A}\) specified by a sequence \(\bigl(K_t\bigr)_{1 \le t \le T}\) of probability kernels \(K_t : \Omega_{t-1}\times \mathcal{B}(\R^{d_0}) \to [0,1]\), which model the choice of the query point \(\mathbf{x}_t\) at time \(t\). More precisely, in round $t$ given the history $\mathbf{y}_{t-1} \in \Omega_{t-1}$, the distribution of $\mathbf{x}_t$ is given using the probability measure $A \mapsto K_t(\mathbf{y}_{t-1}, A)$.  At the first round, \(\mathbf{x}_1\) is chosen according to a distribution defined by \(K_1(U,\cdot)\).
	
	For each \(i \in \{0, \dots, m\}\), let 
	$
	H_{i,t} : \Omega_{t-1} \times \mathbb{R}^{d_0} \times \mathcal{B}\left(\mathbb{R}^{d_0} \right) \;\to\; [0,1]
	$
	be the probability kernel that models the feedback observed by the learner. For any measurable set \(B \in \mathcal{B}\left(\mathbb{R}^{d_0}\right) \), we take
	\[
	H_{i,t}(\by_{t-1}, \mathbf{x}_t, B) = \left(\frac{d_0}{2\pi\sigma^2}\right)^{d_0/2} \int_{\mathbf{x} \in B} \exp\left(-\frac{d_0}{2\sigma^2}\norm{\mathbf{x}-h_i(\mathbf{x}_t)}^2\right) \diff \mathbf{x}~.
	\]
	These kernels define the probability laws \(\mathbb{P}_{i,t} = H_{i,t} \circ K_t \circ \mathbb{P}_{i,t-1} \) over \(\Omega_t\). This construction ensures that, under \(\mathbb{P}_{i,t}\), the coordinate random variables \(\mathbf{X}_t: \Omega_t \to \mathbb{R}^{d_0}\) and \(\mathbf{G}_t: \Omega_t \to \mathbb{R}^{d_0}\), defined by $
	\mathbf{X}_t(u,\mathbf{x}_1, \mathbf{g}_1, \dots, \mathbf{x}_t, \mathbf{g}_t) = \mathbf{x}_t$
	and 
	$
	\mathbf{G}_t(u,\mathbf{x}_1, \mathbf{g}_1, \dots, \mathbf{x}_t, \mathbf{g}_t) = \mathbf{g}_t,
	$
	satisfy the following property: given the history and \(\mathbf{X}_t\), the feedback \(\mathbf{G}_t\) is distributed according to \(\mathcal{N}_{d_0}\bigl(h_i(\mathbf{X}_t), \frac{\sigma^2}{d_0}\mathbf{I}_{d_0}\bigr)\), where \(\mathbf{I}_{d_0}\) is the \(d_0 \times d_0\) identity matrix. We denote by \(\mathbb{E}_i\) the expectation taken with respect to \(\mathbb{P}_{i,t}\).

	The initial laws $\mathbb{P}_{i,0}$ are all the law of $U$. Write $\mathbb{P}_i^T$ for the law of the transcript augmented with the reported output, generated by the same output kernel under every hypothesis. Appending this common kernel leaves the KL divergence unchanged, so the following decomposition also applies to identification events involving randomized outputs.

    For the $d$-dimensional constructions, a learner can be simulated in this reduced model: retain its remaining query coordinates as internal state, query only $\bx_A$, and supply zero feedback in the remaining coordinates. This reproduces its full queries and output because every constructed gradient and reference field is zero outside $A$. Thus the same bound applies to the full transcript, with $S$ tested against $\bx_A$.

    \begin{lemma}\label{lem:chain_rule}
		Consider the notation above, let $S \subset \R^{d_0}$ be measurable and $N_{S}$ denote the variable corresponding to the total number of times the learner queries a point from $S$ after $T$ rounds. Then, for any $i \in \{0, \dots, m\}$, we have 
		\begin{align*}
			\text{KL}\left(\mathbb{P}_0^T, \mathbb{P}_i^T\right)  
			&\le \frac{d_0}{2\sigma^2} \mathbb{E}_0[N_{S}] \sup_{\mathbf{x} \in S} \ \norm{h_0(\mathbf{x})- h_i(\mathbf{x})}^2\\
			&\qquad + \frac{d_0}{2\sigma^2} \mathbb{E}_0[T-N_{S}] \sup_{\mathbf{x} \notin S} \ \norm{h_0(\mathbf{x})- h_i(\mathbf{x})}^2~.
		\end{align*}
	\end{lemma}
	\begin{proof}
		\noindent Under the setting introduced, we have the chain rule
		\begin{align*}
			\text{KL}(\mathbb{P}_{0,t}, \mathbb{P}_{i, t}) &= \text{KL}\left(H_{0,t}\circ K_t \circ \mathbb{P}_{0,t-1}, H_{i,t} \circ K_t \circ \mathbb{P}_{i,t-1}\right)\\
			&= \text{KL}(K_t \circ \mathbb{P}_{0,t-1}, K_t \circ \mathbb{P}_{i,t-1})\\
			&\qquad + \int_{\Omega_{t-1}\times \mathbb{R}^{d_0}} \text{KL}\left(H_{0,t}(\mathbf{y}_{t-1}, \mathbf{x}_t,\cdot), H_{i,t}(\mathbf{y}_{t-1}, \mathbf{x}_t, \cdot)\right) \boldsymbol{\mathrm{d}}K_t\circ\mathbb{P}_{0,t-1}(\mathbf{y}_{t-1}, \mathbf{x}_t)\\
			&= \text{KL}(\mathbb{P}_{0,t-1}, \mathbb{P}_{i,t-1})\\
			&\quad + \int_{\Omega_{t-1}\times \mathbb{R}^{d_0}} \text{KL}\left(H_{0,t}(\mathbf{y}_{t-1}, \mathbf{x}_t,\cdot), H_{i,t}(\mathbf{y}_{t-1}, \mathbf{x}_t, \cdot)\right) \boldsymbol{\mathrm{d}}K_t\circ\mathbb{P}_{0,t-1}(\mathbf{y}_{t-1}, \mathbf{x}_t)\\
			&= \text{KL}(\mathbb{P}_{0,t-1}, \mathbb{P}_{i,t-1}) \\
			&\quad + \int_{\Omega_{t-1}\times \mathbb{R}^{d_0}} \text{KL}\left(\mathcal{N}_{d_0}\left(h_0(\mathbf{x}_t), \frac{\sigma^2}{d_0}\bI_{d_0}\right), \mathcal{N}_{d_0}\left(h_i(\mathbf{x}_t), \frac{\sigma^2}{d_0}\bI_{d_0}\right)\right) \boldsymbol{\mathrm{d}}K_t\circ\mathbb{P}_{0,t-1}(\mathbf{y}_{t-1}, \mathbf{x}_t)\\
			&= \text{KL}\left(\mathbb{P}_{0,t-1}, \mathbb{P}_{i,t-1}\right)\\
			& \quad +\mathbb{E}_0 \left[ \text{KL}\left( \mathcal{N}_{d_0}\left(h_0(\bX_t), \frac{\sigma^2}{d_0}\bI_{d_0}\right), \mathcal{N}_{d_0}\left(h_i(\bX_t), \frac{\sigma^2}{d_0}\bI_{d_0}\right)\right)\right],
		\end{align*}
		where we used in the third line the fact that the algorithm $\mathcal{A}$ is fixed and the learner internal randomization is independent of the objective function, therefore given a fixed history input $\mathbf{y}_{t-1}$ we have $K(\cdot \mid \mathbf{y}_{t-1})$ is identical under the two hypotheses. The Gaussian-kernel equality comes from the fact that the density of the kernel $H_{i,t}$ depends only on the last coordinate $\bx_t$, and is exactly that of a Gaussian variable.
		
		\noindent Iterating $T$ times, we obtain
		\begin{align}
			\text{KL}\left(\mathbb{P}_0^T, \mathbb{P}_i^T\right) &= \mathbb{E}_0\left[\sum_{t=1}^{T} \text{KL}\left( \mathcal{N}_{d_0}\left(h_0(\bX_t), \frac{\sigma^2}{d_0}\bI_{d_0}\right), \mathcal{N}_{d_0}\left(h_i(\bX_t), \frac{\sigma^2}{d_0}\bI_{d_0}\right) \right) \right]\nonumber\\
			&= \sum_{t=1}^{T} \mathbb{E}_0 \left[ \frac{d_0}{2\sigma^2}\norm{h_0(\bX_t)-h_i(\bX_t)}^2\right],\label{kl:1}
		\end{align}
		where we used the expression of the $\text{KL}$ between two normal distributions with the same covariance matrix and difference means. Let $S \subset \R^{d_0}$, for each $i \in \{0, \dots, m\}$, we have 
		\begin{align}
			\mathbb{E}_0\left[ \norm{h_0(\bX_t) -h_i(\bX_t)}^2\right] 			&= \mathbb{P}_0\left(\bX_t \in S\right) \mathbb{E}_0\left[ \norm{h_0(\bX_t)-h_i(\bX_t)}^2 \mid \bX_t \in S \right]\nonumber\\
			&\qquad +\mathbb{P}_0\left( \bX_t \notin S\right) \mathbb{E}_0\left[ \norm{h_0(\bX_t)-h_i(\bX_t)}^2 \mid \bX_t \notin S \right]\nonumber\\
			&\le \mathbb{P}_0\left(\bX_t \in S\right) \sup_{\mathbf{x} \in S} \norm{h_0(\mathbf{x})- h_i(\mathbf{x})}^2\nonumber\\
			&\qquad +\mathbb{P}_0\left( \bX_t \notin S\right) \sup_{\mathbf{x} \notin S} \norm{h_0(\mathbf{x})- h_i(\mathbf{x})}^2~.\label{kl:2}
		\end{align}
		We have
		\begin{align}
			\mathbb{E}_0[N_{S}] &= \mathbb{E}_0\left[\sum_{t=1}^{T} \mathds{1}\left(\bX_t \in S\right)\right]
			= \sum_{t=1}^{T} \mathbb{P}_0\left( \bX_t \in S\right)~.\label{kl:3}
		\end{align}
		Combining \eqref{kl:1}, \eqref{kl:2}, and \eqref{kl:3}, we have the stated result.
	\end{proof}

	\section{Proof of Theorem~\ref{thm:ub1}}\label{sec:proof_ub1}
	We show the guarantees for the two considered classes separately.
	
	\begin{theorem}\label{thm:qcqgu}
		Let $f$ be a function satisfying the $\mu$-RSI with respect to all the global minimizers of $f$ and $L$-smooth. Then, running SGD with step-sizes $\eta_t = \frac{2}{\mu (t+\frac{2L^2}{\mu^2}+1)}$ and initial point $\bx_1$ guarantees
		\[
		\E[f(\hat{\bx}_T)-f^\star] 
		\le\frac{\mu^2L^3+2L^5}{\mu^4 T(T+1)}\norm{\bx_1-\bx_1^*}^2+\frac{2L\sigma^2}{\mu^2(T+1)},	
		\]
		where $\bx^*_t$ is the projection of $\bx_t$ on the set of global minimizers and $\hat{\bx}_T$ is chosen from $\bx_1, \dots, \bx_T$ with weights $w_t = \frac{2t}{T(T+1)}$.
	\end{theorem}
	\begin{proof}
		From $L$-smooth, we have 
		\[\mathbb{E}\left[f(\bx_t)-f^*\right]\leq \frac{L}{2}\mathbb{E}\left[\norm{\bx_t-\bx_t^*}^2\right]~.\]
		Using $\mu$-RSI, we further have
		\begin{align*}
			\mathbb{E}\left[f(\bx_t)-f^*\right]&\leq \frac{L}{\mu}\mathbb{E}\left[\langle \nabla f(\bx_t),\bx_t-\bx_t^*\rangle\right]-\frac{L}{\mu}\mathbb{E}\left[\langle \nabla f(\bx_t),\bx_t-\bx_t^*\rangle\right]+\frac{L}{2}\mathbb{E}\left[\norm{\bx_t-\bx_t^*}^2\right]\\
			&\leq \frac{L}{\mu}\mathbb{E}\left[\langle \nabla f(\bx_t),\bx_t-\bx_t^*\rangle\right] -\frac{L}{\mu}\mu \mathbb{E}\left[\norm{\bx_t-\bx^*_t}^2\right]+\frac{L}{2}\mathbb{E}\left[\norm{\bx_t-\bx_t^*}^2\right]\\
			&=  \frac{L}{\mu}\mathbb{E}\left[\langle \nabla f(\bx_t),\bx_t-\bx_t^*\rangle\right] -\frac{L}{2}\mathbb{E}\left[\norm{\bx_t-\bx_t^*}^2\right]~.
		\end{align*}
		That is,
		\begin{align*}
			\frac{\mu}{L}\mathbb{E}\left[f(\bx_t)-f^*\right]&\leq \mathbb{E}\left[\langle \nabla f(\bx_t), \bx_t-\bx^*_t\rangle\right]-\frac{\mu}{2}\mathbb{E}\left[\norm{\bx_t-\bx_t^*}^2\right]\\
			&\leq \frac{1}{2\eta_t}\mathbb{E}\left[\norm{\bx_t-\bx_t^*}^2\right]-\frac{1}{2\eta_t}\mathbb{E}\left[\norm{\bx_{t+1}-\bx^*_t}^2\right]+\frac{\eta_t}{2}\mathbb{E}\left[\norm{\bg_t}^2\right]-\frac{\mu}{2}\mathbb{E}\left[\norm{\bx_t-\bx_t^*}^2\right]\\
			&\leq \frac{1}{2\eta_t}\mathbb{E}\left[\norm{\bx_t-\bx_t^*}^2\right]-\frac{1}{2\eta_t}\mathbb{E}\left[\norm{\bx_{t+1}-\bx^*_{t+1}}^2\right]+\frac{\eta_t}{2}\mathbb{E}\left[\norm{\bg_t}^2\right]-\frac{\mu}{2}\mathbb{E}\left[\norm{\bx_t-\bx_t^*}^2\right]~.
		\end{align*}
		Multiplying both sides by $t+a$ and choosing $\eta_t=\frac{2}{\mu (t+a+1)}$, yields
		\begin{align*}
			\frac{\mu}{L}(t+a)\mathbb{E}\left[f(\bx_t)-f^*\right]
			&\leq \frac{t+a}{2\eta_t}\mathbb{E}\left[\norm{\bx_t-\bx_t^*}^2\right]-\frac{t+a}{2\eta_t}\mathbb{E}\left[\norm{\bx_{t+1}-\bx^*_{t+1}}^2\right]\\
			&\quad +\frac{\eta_t (t+a)}{2}\mathbb{E}\left[\norm{\bg_t}^2\right]-\frac{\mu (t+a)}{2}\mathbb{E}\left[\norm{\bx_t-\bx_t^*}^2\right]\\
			&\leq \frac{\mu (t+a)(t+a+1)}{4}\mathbb{E}\left[\norm{\bx_t-\bx_t^*}^2\right]-\frac{\mu (t+a)(t+a+1)}{4}\mathbb{E}\left[\norm{\bx_{t+1}-\bx_{t+1}^*}^2\right]\\
			&\quad +\frac{t+a}{\mu(t+a+1)}(\mathbb{E}\left[ \norm{\nabla f(\bx_t)}^2\right]+\sigma^2)-\frac{\mu(t+a)}{2}\mathbb{E}\left[\norm{\bx_t-\bx_t^*}^2\right]\\
			&\leq \frac{\mu (t+a-1)(t+a)}{4}\mathbb{E}\left[\norm{\bx_t-\bx_t^*}^2\right]-\frac{\mu (t+a)(t+a+1)}{4}\mathbb{E}\left[\norm{\bx_{t+1}-\bx_{t+1}^*}^2\right]\\
			&\quad +\frac{t+a}{\mu(t+a+1)}\left(\mathbb{E}\left[ \norm{\nabla f(\bx_t)}^2\right]+\sigma^2\right)~.
		\end{align*}
		From $L$-smooth, we have $\norm{\nabla f(\bx_t)}^2\leq 2L(f(\bx_t)-f^*)$. Summing from $t=1,\dots,T$ yields
		\begin{align*}
			\sum_{t=1}^T(t+a)\mathbb{E}\left[f(\bx_t)-f^*\right]\leq \frac{La(a+1)}{4}\mathbb{E}\left[\norm{\bx_1-\bx_1^*}^2\right]+\frac{2L^2}{\mu^2}\sum_{t=1}^T\mathbb{E}\left[f(\bx_t)-f^*\right]+\frac{L T\sigma^2}{\mu^2}
		\end{align*}
		Picking $a=\frac{2L^2}{\mu^2}$, we have
		\[\sum_{t=1}^T t\mathbb{E}\left[f(\bx_t)-f^*\right]\leq \frac{\mu^2L^3+2L^5}{2\mu^4}\mathbb{E}\left[\norm{\bx_1-\bx_1^*}^2\right]+\frac{L T\sigma^2}{\mu^2}~.\]
		Consider $\hat{\bx}_T$ as sampling iterates from $\bx_1,\dots,\bx_T$ with weights $1,\dots,T$, we have
		\begin{align*}
			\mathbb{E}\left[f(\hat{\bx}_T)-f^*\right]
			&=\sum_{t=1}^T\frac{t}{\frac{T(T+1)}{2}}\mathbb{E}\left[f(\bx_t)-f^*\right]
			\leq \left(\frac{\mu^2L^3+2L^5}{2\mu^4}\mathbb{E}\left[\norm{\bx_1-\bx_1^*}^2\right]+\frac{LT\sigma^2}{\mu^2}\right)\frac{2}{T(T+1)}\\
			&=\frac{\mu^2L^3+2L^5}{\mu^4 T(T+1)}\mathbb{E}\left[\norm{\bx_1-\bx_1^*}^2\right]+\frac{2L\sigma^2}{\mu^2(T+1)}~. 
		\end{align*}
	\end{proof}
	
	\begin{theorem}
		Let $f$ be a function satisfying the $\mu$-QG and $\tau$-WQC condition with respect to all the global minimizers of $f$. Consider running SGD,
		\[
		\bx_{t+1} = \bx_t - \eta_t \bg_t~.
		\]
		If $f$ is $L$-smooth with $\eta_t = \frac{4}{\tau\mu (t+a)}$, we have
		\[
		\E[f(\hat{\bx}_T)-f^\star] \le \frac{\mu(1+a)^2}{2T(T+1)}\norm{\bx_1-\bx_1^*}^2 +  \frac{16\sigma^2}{\tau^2\mu(T+1)},	
		\]
		where $a=\frac{16L}{\tau^2\mu}$, $\bx^*_t$ is the projection of $\bx_t$ on the set of global minimizers, and $\hat{\bx}_T$ is chosen at random from $\bx_1, \dots, \bx_T$ with weights $w_t = \frac{2t}{T(T+1)}$.
	\end{theorem}
	\begin{proof}
		Let $t \in [T]$, denote by $\bx^*_t$ the projection of $\bx_t$ onto the set of global minimizers $\mathcal{S}$, using the fact that $f$ is WQC, we have
		\begin{align*}
			\eta_t\left(f(\bx_t)-f^* \right) &\le  \frac{\eta_t}{\tau} \langle \nabla f(\bx_t), \bx_t-\bx_t^* \rangle\\
			&= \frac{1}{\tau} \mathbb{E}_t\left[ \langle\eta_t  \bg_t, \bx_t-\bx_t^* \rangle\right] \\
			&= \frac{1}{2\tau}\norm{\bx_t-\bx_t^*}^2-\frac{1}{2\tau}\mathbb{E}_t\left[\norm{\bx_{t+1}-\bx_t^*}^2 \right]+\frac{\eta_t^2}{2\tau} \mathbb{E}_t\left[\norm{\bg_t}^2 \right]\\
			&\le \frac{1}{2\tau}\norm{\bx_t-\bx_t^*}^2-\frac{1}{2\tau}\mathbb{E}_t\left[\norm{\bx_{t+1}-\bx_{t+1}^*}^2 \right]+\frac{\eta_t^2}{2\tau} \mathbb{E}_t\left[\norm{\bg_t}^2 \right]. 
		\end{align*}
		Dividing by $\frac{\eta_t}{t+a}$ and summing for $t$ from $1$ to $T$, we have
		\begin{align*}
			\sum_{t=1}^{T} (t+a)\mathbb{E}\left[f(\bx_t)-f^* \right] &\le \sum_{t=1}^{T} \left( \frac{t+a}{2\tau \eta_t} \mathbb{E}\left[\norm{\bx_t-\bx_t^*}^2 \right]-\frac{t+a}{2\tau \eta_t} \mathbb{E}\left[\norm{\bx_{t+1}-\bx_{t+1}^*}^2\right]\right)+\sum_{t=1}^T\frac{\eta_t(t+a)}{2\tau} \mathbb{E}\left[\norm{\bg_t}^2\right]\\
			&\leq \sum_{t=1}^{T} \left( \frac{t+a}{2\tau \eta_t} \mathbb{E}\left[\norm{\bx_t-\bx_t^*}^2 \right]-\frac{t+a}{2\tau \eta_t} \mathbb{E}\left[\norm{\bx_{t+1}-\bx_{t+1}^*}^2\right]\right)\\&+\sum_{t=1}^T\frac{\eta_t(t+a)}{\tau} \mathbb{E}\left[\norm{\nabla f(\bx_t)}^2\right]+\sum_{t=1}^T\frac{\eta_t(t+a)\sigma^2}{\tau}\\
			&\le \frac{1+a}{2\tau \eta_1} \norm{\bx_1-\bx_1^*}^2 + \sum_{t=1}^{T-1} \left(\frac{t+1+a}{2\tau\eta_{t+1}}-\frac{t+a}{2\tau\eta_t}\right)\mathbb{E}\left[\norm{\bx_{t+1}-\bx_{t+1}^*}^2\right]\\&+\sum_{t=1}^{T} \frac{\eta_t(t+a)}{\tau} \mathbb{E}\left[\norm{\nabla f(\bx_t)}^2\right]+\sum_{t=1}^{T} \frac{\eta_t(t+a)\sigma^2}{\tau}~. 
		\end{align*}
		Using $\frac{t+1+a}{2\tau\eta_{t+1}}-\frac{t+a}{2\tau\eta_t} = \frac{\mu(2t+2a+1)}{8}$ and the definition of the QG condition stating that
		\[
		\norm{\bx_{t+1}-\bx_{t+1}^*}^2 \le \frac{2}{\mu} (f(\bx_{t+1})-f^*), 
		\]
		we get
		\begin{align*}
			\sum_{t=1}^{T} (t+a)\mathbb{E}\left[f(\bx_t)-f^* \right] &\le \frac{\mu(1+a)^2}{8}\norm{\bx_1-\bx_1^*}^2 + \frac{1}{4}\sum_{t=1}^{T} (2t+2a)\mathbb{E}\left[f(\bx_t)-f^* \right]+ \\
			&\quad +\sum_{t=1}^{T} \frac{\eta_t(t+a)}{\tau} \mathbb{E}\left[\norm{\nabla f(\bx_t)}^2\right]+\sum_{t=1}^{T} \frac{\eta_t(t+a)\sigma^2}{\tau}\\
			&\leq \frac{\mu(1+a)^2}{8}\norm{\bx_1-\bx_1^*}^2 + \frac{1}{2}\sum_{t=1}^{T} (t+a)\mathbb{E}\left[f(\bx_t)-f^* \right]\\
			&\quad +\sum_{t=1}^{T} \frac{4}{\tau^2\mu} \mathbb{E}\left[\norm{\nabla f(\bx_t)}^2\right]+\sum_{t=1}^{T} \frac{4\sigma^2}{\tau^2\mu}~.
		\end{align*}
		We conclude that in that case that $f$ is $L$-Lipschitz, picking $a=0$ we have
		\[
		\sum_{t=1}^{T} t\mathbb{E}\left[f(\bx_t)-f^* \right] \le \frac{\mu}{4}\norm{\bx_1-\bx_1^*}^2 +  \frac{8T}{\tau^2\mu}(L^2+\sigma^2)~.	
		\]
		Sampling proportional to $t$ yields
		\[
		\E[f(\hat{\bx}_T)-f^\star]\leq \frac{\mu}{2T(T+1)}\norm{\bx_1-\bx_1^*}^2+\frac{16(L^2+\sigma^2)}{\tau^2\mu(T+1)}~.
		\]
		
		For smooth functions, since $\norm{\nabla f(\bx_t)}^2\leq 2L(f(\bx_t)-f^\star)$, we conclude that
		\[
		\sum_{t=1}^{T} \left((t+a)-\frac{t+a}{2}-\frac{8L}{\tau^2\mu}\right)\mathbb{E}\left[f(\bx_t)-f^* \right] \le \frac{\mu(1+a)^2}{8}\norm{\bx_1-\bx_1^*}^2 +  \frac{4T\sigma^2}{\tau^2\mu}~.	
		\]
		Setting $a=\frac{16L}{\tau^2\mu}$, we have 
		\[
		\E[f(\hat{\bx}_T)-f^\star] \le \frac{\mu(1+a)^2}{2T(T+1)}\norm{\bx_1-\bx_1^*}^2 +  \frac{16\sigma^2}{\tau^2\mu(T+1)}~.	
		\]
	\end{proof}

	\section{Proof of Theorem~\ref{thm:dim1}}\label{sec:proof_dim1}
	First, we present the pseudocode for the Algorithm~\ref{algo:n1}.
	\begin{algorithm}
		\caption{ \label{algo:n1} }
		\begin{algorithmic}
			\STATE \textbf{Input}: $ T, D, \mu, L, \sigma, \delta$.
			\STATE Let $K=4\ceil*{L/(4\mu)}$ and $n=\max\{1,\ceil*{\log_{4/3}(L\mu D^2T/(192\sigma^2))}\}$.
			\STATE Let $\Delta=\frac{8}{\mu}\sigma\sqrt{\frac{4n\log(2n(K+1)/\delta)}{T}}$.
			\STATE \textbf{Initialize:} Let $k\gets1$, $a \gets 0$, $b_{\ell} \gets -D/2$, $b_u \gets D/2$.
			\WHILE{$k \le n$}
			\STATE Let $w\gets b_u-b_{\ell}$.
			\STATE $I\gets[a-\min\{\Delta,w/2\},a+\min\{\Delta,w/2\}]$.
			\STATE Let $\gamma\gets\min\{2\Delta,w\}$ be the diameter of $I$.
			\STATE Let $x_i=a-\gamma/2+i\gamma/K$ for $i\in\{0,\dots,K\}$. 
			\STATE Query $\floor*{\frac{T}{n(K+1)}}$ gradients at each point $x_i$ for $i \in \{0, \dots, K\}$.
			\STATE Compute for each $i \in \{0, \dots, K\}$ let $g_{i}$ denote the empirical mean of all the gradients of the points in the set $\mathcal{P}_i$ defined as follows:
			\begin{itemize}
				\item For $i \in \{0, \dots, K/4-1\}$:  $\mathcal{P}_i = \{x_{i}, \dots, x_{i+K/4}\}$.
				\item For $i \in \{K/4, \dots, K/2\}$: $\mathcal{P}_i =\{x_{i}, \dots, x_{3K/4}\}$.
				\item For $i \in \{K/2+1, \dots, 3K/4\}$: $\mathcal{P}_i = \{x_{K/4}, \dots, x_{i}\}$.
				\item For $i \in \{3K/4+1, \dots, K\}$: $\mathcal{P}_i =\{x_{i-K/4}, \dots, x_{i}\}$.
			\end{itemize} 
			\STATE Let: $z \in \argmax\limits_{i \in [K/4, 3K/4]}\{\abs{g_i} \}$
			\IF{$ \abs{g_{z}} \le \frac{\mu}{2}  \Delta $}
			\IF{$\max\limits_{i\in [0, K/4-1]}g_{i} < \frac{\mu}{2}\Delta$ and $\min\limits_{i\in [3K/4+1, K]} g_{i} \ge -\frac{\mu}{2}\Delta$}
			\STATE Break (out of the while loop)
			\ELSIF {$\max\limits_{i\in [0, K/4-1]}g_{i} \ge \frac{\mu}{2}\Delta$}
			\STATE $b_u \gets a$, $a \gets \frac{1}{2}(b_u+b_{\ell})$, $k \gets k+1$
			\ELSIF{$\min\limits_{i\in [3K/4+1, K]} g_{i}<-\frac{\mu}{2}\Delta$}
			\STATE  $b_{\ell} \gets a$, $a \gets \frac{1}{2}(b_u+b_{\ell})$, $k \gets k+1$
			\ENDIF
			\ELSIF{$g_{z} \ge 0$}
			\STATE $b_u \gets a+\gamma/4$, $a \gets \frac{1}{2}(b_u+b_{\ell})$, $k \gets k+1$
			\ELSE
			\STATE $b_{\ell} \gets a-\gamma/4$, $a \gets \frac{1}{2}(b_u+b_{\ell})$, $k \gets k+1$
			\ENDIF
			\ENDWHILE
			\STATE \textbf{return} $a$		
		\end{algorithmic}
	\end{algorithm}
	Below we restate Theorem~\ref{thm:dim1}.
	\begin{theorem}
		Let $f:\R \to \mathbb{R}$ be a $L$-smooth and $\mu$-RSI function.
		For $D,\sigma>0$ and $\delta\in(0,1)$, set
		\[
		K=4\ceil*{\frac{L}{4\mu}},\qquad
		n=\max\left\{1,\ceil*{\log_{4/3}\left(\frac{L\mu D^2T}{192\sigma^2}\right)}\right\}.
		\]
		Suppose that $T\ge 2(K+1)n$, and for some minimizer of $f$ we have $\abs{x^*}\le D/2$. Then, given access to $T$ queries from a $\sigma$-subGaussian oracle, the output of Algorithm~\ref{algo:n1} with input $(T,D,\mu,L,\sigma,\delta)$ satisfies with probability at least $1-\delta$
		\[
		f(a)-f(x^*)\le \frac{1024\sigma^2}{\mu T}\,n\log\left(\frac{2n(K+1)}{\delta}\right)~.
		\]
	\end{theorem}
	\begin{proof}
		By definition, $K$ is a positive multiple of $4$ and $L/K\le\mu$. Let $\mathcal{X}^*$ denote the set of global minimizers of $f$, recall that following Lemma~\ref{lem:1d-rsi}, $\mathcal{X}^*$ is a convex set, therefore it corresponds to an interval. Since we assumed that for some element $x^* \in \sX^*$ we have $\abs{x^*} \le D/2$, we conclude that the intersection of $\sS^* = \sX^* \cap [-D/2, D/2]$ is an interval.
		First, let us show that the total number of queries made by the algorithm is upper bounded by $T$. In each iteration of the while loop, the algorithm queries at most $\frac{T}{n}$ gradients and the maximum number of iterations is $n$, 
		therefore the total number of queries is at most $T$.
		
		\noindent The proof is organized into 3 parts:
		\begin{itemize}
			\item Part 1: We develop a concentration bound on the empirical means computed.
			\item Part 2: We show that if the procedure performs $k$ iterations for some $k \in \{1, \dots, n\}$, then, with probability at least $1-\delta$, for each $q \in \{1, \dots, k\}$ we have $\mathcal{S}^* \cap [b_{\ell}^{(q)}, b_{u}^{(q)}] \neq \emptyset$.
			\item Part 3: We develop an upper bound on $f(a_k)-f^*$ if the procedure stops after $k$ iterations, for $k \in \{1, \dots, n\}$ and conclude.
		\end{itemize}
		
		\paragraph{Notation:} In iteration  $k$ of the while loop, let $a_k$ denote the value of the variable $a$, $b^{(k)}_{\ell}$ and $b_u^{(k)}$ be the values of $b_{\ell}$ and $b_u$ respectively. Let $\gamma_k = \min\{2\Delta, b_u^{(k)}-b_{\ell}^{(k)}\}$ be the value of the variable $\gamma$ at iteration $k$. Let $\mathbb{P}_{k-1}$ denote the conditional probability distribution with respect to the gradients received prior to iteration  $k$, denote by $a_k^*$ the projection of $a_k$ onto the set $\mathcal{S}^*$.
		Let $x_0^{(k)}, \dots, x^{(k)}_{K}$ be the uniform grid sequence chosen in Algorithm~\ref{algo:n1} at iteration $k$. For $i \in \{0, \dots, K\}$, let $g_i^{(k)}$ denote the value of $g_i$. Let $\lambda_i^{(k)}$ be defined as
		\[
		\lambda^{(k)}_i := \mathbb{E}_{k-1}\left[g_i^{(k)}\right]~.
		\]
		\paragraph{Part 1:}  
		From the assumptions, we have $T \ge 2(K+1)n$. Observe that the number of gradients used to compute $g^{(k)}_{i}$ is at least equal to $\frac{K+4}{4}\floor{\frac{T}{n(K+1)}} \ge \frac{K+4}{4}\frac{T}{2n(K+1)}\ge\frac{T}{8n}$.
		
		We use the following concentration bound, which is a direct consequence of Lemma~\ref{lem:cher}.
		For each iteration of the while loop $k \in \{1, \dots, n\}$ and for each $i \in \{0, \dots, K\}$, we have
		\[
		\mathbb{P}_{k-1}\left( \abs{ g^{(k)}_{i} - \lambda^{(k)}_i} \ge 2\sigma \sqrt{\frac{4n\log(2n(K+1)/\delta)}{T}} \right) \le \frac{\delta}{(K+1) n}~. 
		\]
		Recall that following the expression of $\Delta$ we have $\frac{\mu}{4}\Delta = 2\sigma \sqrt{\frac{4n\log(2n(K+1)/\delta)}{T}}$. Define the event
		\[
		\mathcal{E} := \left\lbrace \exists k \in \{1, \dots, n\}, \exists i \in \{0, \dots, K\}: \abs{ g^{(k)}_{i} - \lambda^{(k)}_i} \ge \frac{\mu}{4}\Delta \right\rbrace~.
		\]
		\noindent Therefore, using a union bound, we obtain
		\begin{equation}\label{eq:conc}
			\mathbb{P}\left( \mathcal{E} \right) \le \delta~. 	
		\end{equation}
		\paragraph{Part 2:} Next, let us show the following result: if Algorithm~\ref{algo:n1} makes at least $k$ iterations, if $\neg \mathcal{E}$ holds, then at iteration $k$ we have that $\mathcal{S}^* \cap [b^{(k)}_{\ell}, b_{u}^{(k)}] \neq \emptyset$.
		
		Let $k \in \{1, \dots, n\}$, suppose that the total number of iterations is at least $k$. We use a contradiction argument: suppose that event $\neg \mathcal{E}$ holds, assume that for some $q \in \{2, \dots, k\}$ we have $\mathcal{S}^* \cap [b^{(q)}_{\ell}, b^{(q)}_{u}] = \emptyset$, and let $q$ be the such smallest integer (i.e., $ \mathcal{S}^{*} \cap [b^{(q-1)}_{\ell}, b^{(q-1)}_{u}] \neq \emptyset$). Recall that $\mathcal{S}^*$ is an interval, without loss of generality, assume that $\max \{\mathcal{S}^* \} < b^{(q)}_{\ell}$---the other case, $\min\{\mathcal{S}^*\} > b_u^{(q)}$, can be treated using similar arguments. 
		
		Let $z_{m}$ be the value of $z$ at iteration $q-1$ and $z_{u} \in \argmin_{i \in \{3K/4+1, \dots, K\}} \ g_i^{(q-1)}$ and $z_{\ell} \in \argmax_{i \in \{0, \dots, K/4-1\}} \ g_i^{(q-1)}$. Following the procedure, the fact that $\max\{\mathcal{S}^*\}\ge b_{\ell}^{(q-1)}$ and $\max\{\mathcal{S}^*\} < b_{\ell}^{(q)}$ implies that the variable $b_{\ell}$ was updated, which implies that we either have
		\begin{equation}
			E_1 := \left\lbrace g^{(q-1)}_{z_{m}} \le -\frac{\mu}{2}\Delta\right\rbrace\label{eq:cond1}
		\end{equation}
		or
		\begin{equation}
			E_2 := \left\lbrace \abs{g^{(q-1)}_{z_{m}}} \le \frac{\mu}{2}\Delta \text{ and } g^{(q-1)}_{z_{u}} < -\frac{\mu}{2}\Delta\right\rbrace~. \label{eq:cond2}	
		\end{equation}
		\begin{itemize}
			\item Case 1: $E_1$ holds.
			
			Recall that event $\neg \mathcal{E}$ implies in particular that
			\[
			\abs{g^{(q-1)}_{z_{m}} - \lambda^{(q-1)}_{z_{m}}} \le \frac{\mu}{4}\Delta~.
			\]
			Using the bound above with \eqref{eq:cond1}, we conclude that $\lambda_{z_{m}}^{(q-1)} < 0$. Recall that $\lambda_{z_{m}}^{(q-1)}$ is an average of derivatives of a subset of points from $\{x_{K/4}^{(q-1)}, \dots, x_{3K/4}^{(q-1)}\}$, this average being negative implies that there is some point $j \in \{K/4, \dots, 3K/4\}$ such that
			\begin{equation}\label{eq:sign}
				f'(x^{(q-1)}_{j}) < 0~.	
			\end{equation}
			Recall that following the procedure, we have that when $E_1$ happens, then the variable $b_{u}$ is unchanged and $b_{\ell}$ is updated as
			\begin{align}
				b_{\ell}^{(q)} 
				= a_{q-1} - \frac{\gamma_{q-1}}{4}
				= x_{K/2}^{(q-1)} - \frac{x_{K}^{(q-1)}-x_0^{(q-1)}}{4}
				= x_{K/4}^{(q-1)},\label{eq:bl}
			\end{align}
			where we used in the second line the fact that $\gamma_{q-1}$ is the diameter of the interval $I$ over which the uniform grid is introduced, and in the last line the fact that the last grid is uniform.
			
			\noindent In particular \eqref{eq:bl} implies that we have $x_j^{(q-1)} \ge x_{K/4}^{(q-1)} \ge b_{\ell}^{(q)}$.
			We assumed that $\max \{\mathcal{S}^* \} < b^{(q)}_{\ell}$, since $f'$ is continuous and its roots $\mathcal{S}^*$ are a interval, we conclude that for each $x \ge b_{\ell}^{(q)}$, we have $f'(x) >0$, which contradicts \eqref{eq:sign}.  
			
			\item Case 2: $E_2$ holds.
			
			The event $\neg \mathcal{E}$ gives:
			\[
			\abs{g^{(q-1)}_{z_{u}} - \lambda^{(q-1)}_{z_{u}}} \le \frac{\mu}{4}\Delta.
			\]
			Using the bound above with \eqref{eq:cond2}, we have $\lambda^{(q-1)}_{z_{u}}< 0$.
			Since $\lambda^{(q-1)}_{z_{u}}$ is an average of the derivatives of a subset of $\{x_{K/2}^{(q-1)}, \dots, x_{K}^{(q-1)}\}$, we have that necessarily for some point $j \in \{K/2, \dots, K\}: f'(x_j^{(q-1)}) < 0$, this contradict the fact that $\max\{\mathcal{S}^*\} < b_{\ell}^{(q)} \le x_{K/2}^{(q-1)} \le x_j^{(q-1)}$ since the last inequalities imply that $f'(x_j^{(q-1)}) > 0$. 
		\end{itemize}
		We conclude that if event $\neg \mathcal{E}$ holds (which happens with probability at least $1-\delta$), we have that for all iterations $k$ reached by the algorithm: $\mathcal{S}^* \cap [b^{(k)}_{\ell}, b^{(k)}_{u}] \neq \emptyset$. The same argument applies to the bracket after the final update. 
		
		\paragraph{Part 3:} Suppose that the algorithm halts through its stopping test at iteration $k\in\{1,\dots,n\}$ and that $\neg\mathcal{E}$ holds. The stopping conditions give
		\begin{align}
			&\forall i\in\{K/4,\dots,3K/4\}:\quad |g_i^{(k)}|\le\mu\Delta/2,\label{eq:e0}\\
			&\forall i\in\{0,\dots,K/4-1\}:\quad g_i^{(k)}\le\mu\Delta/2,\label{eq:ee0}\\
			&\forall i\in\{3K/4+1,\dots,K\}:\quad g_i^{(k)}\ge-\mu\Delta/2.\label{eq:eee0}
		\end{align}
		We first show that $\mathcal{S}^*\cap I\neq\emptyset$. If $\gamma_k=b_u^{(k)}-b_{\ell}^{(k)}$, then $I=[b_{\ell}^{(k)},b_u^{(k)}]$ and Part 2 applies. Otherwise, $\gamma_k=2\Delta$. Suppose for a contradiction that $I$ contains no minimizer, and let $x^*$ be the projection of $a_k$ onto $\mathcal{S}^*$. All derivatives on $I$ have the same sign, and the RSI property gives
		\[
		\left|\lambda_{K/4}^{(k)}\right|
		=\frac{2}{K+2}\sum_{i=K/4}^{3K/4}|f'(x_i^{(k)})|
		\ge\frac{2\mu}{K+2}\sum_{i=K/4}^{3K/4}|x_i^{(k)}-x^*|
		=\mu|a_k-x^*|>\mu\Delta.
		\]
		Here we used that the central grid points have mean $a_k$ and lie on the same side of $x^*$. This contradicts
		\[
		|\lambda_{K/4}^{(k)}|\le|g_{K/4}^{(k)}|+\mu\Delta/4\le3\mu\Delta/4.
		\]
		Thus the projection $x^*$ of $a_k$ onto $\mathcal{S}^*$ lies in $I$.

		We now bound $f(a_k)-f^*$. Suppress the superscript $(k)$ and set $h=\gamma_k/K\le2\Delta/K$. If $a_k=x^*$ there is nothing to prove. Suppose first that $x^*<a_k$, and let $p$ be the smallest grid index such that $x_p\ge x^*$. Then $p\le K/2$, $0\le x_p-x^*\le h$, and
		\[
		f(a_k)-f^*\le\frac{Lh^2}{2}+\int_{x_p}^{a_k}f'(x)\,dx.
		\]
		If $x_p=a_k$, the integral is zero. Otherwise, $f'\ge0$ on $[x_p,x_K]$. We cover $[x_p,a_k]$ by intervals corresponding to the queried blocks:
		\begin{itemize}
			\item If $p\ge K/4$, use the single interval $[x_p,x_{3K/4}]$, corresponding to $\mathcal{P}_p$.
			\item If $p<K/4$, use $[x_p,x_{p+K/4}]$ and $[x_{K/4},x_{3K/4}]$, corresponding to $\mathcal{P}_p$ and $\mathcal{P}_{K/4}$.
		\end{itemize}
		In both cases these intervals cover $[x_p,a_k]$, and their total length is at most $3Kh/4$. The sum of their numbers of grid points, counting overlap, is at most $3K/4+2$. By the stopping conditions and $\neg\mathcal{E}$, each block has mean derivative at most $3\mu\Delta/4$. Applying the nonnegative-function consequence of Lemma~\ref{lem:tech} to these intervals gives
		\begin{align}
		f(a_k)-f^*
		&\le\frac{Lh^2}{2}+h\left(\frac{3K}{4}+2\right)\frac{3\mu\Delta}{4}+\frac{3KLh^2}{8}\nonumber\\
		&\le\left(\frac{21}{8}+\frac{5}{K}\right)\mu\Delta^2
		\le4\mu\Delta^2,\label{eq:stop_gap}
		\end{align}
		where we used $h\le2\Delta/K$, $L/K\le\mu$, and $K\ge4$.
		If $x^*>a_k$, take $p$ to be the largest index with $x_p\le x^*$ and apply the same argument to $-f'$ on the reflected intervals. When $x_p=a_k$, smoothness alone suffices; otherwise the corresponding blocks are $\mathcal{P}_p$, or $\mathcal{P}_p$ together with $\mathcal{P}_{3K/4}$. Their mean values of $-f'$ are at most $3\mu\Delta/4$ by \eqref{eq:e0}, \eqref{eq:eee0}, and $\neg\mathcal{E}$. Thus \eqref{eq:stop_gap} holds in every stopping case, including $k=n$.

		By the definition of $\Delta$, a stopped output therefore satisfies
		\begin{equation}\label{eq:g0}
		f(a_k)-f^*\le\frac{1024\sigma^2}{\mu T}\,n\log\left(\frac{2n(K+1)}{\delta}\right).
		\end{equation}
		If instead the algorithm completes all $n$ iterations without halting, each update preserves a minimizer and contracts the bracket length by a factor of at most $3/4$, since $\gamma_k\le b_u^{(k)}-b_{\ell}^{(k)}$. Its returned midpoint $a$ therefore satisfies
		\begin{equation}\label{eq:g1}
		f(a)-f^*\le\frac{LD^2}{8}\left(\frac34\right)^{2n}
		\le\frac{LD^2}{2}\left(\frac34\right)^n
		\le\frac{96\sigma^2}{\mu T}.
		\end{equation}
		The final inequality follows from the definition of $n$. Since $n\ge1$, $K\ge4$, and $\delta\in(0,1)$, the bound in \eqref{eq:g0} also dominates \eqref{eq:g1}. Combining these cases with \eqref{eq:conc} proves the theorem.
	\end{proof}

	\begin{lemma}\label{lem:tech}
		Let $g$ be an $L$-Lipschitz function on an interval $I\subset\R$. Let $a<b$ such that $[a,b] \subset I$. For $n \ge 2$, let $a_1, \dots, a_n$ be a sequence defined by: $a_1 = a$ and $a_i = a+\frac{i-1}{n-1}(b-a)$. We have:
		\[
		\abs{\int_{a}^b g(x)dx - \frac{b-a}{n-1}\sum_{i=1}^{n-1} g(a_i)} \le \frac{L}{2} \frac{(b-a)^2}{n-1}~.
		\]
		In particular, if $g\ge0$ on $[a,b]$, then
		\[
		\int_a^b g(x)\,dx\le\frac{b-a}{n-1}\sum_{i=1}^{n}g(a_i)+\frac{L(b-a)^2}{2(n-1)}.
		\]
	\end{lemma}
	\begin{proof}
		We have
		\begin{align*}
			\abs{\int_{a}^{b} g(x)dx - \frac{b-a}{n-1}\sum_{i=1}^{n-1} g(a_i)} &= \abs{\sum_{i=1}^{n-1}\int_{a_i}^{a_{i+1}}\bigl(g(x)-g(a_i)\bigr)dx}\\
			&\le \sum_{i=1}^{n-1} \int_{a_i}^{a_{i+1}} \abs{g(x)-g(a_i)}dx\\
			&\le \sum_{i=1}^{n-1} \int_{a_i}^{a_{i+1}} L(x-a_i)dx\\
			&= \frac{L}{2} \sum_{i=1}^{n-1} (a_{i+1}-a_i)^2\\
			&= \frac{L}{2} \frac{(b-a)^2}{n-1}~. 
		\end{align*}
	\end{proof}
	
	\begin{lemma}\label{lem:cher}
		Let $X_1, \dots, X_T$ be a sequence of independent variables that are $\sigma$-subGaussian with means $\E[X_i] = \mu_i$. Let $\hat{X}_T = \frac{1}{T}\sum_{t=1}^{T}X_t$ and $\mu = \frac{1}{T}\sum_{t=1}^{T}\mu_t$. Then, for any $\delta \in (0,1)$, with probability at least $1-\delta$, we have
		\[
		\abs{\hat{X}_T - \mu } \le \sigma\sqrt{\frac{2\log(2/\delta)}{T}}~.
		\]
	\end{lemma}
	
	\begin{lemma}\label{lem:1d-rsi}
		Let $g$ be a differentiable function in $R$; suppose $g$ is $\mu$-RSI and has a non-empty set of global minimizes $\sS^*$, then $\sS^*$ is a convex set. 
	\end{lemma}
	\begin{proof}
		Since $g$ is continuous, we only need to worry about there exists a closed interval $E$ such that only the endpoints are global minimizers. Since $g$ is differentiable on $E^o$, from Rolle's theorem, there exists a point $c$ in $E^o$ such that $g^\prime(c)=0$. However, $g$ is $\mu-RSI$,
		\[0=\langle g^\prime(c),c-c_p\rangle\geq\mu\norm{c-c_p}^2>0,\]
		which is a contradiction.
	\end{proof}
	
	\begin{lemma}\label{lem:rsi-qc}
		Let $f:\R \to \R$ be $\tau-\QC$ and $\mu-\QG$. Then $f$ is $(\mu \tau/2)-\RSI$.
	\end{lemma}
	\begin{proof}
		Let $x \in R$, denote $x_p$ its projection on the set of minimizers of $f$. Using the fact that $f$ is $\tau-\QC$ then $\mu-\QG$, we have
		\begin{align*}
			f'(x) (x-x_p) \ge \tau (f(x)-f^*)
			\ge \frac{\mu\tau}{2} (x-x_p)^2~.
		\end{align*}
	\end{proof}
	
\end{document}